\documentclass[twoside,11pt]{article}
\usepackage{blindtext}

\usepackage{subcaption}
\usepackage{jmlr2e}

\usepackage[utf8]{inputenc}
\usepackage[margin=1in]{geometry}
\usepackage[table]{xcolor}
\usepackage{appendix}

\usepackage{tablefootnote}
\usepackage{makecell}

\usepackage[ruled,vlined]{algorithm2e}

\SetKwFunction{FRecurs}{FnRecursive}%

\usepackage{float}

\usepackage{amsfonts, bm, mathtools, amsmath}

\newtheorem{problem}{Problem} %[section]
\def\state{{\vz}}

\def\parameters{{\vw}}
\def\nn{f^\parameters}
\def\act{{\sigma}}

\def\bnn{f^w}

\def\gpMean{{m}}
\def\gpKernel{{k}}

\def\probMeas{{\mathcal{P}}}
\def\discMeas{{\mathcal{D}}}
\def\borrelAlgebra{{\mathcal{B}}}

\def\expect{{\mathbb{E}}}

\def\prob{\mathbb{P}}

\def\signature{{\Delta}}

\def\GMM{\mathcal{G}}

\def\Ndist{\mathcal{N}}
\def\mNdist{{m}}
\def\vNdist{\Sigma}

\def\wasserstein{{\mathbb{W}}}
\def\mw{{\mathbb{MW}}}

\newcommand{\erf}[1]{\mathrm{erf}\left(#1\right)}

\def\lipschitz{{\mathcal{L}}}

\newcommand{\argmin}[1]{\underset{#1}{\mathrm{arg min}}}

\def\compress{{G}}

\def\vectorize{{\mathrm{vec}}}
\def\postImage{\mathrm{Post}}

\def\diag{\mathrm{diag}}
\def\trace{\mathrm{trace}}
\def\rank{\mathrm{rank}}

\newcommand{\elem}[2]{{#1^{{(#2)}}}}

\def\eucl{{\mathbb{R}}}
\def\realNum{{\mathbb{R}}}
\def\natNum{\mathbb{N}}

\def\sA{{\mathcal{A}}}

\def\sC{{\mathcal{C}}}

\def\sI{{\mathcal{I}}}

\def\sR{{\mathcal{R}}}

\def\sX{{\mathcal{X}}}
\def\sY{{\mathcal{Y}}}

\def\mI{{{I}}}

\def\mM{{{M}}}

\def\mR{{{R}}}
\def\mS{{{S}}}
\def\mT{{{T}}}

\def\mV{{{V}}}
\def\mW{{{W}}}

\newcommand{\emM}[1]{\mM^{(#1)}}

\def\vzero{{{\bar{0}}}}

\def\vlambda{{{\lambda}}}
\def\vmu{{{\mu}}}
\def\vpi{{{\pi}}}

\def\vnu{{{\nu}}}

\def\va{{{a}}}
\def\vb{{{b}}}
\def\vc{{{c}}}

\def\ve{{{e}}}

\def\vm{{{m}}}

\def\vw{{{w}}}
\def\vx{{{x}}}

\def\vz{{{z}}}

\newcommand{\evlambda}[1]{{\vlambda^{(#1)}}}

\newcommand{\evpi}[1]{{\vpi^{(#1)}}}

\newcommand{\evb}[1]{{\vb^{(#1)}}}

\newcommand{\evx}[1]{{\vx^{(#1)}}}

\def\rvg{g}

\newcommand{\qed}{\hfill\BlackBox\\[2mm]}

\newcommand{\changed}[1]{{#1}}

\newcommand{\changedNEW}[1]{{#1}}

\usepackage{booktabs} % for professional tables

\usepackage{lastpage}
\jmlrheading{27}{2026}{1-\pageref{LastPage}}{7/24; Revised 11/25}{1/26}{24-1199}{Steven Adams, Andrea Patan\`{e}, Morteza Lahijanian and Luca Laurenti} 
\ShortHeadings{Finite Neural Networks as Mixtures of Gaussian Processes}{Adams, Patan\`{e}, Lahijanian and Laurenti}
\firstpageno{1}

\begin{document}

\title{Finite Neural Networks as Mixtures of Gaussian Processes: From Provable Error Bounds to Prior Selection}

\author{\name Steven Adams \email s.j.l.adams@tudelft.nl \\
       \addr Delft Center for Systems and Control (DCSC),
       Delft University of Technology, The Netherlands
       \AND
       \name Andrea Patan\`{e} \email apatane@tcd.ie \\
       \addr School of Computer Science and Statistics, 
       Trinity College Dublin, Ireland
       \AND
       \name Morteza Lahijanian \email morteza.lahijanian@colorado.edu \\
       \addr Departments of Aerospace Engineering Sciences and Computer Science, 
       University of Colorado Boulder, USA
       \AND
       \name Luca Laurenti \email l.laurenti@tudelft.nl \\
       \addr Delft Center for Systems and Control (DCSC),
       Delft University of Technology, The Netherlands \\
       The Italian Institute of Artificial Intelligence (AI4I), Italy 
       }
       
\editor{Christophe Giraud}

\maketitle

\begin{abstract}%
    Infinitely wide or deep neural networks (NNs) with independent and identically distributed (i.i.d.) parameters have been shown to be equivalent to Gaussian processes. Because of the favorable properties of Gaussian processes, this equivalence is commonly employed to analyze neural networks and has led to various breakthroughs over the years. However, neural networks and Gaussian processes are equivalent only in the limit; in the finite case there are currently no methods available to approximate a trained neural network with a  Gaussian model with bounds on the approximation error. In this work, we present an algorithmic framework to approximate a neural network of finite width and depth, and with not necessarily i.i.d.\ parameters, with a mixture of Gaussian processes with bounds on the approximation error. In particular, we consider the Wasserstein distance to quantify the closeness between probabilistic models and, by relying on tools from optimal transport and Gaussian processes, we iteratively approximate the output distribution of each layer of the neural network as a mixture of Gaussian processes. Crucially, for any NN and $\epsilon >0$ our approach is able to return a mixture of Gaussian processes that is $\epsilon$-close to the NN at a finite set of input points. Furthermore, we rely on the differentiability of the resulting error bound to show how our approach can be employed to tune the parameters of a NN to mimic the functional behavior of a given Gaussian process, e.g., for prior selection in the context of Bayesian inference. We empirically investigate the effectiveness of our results on both regression and classification problems with various neural network architectures. Our experiments highlight how our results can represent an important step towards understanding neural network predictions and formally quantifying their uncertainty. 
\end{abstract}


\newpage
\section{Introduction}
Deep neural networks have achieved state-of-the-art performance in a wide variety of tasks, ranging from image classification \citep{vision_cnns} to robotics and reinforcement learning \citep{mnih2013playing}. In parallel with these empirical successes, there has been a significant effort in trying to understand the theoretical properties of neural networks \citep{goodfellow2016deep} and to guarantee their robustness \citep{szegedy2013intriguing}. In this context, an important area of research is that of stochastic neural networks (SNNs), where some of the parameters of the neural network (weights and biases) are not fixed, but follow a distribution. SNNs include many machine learning models commonly used in practice, such as dropout neural networks \citep{gal2016dropout}, Bayesian neural networks \citep{neal2012bayesian}, neural networks with only a subset of stochastic layers \citep{favaro2025quantitative},  and neural networks with randomized smoothing \citep{cohen2019certified}. Among these, because of their convergence to Gaussian processes, particular theoretical attention has been given to infinite neural networks with independent and identically distributed (i.i.d.) parameters \citep{neal2012bayesian,lee2018deep}. 

Gaussian processes (GPs) are a class of stochastic processes that are widely used as non-parametric machine learning models \citep{rasmussen2003gaussian}. Because of their many favorable analytic properties \citep{adler2009random}, the convergence of infinite SNNs to GPs has enabled many breakthroughs in the understanding of neural networks, including their modeling capabilities \citep{schoenholz2016deep}, their learning dynamics \citep{jacot2018neural}, and their adversarial robustness \citep{bortolussi2024robustness}. 
Unfortunately, existing results to approximate a SNN with a GP are either limited to untrained networks with i.i.d.\ parameters \citep{neal2012bayesian} or lack guarantees of correctness \citep{khan2019approximate}.  
In fact, the input-output distribution of a SNN of finite depth and width is generally non-Gaussian, even if the distribution over its parameters is Gaussian \citep{lee2020finite}.
This leads to the main question of this work: \emph{Can we develop an algorithmic framework to approximate a finite SNN (trained or untrained and not necessarily with i.i.d.\ parameters) with Gaussian models while providing formal guarantees of correctness (i.e., provable bounds on the approximation error and that can be made arbitrarily small)?}

In this paper, we propose an algorithmic framework to approximate the input-output distribution of an arbitrary SNN over a finite set of input points with a Gaussian Mixture Model (GMM), that is, a mixture of $M$ Gaussian distributions \citep{mclachlan2000finite}. Critically, the GMM approximation resulting from our approach comes with error bounds on its distance (in terms of the 2-Wasserstein distance,\footnote{Note that, as we will emphasize in Section \ref{sec:prelim}, the choice of the $2$-Wasserstein distance to quantify the distance between a SNN and a GMM guarantees that a bound on the 2-Wasserstein distance also implies a bounds in how distant are their mean and variance.} \citealp{villani2009optimal}) to the input-output distribution of the SNN. 
An illustrative example of our framework is shown in Figure \ref{fig:NoisySinesSNN}, where, given a SNN trained on a 1D regression task (Figure \ref{fig:NoisySinesSNN}), our framework outputs a GMM approximation (Figure \ref{fig:NoisySinesGMM}) with error bounds on its closeness to the SNN (Figure \ref{fig:NoisySinesERROR}). 
Our approach is based on iteratively approximating the output distribution of each layer of the SNN with a mixture of Gaussian distributions and propagating this distribution through the next layer.
In order to propagate a distribution through a non-linear activation function, we first approximate it with a discrete distribution, which we call a \emph{signature approximation}.\footnote{A discrete approximation of a continuous distribution is also called a codebook in the field of constructive quantization \citep{graf2000foundations} or particle approximation in Bayesian statistics \citep{liu2016stein}.} The resulting discrete distribution can then be propagated exactly through a layer of the neural network (activation function and \changed{affine} combination with weights and biases), which in the case of jointly Gaussian weights and biases, leads to a new Gaussian mixture distribution. 
To quantify the error between the SNN and the resulting GMM, we use techniques from optimal transport, probability theory, and interval arithmetic. In particular, for each of the approximation steps described above, we bound the error introduced in terms of the 2-Wasserstein distance. Then, we combine these bounds using interval arithmetic to bound the 2-Wasserstein distance between the SNN and the approximating GMM. 
Furthermore, by relying on the fact that GMMs can approximate any distribution arbitrarily well \citep{delon2020wasserstein}, we prove that by appropriately picking the parameters of the GMM, the bound on the approximation error converges uniformly to zero by increasing the size of the GMM. Additionally, we show that the resulting error bound is piecewise differentiable in the hyper-parameters of the neural network, allowing gradient-based optimization of the behavior of a NN to mimic that of a given GMM.

\begin{figure}[t]
    \centering
    \begin{subfigure}[b]{0.32\textwidth}
        \centering
        \includegraphics[width=\textwidth]{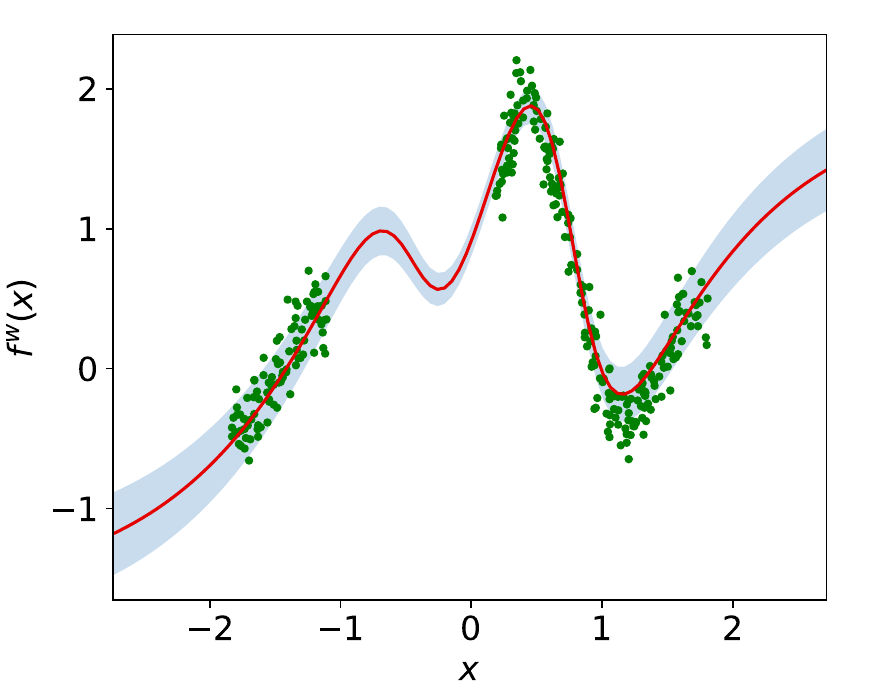}
        \caption{SNN}
        \label{fig:NoisySinesSNN}
    \end{subfigure}
    \hfill
    \begin{subfigure}[b]{0.32\textwidth}
        \centering
        \includegraphics[width=\textwidth]{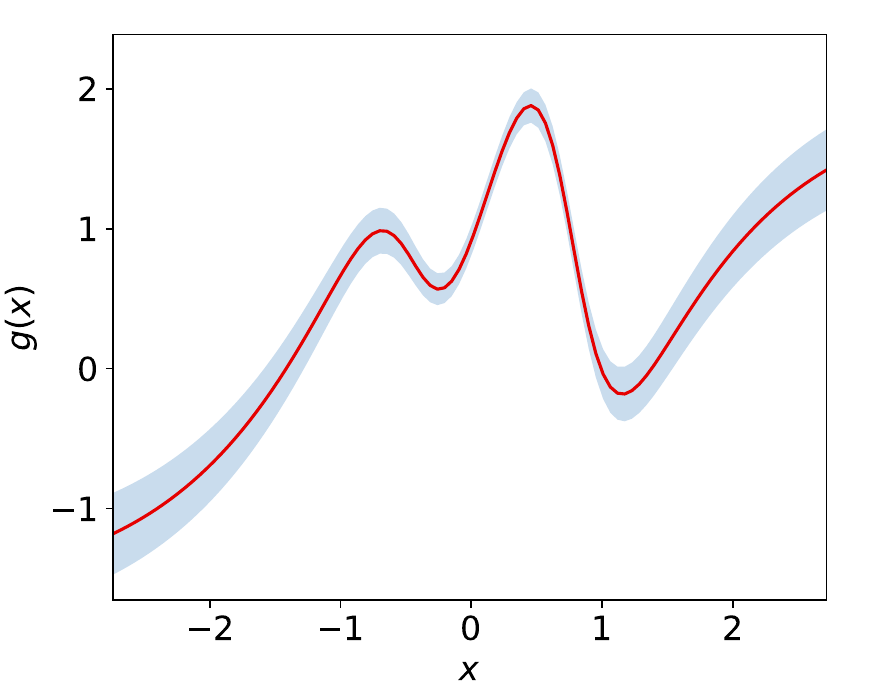}
        \caption{GMM}
        \label{fig:NoisySinesGMM}
    \end{subfigure}
    \begin{subfigure}[b]{0.32\textwidth}
        \centering
        \includegraphics[width=\textwidth]{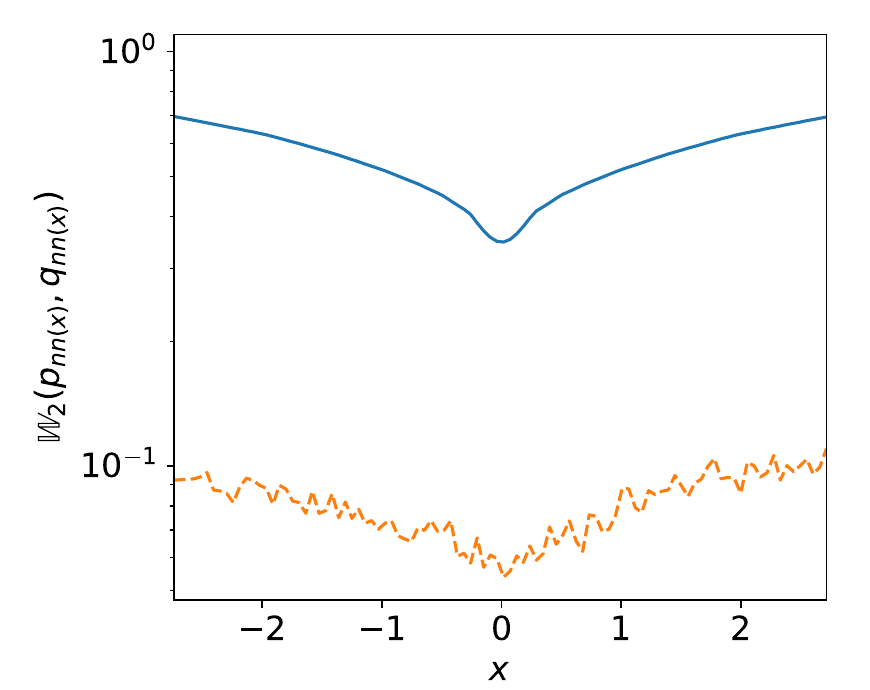}
        \caption{Error}
        \label{fig:NoisySinesERROR}
    \end{subfigure}
    \caption{
    Visualization of (a) a Monte Carlo approximation of a SNN trained on samples from a 1D mixture of sines with additive noise, and (b) the GMM approximation of the same SNN composed of a mixture of size $10$, as obtained with our approach. 
    The SNN has one hidden layer with 64 neurons and is trained using VOGN \citep{khan2018fast}. 
    The red line and blue-shaded region in (a) and (b) illustrate the point-wise mean and standard deviation at a uniform grid of 100 points in the input domain. In (c), the dashed and solid lines represent, respectively, the empirical estimates of the 2-Wasserstein distance between the GMM and SNN distributions at each input point, and formal bounds of the same quantity, as obtained by the results derived in this paper. }
    \label{fig:NoisySine}
\end{figure}

We empirically validate our framework on various SNN architectures, including fully-connected and convolutional layers, trained for both regression and classification tasks on various data sets including MNIST, CIFAR-10, and a selection of UCI data sets. Our experiments confirm that our approach can successfully approximate a SNN with a GMM with arbitrarily small error, albeit with increasing computational costs when the network's depth increases. Furthermore, perhaps surprisingly, our results show that even a mixture with a relatively small number of components generally suffices to empirically approximate a SNN accurately. To showcase the importance of our results, we then consider two applications: (i) uncertainty quantification in SNNs, and (ii) prior selection for SNNs. In the former case we show how the GMM approximation resulting from our framework can be used to study and quantify the uncertainty of the SNN predictions in classification tasks on the MNIST and CIFAR-10 data sets. In the latter case,  we consider prior selection for neural networks, which is arguably one of the most important problems in performing Bayesian inference with neural networks \citep{fortuin2022priors}, and show that our framework allows one to precisely encode functional information in the prior of SNNs expressed as Gaussian processes (GPs), thereby enhancing  SNNs' posterior performance and outperforming state-of-the-art methods for prior selection of neural networks.

In summary, the main contributions of our paper are:
\begin{itemize}
\item We introduce a framework based on discrete approximations of continuous distributions to approximate SNNs of arbitrary depth and width as GMMs, with formal error bounds in terms of the 2-Wasserstein distance.
\item We prove the uniform convergence of our framework by showing that for any finite set of input points our approach can return a GMM of finite size such that the 2-Wasserstein distance between this distribution and the joint input-output distribution of the SNN on these points can be made arbitrarily small. 
\item 
We perform a large-scale empirical evaluation to demonstrate the efficacy of our algorithm in approximating SNNs with GMMs and to show how our results could have a profound impact in various areas of research for neural networks, including uncertainty quantification and prior selection.
\end{itemize}

\subsection{Related Works}
The convergence of infinitely wide SNNs with i.i.d.\ parameters to GPs was first studied by \cite{neal2012bayesian} by relying on the central limit theorem. 
The corresponding GP kernel for the one hidden layer case was then analytically derived by \cite{williams1996computing}. These results were later generalized to deep NNs \citep{hazan2015steps,lee2018deep,matthews2018gaussian}, convolutional layers \citep{garriga2018deep,novak2018bayesian,garriga2021correlated}, and general non-linear activation functions \citep{hanin2023random}. However, these results only hold for infinitely wide or deep neural networks with i.i.d.\ parameters; in practice, NNs have finite size and depth and for trained NNs their parameters are generally not i.i.d..
To partially address these issues, recent works have started to investigate how these results apply in the finite case \citep{dyer2020asymptotics,antognini2019finite,yaida2020non,bracale2021large,klukowski2022rate,balasubramanian2024gaussian}.  
In particular, \cite{eldan2021non} were the first to provide upper bounds on the rates at which a single layer SNN with a specific parameter distribution converges to the infinite width GP in terms of the Wasserstein distance.
% addressed this problem, exploiting Wasserstein distance to quantify how close a single layer SNN is to the infinite width GP. 
This result was later generalized to isotropic Gaussian \changed{parameter} distributions \citep{cammarota2023quantitative}, deep Random NN \citep{basteri2022quantitative}, and to other metrics, such as the total variation, sup norm, and Kolmogorov distances \citep{apollonio2025normal,bordino2023non,favaro2025quantitative}. 
However, to the best of our knowledge, all existing works are limited to untrained SNNs with i.i.d.\ weights and biases, which may be taken as priors in a Bayesian setting \citep{bishop2006pattern} or may represent the initialization of gradient flows in an empirical risk minimization framework \citep{jacot2018neural}.
In contrast, critically, in our paper we allow for correlated and non-identical distributed weights and biases, thus also including trained posterior distributions of SNNs learned via \changed{Variational Inference}.
In this context, \cite{khan2019approximate} showed that for a subset of Gaussian Bayesian inference techniques, the approximate posterior \changed{parameter} distributions are equivalent to the posterior distributions of GP regression models, implying a linearization of the approximate SNN posterior in function space \citep{immer2021scalable}. Instead, we approximate the SNN with guarantees in function space and our approach generalizes to any approximate inference technique. 

\changedNEW{
A fundamentally different approach to approximating general distributions is to use empirical measures obtained from independent samples of the target distribution. 
Such methods provide only statistical rather than deterministic guarantees: their accuracy holds with high probability and relies on assumptions such as bounded support or low dimensionality \citep{fournier2015rate}, which typically do not hold for SNN output distributions.
}

While the former set of works focuses on the convergence of SNNs to GPs, various works have considered the complementary problem of finding the distribution of the parameters of a SNN that mimic a given GP \citep{flam2017mapping,flam2018characterizing,tran2022all,matsubara2021ridgelet}. This is motivated by the fact that GPs offer an excellent framework to encode functional prior knowledge; consequently, this problem has attracted significant interest in encoding functional prior knowledge into the prior of SNNs. 
In particular, closely related to our setting are the works of \cite{flam2017mapping} and \cite{tran2022all}, which optimize the parametrization of the \changed{parameter}-space distribution of the SNN to minimize the KL divergence and 1-Wasserstein distance, respectively, with a desired GP prior. They then utilize the optimized \changed{parameter} prior to perform Bayesian inference, showing improved posterior performance. However, these methods lack formal guarantees on the closeness of the optimized SNN and the GP that the error bounds we derive in this paper provide.

\section{Notation}
For a vector $\vx\in\eucl^n$, we denote by $\|\vx\|$ the Euclidean norm on $\eucl^n$, and use $\evx{i}$ to denote the $i$-th element of $\vx$. Similarly, for a matrix $\mM\in\eucl^{n\times m}$ we denote by $\|\mM\|$ the spectral (matrix) norm of $\mM$ and use $\emM{i,j}$ to denote the $(i,j)$-th element of $\mM$. 
In addition, $\diag(\vx)$ denotes a matrix with the elements of $\vx$ on its diagonal. 
For a matrix $\mT\in\eucl^{m\times n}$, the post image of a region $\sX\subset\eucl^n$ under $\mT$ is defined as $\postImage(\sX,\mT)=\{\mT\vx\mid\vx\in\sX\}$.
If $\sX$ is finite, $|\sX|$ denotes its cardinality.

Given a measurable space $(\sX, \borrelAlgebra(\sX))$ with $\borrelAlgebra(\sX)$ being the $\sigma-$algebra,  
we denote with $\probMeas(\sX)$ the set of probability distributions on $(\sX, \borrelAlgebra(\sX))$. In this paper, for a metric space $\sX$, $\borrelAlgebra(\sX)$ is assumed to be the Borel $\sigma$-algebra of $\sX$. 
Considering two measurable spaces $(\sX, \borrelAlgebra(\sX))$ and $(\sY, \borrelAlgebra(\sY))$, a probability distribution $p\in\probMeas(\sX)$, and a measurable mapping $h:\sX\rightarrow\sY$, we use $h\#p$ to denote the pushforward measure of $p$ by $h$, i.e., the measure on $\sY$ such that \changed{$\forall\sA\in\borrelAlgebra(\sY)$}, $(h\#p)(\sA)=p(h^{-1}(\sA))$.  
For $N\in\mathbb{N}$, $\Pi_N=\{\pi \in \eucl^N_{\geq0}: \sum^N_{i=1}\elem{\pi}{i}=1 \}$ is the N-simplex. A discrete probability distribution $d\in \probMeas(\sX)$ is defined as $d=\sum_{i=1}^N\evpi{i}\delta_{\vc_i}$, where $\delta_\vc$ is the Dirac delta function centered at location $\vc \in \sX$ and  $\pi \in \Pi_N$. Lastly, the set of discrete probability distributions on $\sX$ with at most $N$ locations is denoted as $\discMeas_N(\sX)\subset \mathcal{P}(\sX)$.

\section{Preliminaries}\label{sec:prelim}
In this section, we give the necessary preliminary information on  Gaussian models and on the Wasserstein distance between probability distributions. 
\subsection{Gaussian Processes and Gaussian Mixture Models}
A Gaussian process (GP) $\rvg$ is a stochastic process such that for any finite collection of input points $\sX=\{\vx_1,\hdots,\vx_D\}$  the joint distribution of $\rvg(\sX)\coloneqq \{\rvg(\vx_1),\hdots,\rvg(\vx_D)\}$ follows a multivariate Guassian distribution with mean function $\gpMean$ and covariance function $\gpKernel$, i.e., $\rvg(\sX)\sim \Ndist(\gpMean(\sX),\gpKernel(\sX,\sX))$ \citep{adler2009random}. 
A Gaussian Mixture Model (GMM) with $M\in \mathbb{N}$ components, is a set of $M$ GPs, also called components, averaged w.r.t.\ a probability vector $\pi\in\Pi_M$ \citep{tresp2000mixtures}. Therefore, the probability distribution of a GMM  follows a Gaussian mixture distribution. 
\begin{definition}[Gaussian Mixture Distribution]
   A probability distribution $q\in \probMeas(\eucl^d)$ is called a  Gaussian Mixture distribution of size $M\in\natNum$  if  $q=\sum_{i=1}^M\elem{\vpi}{i}\Ndist(\mNdist_i,\vNdist_i),$
    where $\pi\in \Pi_M$ and $\mNdist_i\in\eucl^d$ and $\vNdist_i\in\eucl^{d\times d}$ are the mean and covariance matrix of the $i$-th Gaussian distribution in the mixture. The set of all Gaussian mixture distributions with $M$ or less components is denoted by $\GMM_M(\eucl^d)\subset \probMeas(\eucl^d)$.
\end{definition}
One of the key properties of GMMs, which motivates their use in this paper to approximate the probability distribution induced by a neural network, is that for large enough $M$ they can approximate any continuous probability distribution arbitrarily well \citep{delon2020wasserstein}. Furthermore, being a Gaussian mixture distribution a weighted sum of Gaussian distributions, it inherits the favorable analytic properties of Gaussian distributions \citep{bishop2006pattern}.

\subsection{Wasserstein Distance}
To approximate SNNs to GMMs and vice-versa, and quantify the quality of the approximation, we need a notion of distance between probability distributions. While various distance measures are available in the literature, in this work we consider the Wasserstein distance \citep{gibbs2002choosing}. 
To define the Wasserstein distance, for $\rho\geq 1$ we define the $\rho-$Wasserstein space of distributions $\probMeas_\rho(\eucl^d)$ as the set of probability distributions with finite moments of order $\rho$, i.e., any $p \in \probMeas_\rho(\eucl^d)$ is such that $\int_{\eucl^d}\|x\|^\rho dp(x)< \infty$.  For $p, q \in \probMeas_\rho(\eucl^d)$, the $\rho$-Wasserstein distance $\wasserstein_\rho$ between $p$ and $q$ is defined as 
\begin{equation}\label{eq:defWasser}
    \wasserstein_\rho(p,q)\coloneqq \left( \inf_{\gamma\in\Gamma(p,q)}\expect_{(\vx,\vz) \sim \gamma}[\|\vx-\vz\|^\rho]\right)^{\frac{1}{\rho}}=\left( \inf_{\gamma\in\Gamma(p,q)}\int_{\eucl^d\times \eucl^d}\|\vx-\vz\|^\rho \gamma(d\vx,d\vz)\right)^{\frac{1}{\rho}},
\end{equation}
where  $\Gamma(p,q)\subset \probMeas_\rho(\eucl^d\times\eucl^d)$ represents the set of probability distributions with marginal distributions $p$ and $q$. 
It can be shown that $\wasserstein_\rho$ is a metric, which is given by the minimum cost, according to the $\rho-$power of the Euclidean norm, required to transform one probability distribution into another \citep{villani2009optimal}. 
Furthermore, another attractive property of the $\rho$-Wasserstein distance, which distinguishes it from other divergence measures such as the KL divergence \citep{hershey2007approximating}, is that closeness in the $\rho$-Wasserstein distance implies closeness in the first $\rho$ moments.\footnote{See Lemma~\ref{lemma:1WasserBounds12Moment} in Appendix~\ref{append:wasserGeneralProp}.} 

\changed{In what follows, we focus on the $2$-Wasserstein distance}.\footnote{Since $\wasserstein_1\leq\wasserstein_2$, the methods presented in this work naturally extend to the 1-Wasserstein distance. Detailed comparisons and potential improvements when utilizing the $1$-Wasserstein distance are reported in the Appendix.}  
\changed{
While closed-form expressions exist for the $2$-Wasserstein distance for Gaussian distributions \citep{givens1984class},\footnote{For $p_1=\Ndist(\mNdist_1,\vNdist_1)\in\GMM(\eucl^n)$ and $p_2=\Ndist(\mNdist_2,\vNdist_2)\in\GMM(\eucl^n)$, it holds that $\wasserstein_2^2(p_1,p_2) = \|\mNdist_1-\mNdist_2\|^2 + \trace\left(\vNdist_1+\vNdist_2 - 2\left(\vNdist_1^{1/2}\vNdist_2\vNdist_1^{1/2}\right)^{1/2}\right)$.}
this is not the case for Gaussian mixture distributions. However, a tractable upper bound for the $ \wasserstein_2 $ is provided by the $\mw_2$ distance, formally defined below following \citep{delon2020wasserstein}. 
\begin{definition}[$\mw_2$ Distance]\label{def:mw2}
    Let $p=\sum_{i=1}^{N_p}\elem{\vpi_p}{i}p_i\in\GMM_{N_1}(\eucl^n)$ and $q=\sum_{j=1}^{N_q}\elem{\vpi_q}{j}q_j\in\GMM_{N_1}(\eucl^n)$ be two Gaussian Mixture distributions. Then, the $\mw_2$-distance is defined as
    \begin{equation}\label{eq:definitionMW2}
        \begin{aligned}
        \mw_2(p,q) \coloneqq \left(\inf_{\gamma\in\Gamma(p,q)\cap\GMM_{<\infty}(\eucl^{2n})}\expect_{(\vx,\vz)\sim\gamma}[\|\vx-\vz\|^2]\right)^{\frac{1}{2}} = \bigg(\min_{\vpi\in\Gamma(\vpi_p,\vpi_q)}
        % \sum_{i}^{N_p}\sum_{j}^{N_q}
        \sum_{i,j}
        \elem{\vpi}{i,j} \wasserstein_2^2(p_i,q_j)\bigg)^{\frac{1}{2}},
        \end{aligned}
    \end{equation}
    where $\GMM_{<\infty}(\eucl^n\times\eucl^n)$ is the set of Gaussian mixture distributions with finitely many components.
\end{definition}
The $\mw_2$ distance is a Wasserstein-type distance that restricts the set of possible couplings to Gaussian mixtures. This guarantees that $\mw_2$ computation reduces to a finite discrete linear program with  $N_p\cdot N_q$ optimization variables and coefficients \citep{delon2020wasserstein}. The coefficients are given by the Wasserstein distance between the mixture's Gaussian components, which have closed-form expressions. Furthermore, as Definitions~\ref{eq:defWasser} and \ref{eq:definitionMW2} only differ for the intersection on $\GMM_{<\infty}$ in the couplings, it is easy to see that $\mw_2(p,q)\geq \mathbb{W}_2(p,q)$. 
}

\section{Problem Formulation}
In this section, 
we first introduce the class of neural networks considered in this paper, and then we formally state our problem.
\subsection{Stochastic Neural Networks (SNNs)}\label{sec:snns}
For an input vector $x\in \eucl^{n_0}$, we consider a \changed{feedforward} neural network of $K$ hidden layers $f^{w}(x)$ defined iteratively as:
\changed{
\begin{subequations}\label{eq: nn} 
\begin{align}
    &\vz_{1} = L^{\vw_0}(\vx) \\
    &\vz_{k+1} = L^{\vw_k}(\act(\vz_k)), \qquad k \in \{1,\hdots,K\} \\
    &\nn(\vx)=\state_{K+1}
\end{align}
\end{subequations}
where,} for $n_k$ being the number of neurons of layer $k$, we have that $\act_k:\eucl^{n_{k}}\to \eucl^{n_{k}}$ is a the vector of piecewise continuous activation functions, and \changed{$L^{\vw_k}:\realNum^{n_k}\rightarrow\realNum^{n_{k+1}}$ denotes an affine transformation, such as a fully-connected (dense) layer or a convolutional layer. Each \(L^{\vw_k}\) is parametrized by weights $\mW_{k}$, e.g., a transformation matrix or convolutional kernel, and bias \(\vb_k\), which are collectively stored in \(\vw_k=\{\mW_k,\vb_k\}\).}
We denote by $\vw=\{\vw_k\}_{k=0}^K$ the union of all the neural network parameters.
$\changed{\vz_{K+1}}$ is the final output of the network, possibly corresponding to the vector of logits in case of classification problems.

\changed{In this work, we assume that the parameters \(\vw_k\) of each layer \(k\), rather than being fixed, are distributed according to a probability distribution $p_{\vw_k}$.}\footnote{This includes architectures with both deterministic and stochastic \changed{parameters}, where the distribution over deterministic \changed{parameters} can be modeled as a Dirac delta function.}
For any $x\in \mathbb{R}^{n_0},$ placing a distribution over the \changed{parameters} 
leads to a distribution over the outputs of the NN. That is, $\bnn(x)$ is a random variable and  $\{\bnn(x)\}_{x\in \mathbb{R}^{n_0}}$ is a stochastic process, which we call a \emph{stochastic neural network} (SNN) to emphasize its randomness \citep{yu2021simple}. In particular, $\bnn(x)$ follows a probability distribution $p_{nn(x)}$, which, as shown in \cite{adams2023bnn}, can be iteratively defined over the layers of the SNN. Specifically, as shown in Eqn~\eqref{eq:p_nn(x),k} below, $p_{nn(x)}$ is obtained by recursively propagating through the layers of the NN architecture the distribution induced by the random \changed{parameters} at each layer. The propagation is obtained by marginalization of the output distribution at each layer with the distribution of the previous layer:
\changed{
\begin{subequations}\label{eq:p_nn(x),k}
\begin{align}
    & p_{nn(\vx),1}=\expect_{\state_0\sim\delta_\vx}[L^{\vw_0}(\state_0)\# p_{\vw_0}] \label{eq:p_nn(x),k,init} \\
    & p_{nn(\vx),k+1} = \expect_{\vz_{k} \sim p_{nn(\vx),k}}[L^{\vw_{k}}(\act(\vz_{k}))\#p_{\vw_k}], \qquad k \in \{1,\hdots,K\}  \label{eq:p_nn(x),k,update} \\
    & p_{nn(x)} = p_{nn(\vx),K+1}
\end{align}
\end{subequations}
where} $\delta_\vx$ is the Delta Dirac function centered at $x$.\footnote{\changed{Equivalently, for all \(k\), \(p_{nn(\vx),k+1}\) is a mixture distribution with infinitely many components of the form \(L^{\vw_k}(\sigma(\vz_k))\#p_{\vw_k}\) for \(\vz_k\in\realNum^{n_k}\), weighted by \(p_{nn(\vx),k}\), or, equivalently, the pushforward of the joint distribution of \(p_{\vw_k}\) and \( p_{nn(\vx),k}\) by \(L^{\vw_k}(\sigma(\vz_k))\).
}} 
For any finite subset $\sX\subset\eucl^{n_0}$ of input points, we use $p_{nn(\sX)}$ to denote the joint distribution of the output of $\bnn$ evaluated at the points in $\sX$.

In what follows, because of its practical importance and the availability of closed-form solutions, we will introduce our methodological framework under the assumption that \changed{$p_{\vw_k}$ is a multivariate Gaussian distributions for all $k$.}
We stress that this does not imply that the output distribution of the SNN, $p_{nn}(x)$, is also Gaussian; in fact, in general, it is not. Furthermore, we should also already remark that, as we will explain in \changed{Remark~\ref{remark:GenerlizationAlgoNonGaus}} and illustrate in the experimental results, the methods we present in this paper can be extended to more general (non-necessarily Gaussian) $\changed{p_{\vw_k}}\in\probMeas_2(\cdot)$.

\begin{remark}
    The above definition of SNNs encompasses several stochastic models of practical importance, including Bayesian Neural Networks (BNNs) trained with Gaussian Variational Inference methods such as, e.g., Bayes by Backprop \citep{blundell2015weight} or Noisy Adam \citep{zhang2018noisy},
    NNs with only a subset of stochastic layers \citep{tang2013learning}, Gaussian Dropout NNs \citep{srivastava2014dropout}, and NNs with randomized smoothing and/or stochastic inputs \citep{cohen2019certified}.
    \changed{The methods proposed in this paper apply to all such models that assume the weight distributions of different layers are independent.} 
    In particular, in the case of BNNs, in Section \ref{sec:experiments}, we will show how our approach can be used to investigate both the prior and posterior behavior, in which case, depending on the context, \changed{$p_w$} could represent either the prior or the \changed{approximate} posterior distribution of the weights and biases.
\end{remark}

\subsection{Problem Statement}
Given an error threshold $\epsilon>0$ and a SNN, the main problem we consider, as formalized in Problem~\ref{Prob:mainProb} below, is that of finding a GMM that is $\epsilon-$close according to the 2-Wasserstein distance to the SNN. 
\begin{problem}\label{Prob:mainProb}
    Let $\sX\subset\eucl^{n_0}$ be a finite set of input points, and $\epsilon >0$ be an error threshold. Then, for a SNN with distribution $p_{nn(\mathcal{X})}$ find a GMM with distribution $q_{nn(\sX)}$ such that:
    \begin{equation}\label{eq:mainProb}
        \wasserstein_2\left(p_{nn(\mathcal{X})}, q_{nn(\sX)}\right)\leq \epsilon.
    \end{equation}
\end{problem}
In Problem~\ref{Prob:mainProb} we aim to approximate a SNN with a GMM, thus extending existing results \citep{neal2012bayesian,matthews2018gaussian} that approximate an untrained SNN with i.i.d.\ parameters with a GP under some limit (e.g., infinite width, \citealp{neal2012bayesian,matthews2018gaussian}, or infinitely many convolutional filters, \citealp{garriga2018deep}). 
In contrast, in Problem \ref{Prob:mainProb} we consider both trained and untrained SNNs of finite width and depth and, crucially, we aim to provide formal error bounds on the approximation error. Note that in Problem \ref{Prob:mainProb} we consider a set of input points and compute the $2-$Wasserstein distance of the joint distribution of a SNN and a GMM on these points. Thus, also accounting for the correlations between these input points. 

\begin{remark}\label{remark:InverseProb}
    While in Problem \ref{Prob:mainProb} we seek for a Guassian approximation of a SNN, one could also consider the complementary problem of finding the parameters of a SNN that best approximate a given GMM. Such a problem also has wide application. In fact, a solution to this problem would allow one to encode informative functional priors represented via Gaussian processes to SNNs, 
    addressing one of the main challenges in performing Bayesian inference with neural networks \citep{flam2017mapping,tran2022all}. 
    As the Wasserstein distance satisfies the triangle inequality and because there exist closed-form expressions for an upper bound on the Wasserstein distance between Gaussian Mixture distributions \citep{delon2020wasserstein}, a solution to Problem \ref{Prob:mainProb} can be readily used to address the above mentioned complementary problem. This will be detailed in Section \ref{sec:differentiability} and empirically demonstrated in Section \ref{sec:expPriorTuning}.
\end{remark}

\begin{remark}
    Note that in Problem \ref{Prob:mainProb} we consider a GMM of arbitrary \changed{finite} size. This is because such a model can approximate any continuous distribution arbitrarily well, and, consequently, guarantees of convergence to satisfy Problem~\ref{Prob:mainProb} can be obtained, as we will prove in Section \ref{sec:approxPredPostByGMM}. 
    However, Problem~\ref{Prob:mainProb} could be restricted to the single Gaussian case, in which case Problem \ref{Prob:mainProb} reduces to finding the GP that best approximates a SNN. However, in this case, because $p_{nn(\mathcal{X})}$ is not necessarily Gaussian, the resulting distance bound between the GP and the SNN may not always be made arbitrarily small.
\end{remark}

\begin{figure}[ht]
    \centering
    \includegraphics[width=0.95\textwidth]{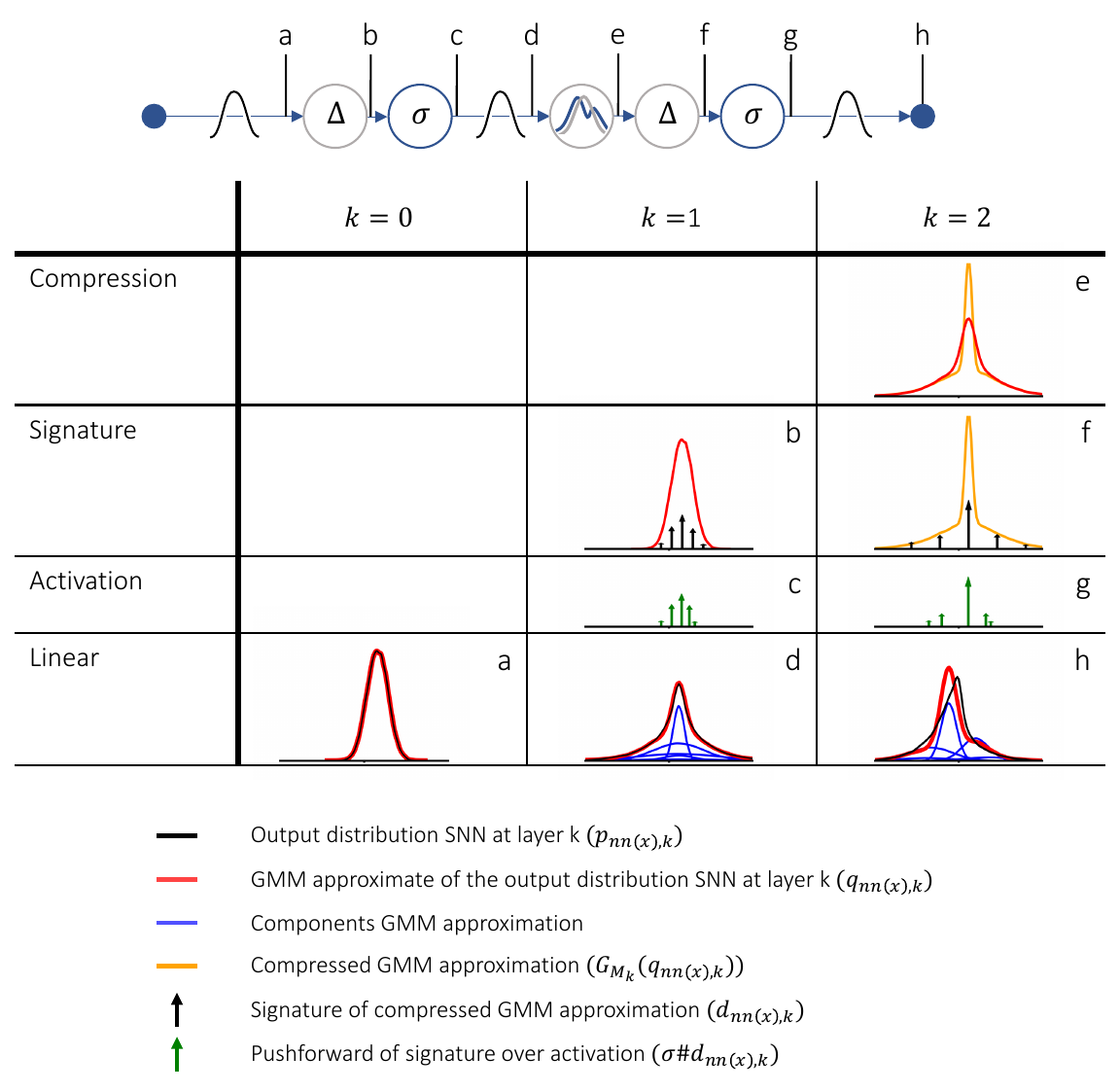}
    \caption{Illustration of the iterative approximation procedure of the output distribution of a single-input single-output SNN with two hidden layers ($K=2$) and one neuron per layer ($n_1,n_2=1$) at a single input point by a Gaussian Mixture distribution.
    }
    \label{fig:method}
\end{figure}

\subsubsection{Approach Outline}
Our approach to solving Problem \ref{Prob:mainProb} is based on iteratively approximating the output distribution of each layer of a SNN as a mixture of Gaussian distributions and quantifying the approximation error introduced by each operation. As illustrated in Figure \ref{fig:method}, following the definition of $p_{nn(\vx)}$ in Eqn \eqref{eq:p_nn(x),k}, the first step is to perform an \changed{affine} combination of the input point $\vx$ with the parameters of the first layer (step a). Because the SNN 
\changed{parameters} are assumed to be jointly Gaussian and jointly Gaussian random variables are closed under \changed{affine} combination, this \changed{affine} operation leads to a Gaussian distribution. Then, before propagating this distribution through an activation function, a signature operation on the resulting distribution is performed; that is, the continuous distribution is approximated into a discrete distribution (step b). After the signature approximation operation, the resulting discrete distribution is passed through the activation function (step c) and an \changed{affine} combination with the \changed{parameters} of the next layer is performed (step d). Under the assumption that the \changed{parameters} are jointly Gaussian, 
this \changed{affine} operation results into a Gaussian mixture distribution of size equal to the support of the discrete distribution resulting from the signature operation. To limit the computational burden of propagating a Gaussian mixture with a large number of components, a compression operation is performed that compresses this distribution into another mixture of Gaussian distributions with at most $M$ components for a given $M$ (step e). 
After this, a signature operation is performed again on the resulting Gaussian mixture distribution (step f), and the process is repeated until the last layer. Consequently, to construct $q_{nn(\vx)}$, the Gaussian mixture approximation of $p_{nn(\vx)}$, our approach iteratively performs four operations: \changed{affine} transformation, compression, signature approximation, and propagation through an activation function. To quantify the error in our approach, we derive formal error bounds in terms of the Wasserstein distance for each of these operations and show how the resulting bounds can be combined via interval arithmetic to an upper bound of $\wasserstein_2\left(p_{nn(\vx)}, q_{nn(\vx)}\right)$.

In what follows, first, as it is one of the key elements of our approach, in Section \ref{sec:signatures} we introduce the concept of signature of a probability distribution and derive error bounds on the 2-Wasserstein distance between a Gaussian mixture distribution and its signature approximation. Then, in Section \ref{sec:approxPredPostByGMM} we formalize the approach described in Figure \ref{fig:method} to approximate a SNN with a GMM and derive bounds for the resulting approximation error. Furthermore,  in Subsection \ref{sec:ConvergenceAnalysis} we prove that $q_{nn(x)}$, the Gaussian mixture approximation resulting from our approach, converges uniformly to $p_{nn(x)},$ the distribution of the SNN, that is, at the cost of
increasing the support of the discrete approximating distributions and the number of components of $q_{nn(x)}$,
the 1- and 2-Wasserstein distance between $q_{nn(x)}$ and $p_{nn(x)}$ can be made arbitrarily small. 
Then, in Section \ref{sec:AlgorithmicSection} we present a detailed algorithm of our approach. 
Finally, we conclude our paper with Section \ref{sec:experiments}, where an empirical analysis illustrating the efficacy of our approach is performed.

\section{Approximating Gaussian Mixture Distributions by Discrete Distributions}\label{sec:signatures}
One of the key steps of our approach is to perform a signature operation on a Gaussian Mixture distribution, that is, a Gaussian Mixture distribution is approximated with a discrete distribution, called a signature (step b and f in Figure \ref{fig:method}). In Subsection \ref{sec:introSignatures} we formally introduce the notion of the signature of a probability distribution. Then, in Subsection \ref{sec:wasser4signatures} we show how for a Gaussian Mixture distribution a signature can be efficiently computed with guarantees on the closeness of the signature and the original Gaussian mixture distribution in the $\wasserstein_2$-distance.  
\begin{figure}[htbp]
    \centering
    \includegraphics[width=0.5\textwidth]{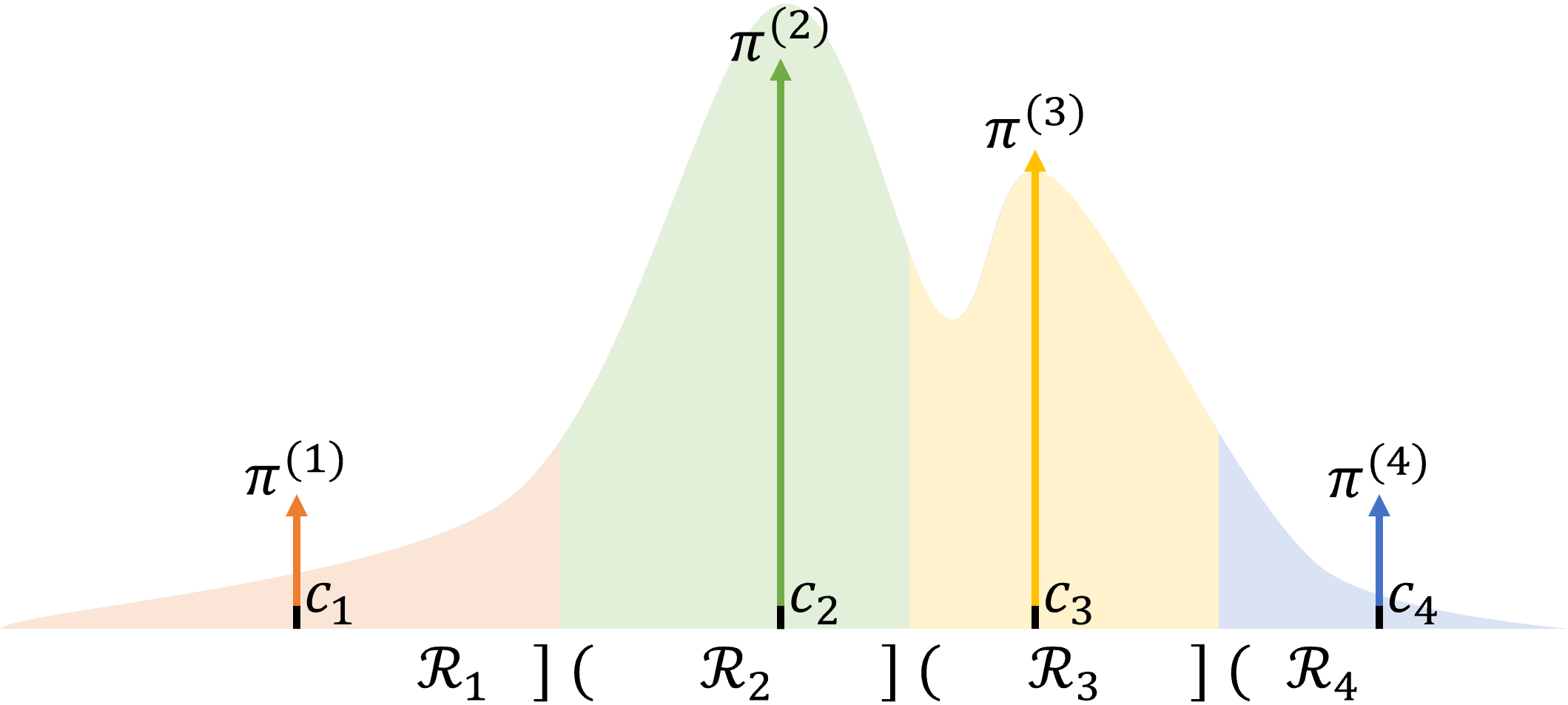}
    \caption{Signature of a continuous probability distribution. The signature is defined by locations $\{\vc_i\}_{i=1}^4$ which imply Voronoi partition $\{\sR_i\}_{i=1}^4$. The continuous distribution is discretized w.r.t. the partition by storing all the probability mass $\evpi{i}$ covered by $\sR_i$ at $\vc_i$.}
    \label{fig:signature}
\end{figure}

\subsection{Signatures: the Discretization of a Continuous Distribution}\label{sec:introSignatures}
For a set of points $\sC=\{\vc_i\}_{i=1}^N\subset\eucl^d$ called \emph{locations}, we define $\signature_\sC: \eucl^d\rightarrow\sC$ as the function that assigns any $\vz\in\eucl^d$ to the closest point in $\sC$, i.e.,   
\begin{equation*}
    \signature_\sC(\vz)\coloneqq\argmin{\vc\in\sC}\|\vz-\vc\|.
\end{equation*}
The pushforward operation induced by $\signature_\sC$ is a mapping from $\probMeas(\eucl^d)$ to $\discMeas_N(\eucl^d)$ and defines the \textit{signature} of a probability distribution. That is, as formalized in Definition \ref{def:sigagnature} below, a signature induced by $\signature_\sC$ is an approximation of a continuous distribution with a discrete one of support with a cardinality $N$.
\begin{definition}[Signature of a Probability Distribution]\label{def:sigagnature}
    The signature of a probability distribution $p\in\probMeas(\eucl^d)$ w.r.t.\ points $\sC=\{\vc_i\}_{i=1}^N\subset\eucl^d$ is the discrete distribution $\signature_\sC\#p = \sum_{i=1}^N \pi^{(i)} \delta_{c_i} \in\discMeas_N(\eucl^d)$, where $\evpi{i}=\prob_{\vz\sim p}[\vz\in\sR_i]$ with 
    \begin{equation}\label{eq:VoronoiPartitionOfSignature}
        \sR_i\coloneqq\left\{\vz\in\eucl^d \colon \|\vz-\vc_i\|\leq\|\vz-\vc_j\|, \forall j\in\{1,\hdots,N\}, j\neq i\right\}.  
    \end{equation}
\end{definition}
The intuition behind a signature is illustrated in Figure \ref{fig:signature}: the signature of distribution $p$ is a discretization of $p$ that assigns to each location $\vc_i\in \sC$ a probability given by the probability mass of $\sR_i$ according to $p$. 
Note that partition $\{\sR_i\}_{i=1}^N$, as defined in Eqn \eqref{eq:VoronoiPartitionOfSignature}, is the Voronoi partition of $\eucl^d$ w.r.t.\ the Euclidean distance. Hence, the signature of a probability distribution can be interpreted as a discretization of the probability distribution w.r.t.\ a Voronoi partition of its support. 
In the remainder, we refer to $\sC$ as the locations of the signature, $|\sC|$ as the signature size, $\{\sR_i\}_{i=1}^N$ as the partition of a signature, and we call $\signature$ the signature operation. 
\changed{
\begin{remark}\label{remark:tractable_signature}
    The signature of a probability distribution admits a closed-form expression when $\pi$ in Definition~\ref{def:sigagnature} is analytically tractable, i.e., the probability mass over the signature's partition $\{\sR_i\}_{i=1}^N$, is analytically tractable. 
    For a Gaussian distribution, this is the case when the regions $\sR_i$ are aligned with the geometry induced by the covariance matrix. 
    In particular, for $\Ndist(\vmu,\Sigma)$, axis-aligned hyperrectangles in the eigenbasis of $\Sigma$, ellipsoids of the form \(\{\vx \colon (\vx-\vmu)^T\Sigma^{-1}(\vx-\vmu)\leq c\}\), and unions or intersections of such sets all admit a closed-form expressions for their probability under $\Ndist(\vmu,\Sigma)$ \citep{genz2009computation}. 
\end{remark}
As a consequence of Remark~\ref{remark:tractable_signature}, for a Gaussian mixture distribution with components with different eigenbases of their covariance, we apply the signature operation separately to each Gaussian distribution in the mixture. Specifically, for a Gaussian mixture distribution $p = \sum_{i=1}^N \evpi{i} p_i$ and a set of signature locations $\bm\sC= \{\sC_i\}_{i=1}^N$, we define the \emph{component-wise signature} of \(p\) as\footnote{An alternative approach would be to apply the same signature locations to all components in the mixture, and approximate the resulting discrete distribution via numerical integration. However, this becomes intractable when component covariances differ significantly.}
\begin{equation}\label{eq:component_wise_signature}
    \bm\signature_{\bm\sC} \# p \coloneqq 
    \sum_{i=1}^N \pi^{(i)} \signature_{\sC_i} \# p_i.
\end{equation}
In the next subsection, we show how to efficiently bound $\wasserstein_2(p, \bm\signature_{\bm\sC} \# p)$.
}

\begin{remark}
    The approximation of a continuous distribution by a discrete distribution, also called a particle- or \changed{quantization} approximation, is a well-known concept in the literature \changed{
    \citep{graf2000foundations, pages2012optimal}}. The notion of the signature of a probability distribution introduced here is unique in that the discrete approximation is fully defined by the Voronoi partition of the support of the continuous distribution, thereby connecting the concept of signatures to semi-discrete optimal transport \citep{peyre2017computational}. 
\end{remark}

\subsection{Wasserstein Bounds for the Signature Operation}\label{sec:wasser4signatures}
The computation of $W_2(p,\signature_\sC\#p)$ requires solving a semi-discrete optimal transport problem, where we need to find the optimal transport plan from each point $\vz$ to a specific location $\vc_i \in \sC$ \citep{peyre2017computational}.
Luckily, as illustrated in the following proposition, which is a direct consequence of 
\changed{Lemma~3.1 in \citep{canas2012learning}},
% Theorem 1 in \citep{ambrogioni2018wasserstein}, 
for the case of a signature of a probability distribution, the resulting transport problem can be solved exactly.
\begin{proposition}\label{prop:OTSignature}
    \changed{Let \(\rho \geq 1\), then} for a probability distribution $p\in\probMeas_\changed{\rho}(\eucl^n)$, and signature locations $\sC=\{\vc_i\}_{i=1}^N\subset\eucl^n$,
    we have that
    \begin{equation*}
        \wasserstein\changed{_\rho^\rho}(p, \signature_\sC\#p) = \sum_{i=1}^N \evpi{i}\expect_{\vz\sim p}[\|\vz-\vc_i\|^\changed{\rho}\mid \vz \in \sR_i].
    \end{equation*}
\end{proposition}
According to Proposition~\ref{prop:OTSignature}, transporting the probability mass at each point $\vz$ to a discrete location $\vc_i$ based on the Voronoi partition of $\eucl^d$ guarantees that the cost $\|\vz-\vc_i\|$ is the smallest and leads to the smallest possible transportation cost, i.e., the optimal transportation strategy for the Wasserstein Distance. 
Proposition~\ref{prop:OTSignature} is general and guarantees that for any distribution $p\in \probMeas_2(\eucl^n)$, $\wasserstein_2(p, \signature_\sC\#p)$ only depends on 
\changed{the probability mass and the conditional $2$-moment of $p$} w.r.t.\ regions $\sR_i\in\{\sR_1,...,\sR_N \}$.
In the following proposition, we show how closed-form expressions for \changed{$\wasserstein_2^2(p, \signature_\sC\#p)$} can be derived for univariate Gaussian distributions. This result is then generalized in Corollary~\ref{corol:Was4SignMultGaus} to general multivariate Gaussian distributions.
\begin{proposition}\label{prop:Was4SignStdGaus}
    For $p=\Ndist(0,1)$, signature locations $\sC=\{c_i\}_{i=1}^N\subset\eucl$ and associated partition $\{\sR_i\}_{i=1}^n$ with $\sR_i=[l_i,v_i]\subseteq\eucl\cup\{-\infty,\infty\}$ for each $i$, it holds that
    \changed{
    \begin{equation}\label{eq:Was4SignStdGauss}
        \wasserstein_2^2(\signature_{\sC}\#\Ndist(0,1),\Ndist(0,1)) = 
        \sum_{i=1}^N\evpi{i}\left[\nu_i + (\mu_i - c_i)^2\right]
    \end{equation}
    with,
    \begin{equation}\label{eq:Was4SignStdGausTerms}
        \evpi{i} =\Phi(u_i) - \Phi(l_i), \quad
        \mu_i = \frac{1}{\evpi{i}}[\phi(l_i) - \phi(u_i)], \quad
        \nu_i = 1+\frac{1}{\evpi{i}}[l_i\phi(l_i) - u_i\phi(u_i)]  -\mu_i^2 
    \end{equation}
    where $\phi$ and $\Phi$ are the pdf and cdf of a standard (univariate) Gaussian distribution, respectively, i.e., $\phi(x)=\tfrac{1}{\sqrt{2\pi}}\exp\left(-\tfrac{x^2}{2}\right)$ and $\Phi(x)=\tfrac{1}{2}\left[1+\erf{\tfrac{x}{\sqrt{2}}}\right]$.}
\end{proposition}
In Proposition \ref{prop:Was4SignStdGaus}, $\mu_i$ and $\nu_i$ are the mean and variance of a standard univariate Gaussian distribution restricted on $[l_i,u_i]$.
Note that some of the regions in the partition will be unbounded. However, as the standard Gaussian distribution exponentially decays to zero for large $x$, i.e., $\lim_{x\rightarrow \infty}\exp(x^2)\Ndist(x\mid 0, 1)=0$, it follows that the bound in Eqn~\eqref{eq:Was4SignStdGauss} is finite even if some of the regions in $\{\sR_i\}_{i=1}^N$ are necessarily unbounded.

\begin{figure}[htbp]
    \centering
    \includegraphics[width=0.68\textwidth]{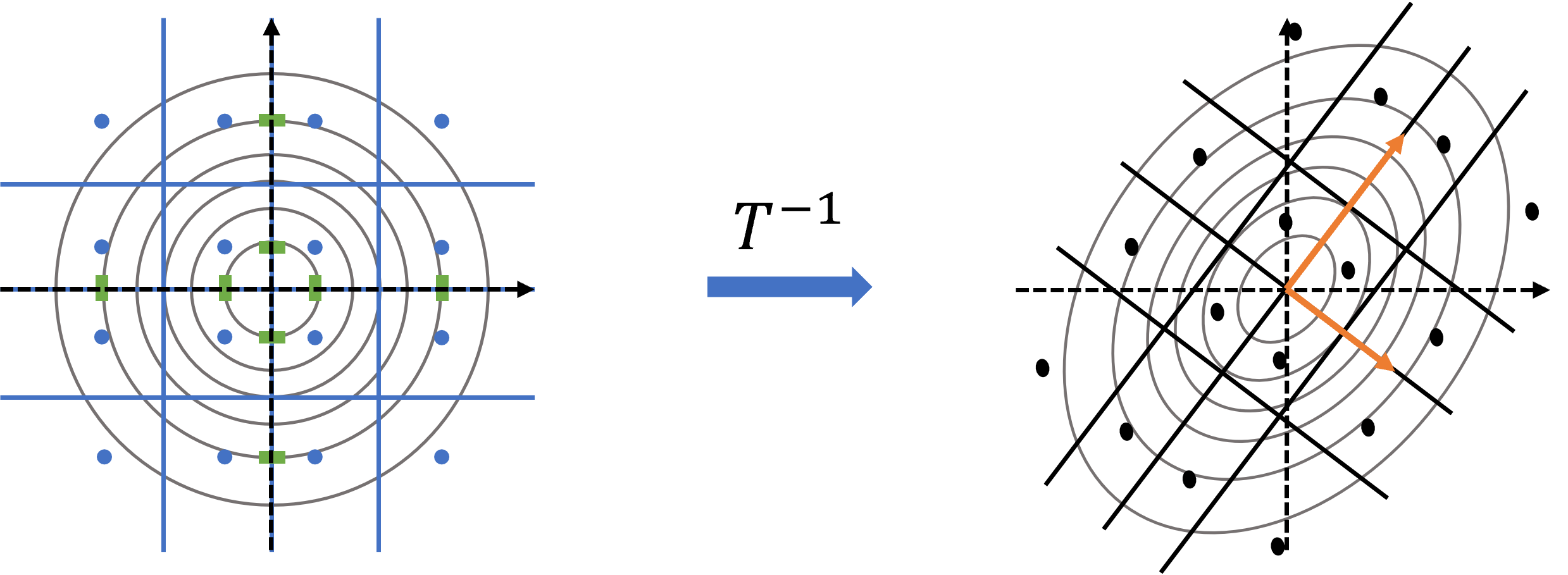}
    \caption{Construction of the signature of a 2D Gaussian distribution (black dots) using the signature of a standard Gaussian distribution (blue dots) as a template. In the space induced by $\mT$, that is, the basis of the covariance matrix (orange arrows) of the Gaussian distribution, all dimensions of the Gaussian distribution are independent. Hence, we can take the cross-product of the signatures of each univariate Gaussian marginal in dimension $j$ of the transformed space with signature locations $\sC^j$ (green stripes) as a template. The signature in the original space is then obtained by taking the post image of the template under $\mT^{-1}$. The blue and black lines represent the edges of the Voronoi partition associated with the signature locations in the transformed and original space, respectively.}
    \label{fig:signatureGaussian}
\end{figure}

A corollary of Proposition~\ref{prop:Was4SignStdGaus} is Corollary~\ref{corol:Was4SignMultGaus}, where we extend Proposition~\ref{prop:Was4SignStdGaus} to general multivariate Gaussian distributions under the assumption that locations $\sC$ are such that sets $\{\sR_{i}\}_{i=1}^N$ define a grid in the transformed space induced by the basis of the covariance matrix that we denote as $\mT$.
As illustrated in Figure~\ref{fig:signatureGaussian}, 
it is always possible to satisfy this assumption by taking $\sC$ as the image of an axis-aligned grid of points under transformation $\mT^{-1}$, where $\mT^{-1}$ can be computed via an eigendecomposition \citep{kasim2020derivatives} or Cholesky decomposition \citep{davis2016survey} of the covariance matrix.\footnote{Given the eigenvalues vector $\vlambda$ and eigenvector matrix $\mV$ of a covariance matrix $\vNdist$, we can take $\mT^{-1}=\mV\diag(\vlambda)$. In the case of a degenerate multivariate Gaussian, where $\vNdist$ is not full rank, we can take $\sC$ as the image of an axis-aligned grid of points in a space of dimension $\rank(\vNdist)$, under transformation $\mT^{-1}=\elem{(\mV\diag(\vlambda)^{\frac{1}{2}})}{:,1:\rank(\vNdist)}$.}

\begin{corollary}[of Proposition \ref{prop:Was4SignStdGaus}]\label{corol:Was4SignMultGaus}
    For $\Ndist(\mNdist,\vNdist)\in\GMM(\eucl^n)$, let matrix $\mT\in\eucl^{n\times n}$ be such that $\mT=\diag(\vlambda)^{-\frac{1}{2}}\mV^T$ where $\diag(\vlambda)=\mV^T\vNdist\mV$ is a diagonal matrix whose entries are the eigenvalues of $\vNdist$, and $\mV$ is the corresponding orthogonal (eigenvector) matrix. Further, let $\sC=\{\vc_i\}_{i=1}^N\subset\eucl^n$ be a set of signature locations on a grid in the transformed space induced by $\mT$, i.e.,
    \begin{equation}\label{eq:AxisAlignedGridAssumption}
        \changed{\postImage(\sC- \{\mNdist\} ,\mT)} = \sC^1 \times \sC^2 \times \hdots \times\sC^n
    \end{equation}
    with \changed{$\sC^j=\{c_i\}_{i=1}^{N_j}$}
    the set of signature locations in the transformed space for dimension $j$.
    Then,
    \begin{equation*}%\label{eq:Was4SignMultGaus}
        \wasserstein^2_2(\signature_\sC\#\Ndist(\mNdist,\vNdist),\Ndist(\mNdist,\vNdist)) =\sum_{j=1}^n \evlambda{j}\wasserstein^2_2\left(\signature_{\sC^j}\#\Ndist(0,1),\Ndist(0,1)\right).
    \end{equation*} 
\end{corollary}
\changed{
In the case of Gaussian mixture distributions, i.e., where $p=\sum_{i=1}^M \elem{\tilde{\pi}}{i} \Ndist(m_i,\Sigma_i)$,
and the component-wise signature operation defined in Eqn~\eqref{eq:component_wise_signature}, i.e., $\bm\signature_{\bm\sC}$ with $\bm\sC=\{\sC_{i=1}^N\}$, we have that
\begin{equation}\label{eq:Was4SignMixMultGaus}
    \wasserstein_2^2\left(p,\bm\signature_{\bm\sC}\#p\right)
    \leq \sum_{i=1}^M\elem{\tilde{\pi}}{i}\wasserstein_2^2\left(p_i,\signature_{\sC_i}\#p_i\right) \eqqcolon \hat\wasserstein_{2,\bm\signature_{\bm\sC}}^2(p).
\end{equation}
That is, for a mixture of Gaussian distributions, the $2$-Wasserstein distance between the mixture and its component-wise signature is bounded by the weighted sum of the 2-Wasserstein distances between each Gaussian component and its signature.\footnote{This trivial result is a special case of Lemma~\ref{lem:wasser4Mixtures} in the Appendix on the Wasserstein distance between mixtures.} Note that, under the assumption that the signature locations of each component satisfy Eqn~\eqref{eq:AxisAlignedGridAssumption}, this upper bound admits a closed form as given by Corollary~\ref{corol:Was4SignMultGaus}.}

\begin{remark}\label{remark:optimalSignature}
    For a multivariate Gaussian $p=\Ndist(\mNdist,\vNdist)$, the set $\sC$ of $N>1$ signature locations that minimize $\wasserstein_2(\signature_\sC\#p,p)$ is generally non-unique, and finding any such set is computationally intractable \citep{graf2000foundations}. However, for univariate Gaussians, an optimal signature placement strategy exists as outlined in \cite{pages2003optimal} and will be used in Subsection \ref{sec:algo:gmmApproximateSNN} to construct grid-constrained signature locations $\sC$ for multivariate Gaussians.
\end{remark}

\changed{
\subsection{Convergence Rates for Signature Approximations of Gaussian Mixtures}
In the previous section, we showed how we can obtain a closed form expression for $\wasserstein_2(p,\signature_\sC\#p)$ when $p$ is a multivariate Gaussian distribution, 
and that this expression can be used to upper bound the $2$-Wasserstein distance error from the component-wise signature $\bm\signature_{\bm\sC}\#p$ when $p$ is a mixture of Gaussian distributions.
Here, we show that this error converges uniformly to zero as the number of signature locations increases. We start with the case where $p$ is a univariate Gaussian distribution, and then extend to mixtures of Gaussian distributions. 
To our knowledge, the resulting uniform, non‑asymptotic convergence rates are novel. Indeed, existing works on quantization for Gaussian measures (e.g., \citealp{graf2000foundations,pages2003optimal}) typically provide asymptotic rates and existence results for optimal $N$-point quantizers, rather than the explicit uniform‑in‑$N$ error bounds derived here.

\begin{proposition}\label{prop:Was4SignStdGausUniformGrid}
    For $p=\Ndist(0,1)$, consider $N\in\natNum$ signature locations on a uniform grid centered at zero with spacing $(2\sqrt{\log N})/N$ between consecutive points,\footnote{We choose the spacing $(2\sqrt{\log N})/N$ as it balances the discretization error in the central interval and the Gaussian tail mass. Other spacings trade off these contributions differently: for example, a spacing of $(\log N)/N$ would reduce the remainder term $r(N)$ but increase the leading term to $(\log^2 N)/N^2$.} that is
    \begin{equation}\label{eq:uniform1DGrid}
        \sC=\left\{\frac{2\sqrt{\log N}}{N}\left(2i - N -1\right) \colon  i \in \{1,\hdots,N\}\right\}.
    \end{equation}
    Then, for all $N\in\natNum$,
    \begin{equation}\label{eq:Was4SignStdGausUniformGrid}
        \wasserstein_2^2(\signature_\sC\#\Ndist(0,1),\Ndist(0,1)) \leq 
        \frac{4\log N}{N^2} + r(N),
    \end{equation}
    where the remainder term
    \begin{equation}\label{eq:uniformGridRemainder}
        r(N) = \begin{cases}
            1, & N=1, \\
            \frac{4}{N^{2}\sqrt{2\pi}}\left(\sqrt{\log N}+\frac{1}{4\sqrt{\log N}}\right), & N\geq 2.
        \end{cases}
    \end{equation}
\end{proposition}
The proof of Proposition~\ref{prop:Was4SignStdGausUniformGrid} (given in the Appendix) follows by applying Proposition~\ref{prop:Was4SignStdGaus} to write the $\wasserstein_2$ error as a sum over the Voronoi partition induced by \(\sC\), then bound each interval's contribution by its squared radius and control the unbounded intervals using Gaussian tail bounds (Mills' ratio).

\begin{remark}\label{remark:asymptoteRemainder}
    For $N\ge 2$, the remainder term of the bound of Proposition~\ref{prop:Was4SignStdGausUniformGrid} (Eqn~\ref{eq:uniformGridRemainder}) is
    \[
        r(N)= \tfrac{4\sqrt{\log N}}{N^2\sqrt{2\pi}} + O\!\left(\tfrac{1}{N^2\sqrt{\log N}}\right).
    \]
    Hence $r(N)=O\!\left(\tfrac{\sqrt{\log N}}{N^2}\right)$ and $r(N)=o\!\left(\tfrac{\log N}{N^2}\right)$, so the bound in Eqn~\eqref{eq:Was4SignStdGausUniformGrid} is dominated by its first term $\tfrac{4\log N}{N^2}$, and consequently
    \[
        \wasserstein_2^2(\signature_\sC\#\Ndist(0,1),\Ndist(0,1)) \le \frac{4\log N}{N^2}\bigl(1+o(1)\bigr).
    \]
    That is, for a uniform grid, the squared $W_2$ error decays on the order of $(\log N)/N^2$. 
    In one dimension, Theorem~6.2 of \cite{graf2000foundations} implies that the minimal achievable squared $\wasserstein_2$ error of any $N$-location signature approximating a univariate Gaussian decays as $1/N^2$ up to constants. 
    Thus the uniform grid in Proposition~\ref{prop:Was4SignStdGausUniformGrid} attains the optimal polynomial rate $1/N^2$ up to a single logarithmic factor. The $\log N$ term arises from using a uniform construction that provides an explicit bound for all $N$, with a remainder of strictly lower order than the leading term. In practice, adaptive (non-uniform) placement (Remark~\ref{remark:optimalSignature}) can remove this logarithmic factor and improve the constants in the error bound, as illustrated in Figure~\ref{fig:wasserGMMs}.
\end{remark}

\begin{figure}[htbp]
    \centering
    \begin{subfigure}[b]{0.49\textwidth}
        \includegraphics[width=\textwidth]{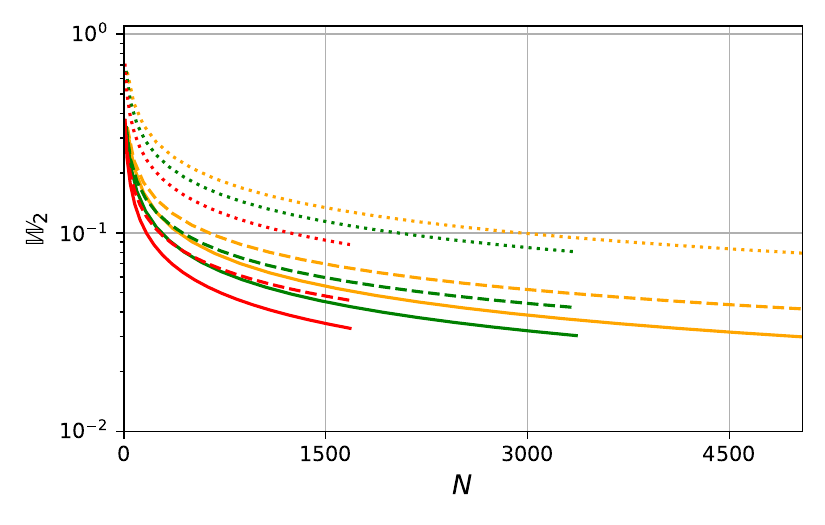}
        \caption{Formal upper bounds on $\wasserstein_2$ for uniform and adaptive grid-constrained signatures.}    
        \label{fig:formalConvergenceRate}
    \end{subfigure}
    \hfill
    \begin{subfigure}[b]{0.49\textwidth}
        \includegraphics[width=\textwidth]{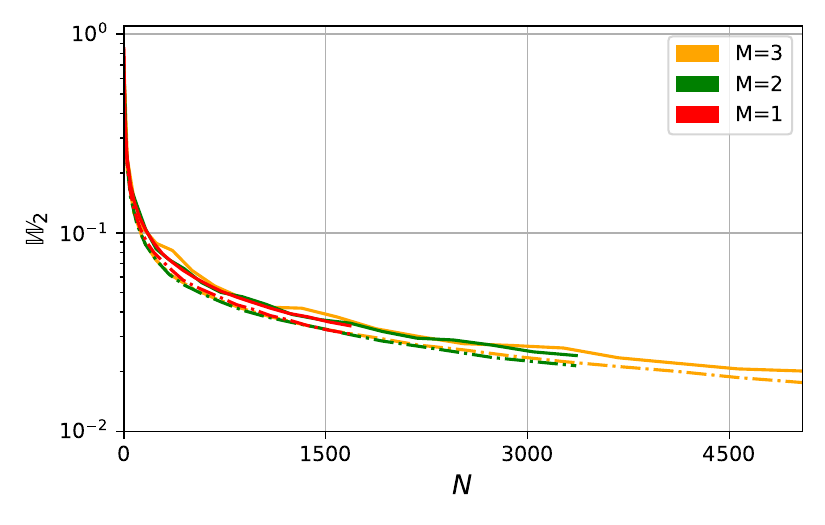}
        \caption{Empirical $\wasserstein_2$ comparisons of grid‑constrained and approximate optimal signatures.}
        \label{fig:EmpiricalConvergenceRate}
    \end{subfigure}
    \caption{\changed{
    The 2-Wasserstein distance between a 2D Gaussian mixture distribution with $M$ components and signature approximations with $N$ locations.
    Solid lines correspond to signatures with non-uniform grid‑constrained locations obtained via Algorithm~\ref{al:signaturesOfGaussianMixutres}, while dashed lines correspond to signatures with uniform grid-constrained locations as in Corollary~\ref{corol:WasConvergence4SignMixtureGaus}. 
    Dash-dotted lines correspond to signatures with locations given by the cluster centers obtained via $N$-Means on $10^5$ Monte Carlo samples from the GMM, which provide an approximation to the optimal signature locations \citep{pages2003optimal}.
    For grid‑constrained signatures, formal bounds are computed using Algorithm~\ref{al:signaturesOfGaussianMixutres}, while empirical estimates are computed from $10^4$ samples from each distribution by computing the 2‑Wasserstein distance between the resulting empirical distributions. 
    The dotted reference curves show the convergence rate from Corollary~\ref{corol:WasConvergence4SignMixtureGaus}, which corresponds to the uniform grid.
    For comparability, each mixture has an expected squared $\ell_2$ norm of one, so the reported distance equals the relative 2‑Wasserstein distance defined in Eqn~\eqref{eq:relative2wasserstein}.
    }}    
    \label{fig:wasserGMMs}
\end{figure}

\newpage
\noindent
Analogously to how Proposition~\ref{prop:Was4SignStdGaus} extends to mixtures of Gaussian distributions, Proposition~\ref{prop:Was4SignStdGausUniformGrid} extends to mixtures of multivariate Gaussian distributions, as shown in the following corollary.
\begin{corollary}[of Proposition~\ref{prop:Was4SignStdGausUniformGrid}]\label{corol:WasConvergence4SignMixtureGaus}
    Let $p=\sum_{i=1}^M \elem{\tilde{\pi}}{i} \mathcal{N}(m_i,\Sigma_i)$. For each component $i$ of $p$, define a set of signature locations $\sC_i\in\eucl^n$ as in Eqn~\eqref{eq:AxisAlignedGridAssumption}, using uniform one-dimensional grid of points $\sC_i^j$ of size $N_{i,j}$ as in Eqn~\eqref{eq:uniform1DGrid} for each dimension $j$.
    \begin{equation*}
        \wasserstein_2^2\left(p,\bm\signature_{\{\sC_i\}_{i=1}^M}\#p\right) \leq 
        \sum_{i=1}^M \elem{\tilde{\pi}}{i}
    \sum_{j=1}^n \elem{\vlambda_i}{j}\left(\frac{4\log N_{i,j}}{N_{i,j}^2} + r(N_{i,j})\right)
    \end{equation*}    
    where $\vlambda$ is the vector of eigenvalues of $\Sigma_i$. 
\end{corollary}
Corollary \ref{corol:WasConvergence4SignMixtureGaus} follows from Eqn~\eqref{eq:Was4SignMixMultGaus}, which decomposes the mixture error into a mixture-weighted sum of component errors. For each component we diagonalize $\Sigma_i$, apply Corollary~\ref{corol:Was4SignMultGaus} to separate dimensions, and then invoke Proposition~\ref{prop:Was4SignStdGausUniformGrid} on each one-dimensional uniform grid of points $\sC_i^j$, yielding the closed-form expression for the error bound after summing over $i$ and $j$. 

In Figure~\ref{fig:wasserGMMs}, we examine the convergence rate of the approximation error for 2D Gaussian mixtures decays as the number of signature locations increases, comparing three constructions: uniform Cartesian product grids (Corollary~\ref{corol:WasConvergence4SignMixtureGaus}), adaptive (non‑uniform) grids, and an empirical proxy for the optimal signature locations. Figure~\ref{fig:formalConvergenceRate} confirms that adaptive grids substantially reduce the $\wasserstein_2$ error relative to the uniform case (see Remark~\ref{remark:asymptoteRemainder}). Figure~\ref{fig:EmpiricalConvergenceRate} further shows that adaptive placement already captures most of the attainable improvement, with only modest additional gains from unconstrained (cluster-based) locations for the distinct-mode mixtures considered. Appendix~\ref{append:ExperimentsSigns4GMMs} provides further discussion and additional experiments on the conservatism of the component‑wise signature bound in Eqn.~\eqref{eq:Was4SignMixMultGaus}.

}

\section{Stochastic Neural Networks as Gaussian Mixture Models}\label{sec:approxPredPostByGMM}
In this section, we detail and formalize our approach as illustrated in Figure \ref{fig:method} to iteratively approximate a SNN with a GMM \changed{at a finite set of input points $\sX$}. 
We first consider the case where $\sX=\{\vx\}$, i.e., we approximate the distribution of a SNN over a single input point. The extension \changed{to} finite set of input points \changed{is} considered in Section \ref{sec:extendToInputSets}. 
Last, in Section \ref{sec:ConvergenceAnalysis}, we prove that $q_{nn(\sX)}$, i.e., the Gaussian mixture approximation resulting from our approach, \changed{converges to $p_{nn(\sX)}$ in Wasserstein distance}. That is, the error between $p_{nn(\sX)}$ and $q_{nn(\sX)}$ can be made arbitrarily small by increasing the number of signature \changed{locations} and GMM components.

\subsection{Gaussian Mixture Model Approximation of a Neural Network}\label{subsec:SNN2GMM}
As shown in Eqn~\eqref{eq:p_nn(x),k}, $p_{nn(\vx)}$, the output distribution of a SNN at input point $\vx$, can be represented as a composition of $K$ stochastic operations, one for each layer. 
To find a Gaussian mixture distribution $q_{nn(\vx)}$ that approximates $p_{nn(\vx)}$, our approach is based on iteratively approximating each of these operations with a GMM, as illustrated in Figure \ref{fig:method}.
Our approach can then be formalized as in Eqn \eqref{eq:q_nn(x),k} below for $ k\in\{1,\hdots,K\}$:
\begin{subequations}\label{eq:q_nn(x),k}
    \begin{align}
        q_{nn(x),1} &= \expect_{\state_0\sim \delta_\vx}[\changed{L^{\vw_0}(\state_0)\#p_{\vw_0}}] & \text{(Initialization and \changed{Affine} operation)}  \\
        d_{nn(x),k} &= \changed{\bm\signature_{\bm\sC_k}}\#\compress_{\changed{M_k}}(q_{nn(\vx),k})  & \text{(Compression and Signature operations)} \label{eq:q_nn(x),k:Compress&Signature} \\
        q_{nn(x),k+1} &= \expect_{\state_{k}\sim d_{nn(\vx),k}}[\changed{L^{\vw_{k}}(\act(\state_{k}))\#p_{\vw_k}}] & \text{(Activation and \changed{Affine} operations)} \label{eq:q_nn(x),k:actiLinOperation} \\ 
        q_{nn(\vx)} &= q_{nn(\vx),K+1}& \text{(Output)}\label{eq:q_nn(x),output}
    \end{align}
\end{subequations}
where \changed{$\bm\sC_k=\{\sC_{k,i} \subset \eucl^{n_k}\}_{i=1}^{M_k}$ denotes} the signature locations for layer $k$. $\compress_{M_k}:\GMM_{\infty}(\eucl^{n_{k}})\rightarrow\GMM_{\changed{M_k}}(\eucl^{n_k})$ is \changed{a} compression operation that compresses the Gaussian mixture distribution $q_{nn(\vx),k}$ into a Gaussian Mixture distribution with at most \changed{$M_k$} components, thus limiting the approximation's complexity. We show how operation $\compress_{\changed{M_k}}$ is performed using moment matching in Section \ref{sec:algo:gmmApproximateSNN}. 

Eqn~\eqref{eq:q_nn(x),k} consists of the following steps: $q_{nn(x),1}$ is the distribution resulting from an \changed{affine} combination of the input $x$ with the \changed{parameters} at the first layer, $d_{nn(x),k}$ is the result of a compression and signature operation of $q_{nn(\vx),k}$, the output of the previous layer. The approximate output distribution $q_{nn(x),k+1}$ at each layer $k$ is obtained by marginalizing 
\changed{$L^{\vw_k}(\act(\state_k))\#p_{\vw_k}$} w.r.t.\ $d_{nn(x),k}$. We recall that for any fixed $\state_k$, 
\changed{and Gaussian distribution $p_{\vw_k}$, $L^{\vw_k}(\act(\state_k))\#p_{\vw_k}$ is Gaussian, as Gaussian distributions are closed w.r.t. affine combination.}
Consequently, as for all $k$, $d_{nn(x),k}$ is a discrete distribution, $q_{nn(x),k+1}$ is a Gaussian mixture distribution with a size equal to the support of $d_{nn(x),k}$.

\changed{
\begin{example}\label{example:closedFormMixtureFeedforward}
    Consider the case where $L^{\vw_k}$ is a fully-connected layer with weight matrix \(\mW \in \realNum^{n_{k+1} \times n_k}\) and bias vector \(\vb \in \realNum^{n_{k+1}}\), i.e., \(L^{\vw_k}(\vz) = \mW \vz + \vb\). Assume \(\vectorize(\mW)\) and \(\vb\) are distributed according to \(\Ndist(\vectorize(\mM), \Sigma)\) and \(\Ndist(\vm_{\vb}, \Sigma_{\vb})\), respectively, where \(\vectorize(\cdot)\) stacks the columns of a matrix into a single vector. 
    Then, the approximate output distribution at the \((k+1)\)-th layer, \(q_{nn(x),k+1}\), as defined in Eqn~\eqref{eq:q_nn(x),k:actiLinOperation} for $M_k=1$, signature locations \(\sC_{k,1} = \{\vc_i\}_{i=1}^{N_k}\) and partition \(\{\sR_{i}\}_{i=1}^{N_k}\) from Eqn~\eqref{eq:VoronoiPartitionOfSignature}, can be written as
    \begin{equation*}%\label{eq:example_q_nn(x),k+1}
        q_{nn(x),k+1} = 
        \sum_{i=1}^{N_k}
        \prob_{\vz \sim \compress_{1}(q_{nn(x),k})}[\vz \in \sR_{i}]\cdot
        \Ndist\Bigg(
            \vm_{\vb} + \mM \act(\vc_i),
            \Sigma_{\vb} + \sum_{j,l = 1}^{n_{k+1}} \elem{\act(\vc_i)}{j}\elem{\act(\vc_i)}{l}\Sigma_{jl}
        \Bigg)
    \end{equation*}
    where \(\Sigma_{jl}\) denotes the \((j,l)\)-th block of the covariance matrix \(\Sigma\), each block of size \(n_k \times n_k\), i.e., the covariance between the \(j\)-th and \(l\)-th columns of \(\mW\).
\end{example}
}

\changed{
Note that the approximation scheme in Eqn~\eqref{eq:q_nn(x),k} allows for an arbitrary choice of signature locations $\sC_{k,i}$ for the $i$-th component at each layer $k$. To obtain a closed-form expression for \(q_{nn(\vx)}\), the locations $\bm\sC_k$ should be selected as discussed in Remark~\ref{remark:tractable_signature}; that is, so that 
\(\prob_{\vz \sim \compress_{1}(q_{nn(x),k})}[\vz \in \sR]\) in Example~\ref{example:closedFormMixtureFeedforward} is analytically tractable for all regions $\sR$.
% the corresponding partitions admit tractable probability integrals.
}

\begin{remark}\label{remark:GenerlizationAlgoNonGaus}
    The approach described in Eqn \eqref{eq:q_nn(x),k} leads to a Gaussian mixture distribution $q_{nn(\vx)}$ under the assumption that the \changed{parameter distributions $p_{\vw_k}$ are} Gaussian. 
    In the more general case, where $p_{\vw_k}$ is non-Gaussian, one can always recursively approximate \changed{$L^{\vw_k}(\vz)\#p_{\vw_k}$} for each $k$ by a discrete distribution using the signature operation. That is, \changed{$L^{\vw_k}(\vz)\#p_{\vw_k}$} in Eqn~\eqref{eq:q_nn(x),k:actiLinOperation} is replaced with \changed{$L^{\vw_k}(\vz)\#\signature_{\sC_{\vw_k}}\#p_{\vw_k}$}, where \changed{$\sC_{\vw_k}$} is the set of signature locations.
    Applying this additional operation, Eqn \eqref{eq:q_nn(x),k} will lead to a discrete distribution for $q_{nn(\vx)}$. 
\end{remark}

\subsection{Wasserstein Distance Guarantees}\label{subsec:Wasser4SNN2GMM}
Since $q_{nn(x)}$ is built as a composition of $K$ iterations of the stochastic operations in Eqn \eqref{eq:q_nn(x),k}, to bound $\wasserstein_2\left(p_{nn(x)}, q_{nn(x)}\right)$ we need to \changed{formally} quantify the error introduced by each of the steps in Eqn \eqref{eq:q_nn(x),k}. A summary of how each of the operations occurring in Eqn \eqref{eq:q_nn(x),k} modifies the Wasserstein distance between two distributions is provided in Table \ref{tab:wasser4layers} (details are given in Appendix \ref{append:wasser4SNNs}). These bounds are then composed using interval arithmetic to bound $\wasserstein_2\left(p_{nn(x)}, q_{nn(x)}\right)$ in the following theorem.

\begin{theorem}\label{thm:WasserNetwork}
    Let $p_{nn(\vx)}$ be the output distribution of a SNN with $K$ hidden layers for input $\vx\in\eucl^{n_0}$ and $q_{nn(\vx)}$ be a Gaussian mixture distribution built according to Eqn~\eqref{eq:q_nn(x),k}. Iteratively define $\hat{\wasserstein}_{2}$ as
    \changed{
    \begin{subequations}\label{eq:hat{wasserstein}_k}
        \begin{align}
            &\hat{\wasserstein}_{2,1}=0, \label{eq:hat{wasserstein}_k,init} \\
            &\hat{\wasserstein}_{2,k+1} = 
            \lipschitz_k\left[\hat{\wasserstein}_{2,k} + \mw_2\left(q_{nn(\vx),k}, \compress_{\changed{M_k}}(q_{nn(\vx),k})\right) 
            \right.\nonumber \\ &\hspace{5cm} \left.
            +\hat\wasserstein_{2,\bm\signature_{\bm\sC_k}}\left(\compress_{\changed{M_k}}(q_{nn(\vx),k})\right) \right],\qquad \forall k\in\{1,\hdots,K\},  \label{eq:hat{wasserstein}_k,update} \\
            &\hat{\wasserstein}_{2}=\hat{\wasserstein}_{2,K+1},   \label{eq:hat{wasserstein}_k,out}
        \end{align}
    \end{subequations}
    where $\hat\wasserstein_{2,\bm\signature_{\bm\sC}}$ is defined in Eqn~\eqref{eq:Was4SignMixMultGaus}, $\mw_2$ as defined in Eqn~\eqref{eq:definitionMW2}, and 
    \begin{equation}\label{eq:lipschitz_k}
        \lipschitz_{k} \coloneqq \lipschitz_\act\expect_{\vw_k\sim p_{\vw_k}}[\lipschitz_{L^{\vw_k}}^2]^{\frac{1}{2}},
    \end{equation}
    with \(\lipschitz_\sigma\) and \(\lipschitz_{L^{\vw_k}}\) the Lipschitz constants of the activation function \(\act\) and affine operation \(L^{\vw_k}\), respectively.} 
    Then, it holds that 
    \begin{equation*}%\label{eq:WasserNetworkFinal}
        \wasserstein_2\left(p_{nn(x)}, q_{nn(x)}\right) \leq \hat{\wasserstein}_2(p_{nn(x)}, q_{nn(x)}) \coloneqq \hat{\wasserstein}_{2}.
    \end{equation*}
\end{theorem}
The proof of Theorem~\ref{thm:WasserNetwork} is reported in Appendix~\ref{append:wasser4SNNs} and is based on the \changed{triangle} inequality property of the 2-Wasserstein distance, which allows us to bound $\wasserstein_2\left(p_{nn(\vx)}, q_{nn(\vx)}\right)$ iteratively over the hidden layers. 
\changed{Note that Theorem~\ref{thm:WasserNetwork} depends on the quantities $\expect_{\vw_k\sim p_{\vw_k}}[\lipschitz_{L^{\vw_k}}^2]^{1/2}$, $\hat\wasserstein_{2,\bm\signature_{\bm\sC_k}}\left(\compress_{M_k}(q_{nn(\vx),k})\right)$, and
$\mw_2\left(q_{nn(\vx),k}, \compress_{M_k}(q_{nn(\vx),k})\right)$, 
which can be computed as follows.
First, $\changed{\expect_{\vw_k\sim p_{\vw_k}}[\lipschitz_{L^{\vw_k}}^2]^{1/2}}$, is the expectation of any Lipschitz constant $\lipschitz_{L^\vw}$ (any constant satisfying the Lipschitz inequality, not necessarily the minimal one) of the affine operation $L^{\vw_k}$, for which closed-form expressions for both feedforward and convolutional layers are given in Appendix~\ref{append:expected_lipscthiz}.
Second, $\hat\wasserstein_{2,\bm\signature_{\bm\sC_k}}\left(\compress_{M_k}(q_{nn(\vx),k})\right)$ can be computed using Corollary~\ref{corol:Was4SignMultGaus}, when the sets of signature locations satisfy the grid-regularity constraint in Eqn~\eqref{eq:AxisAlignedGridAssumption}.
Finally, $\mw_2\left(q_{nn(\vx),k}, \compress_{\changed{M_k}}(q_{nn(\vx),k})\right)$ is obtained by computing the $\mw_2$-distance between Gaussian mixtures distributions (see Definition~\ref{def:mw2}), which requires solving a finite discrete linear program with $M_k$ times the number of components of $q_{nn(\sX),k}$ optimization variables and coefficient given by the pairwise Wasserstein distances between the mixture’s Gaussian components.
}

\changed{
\begin{remark}
    If \(p_{\vw_k}\) is non-Gaussian, the additional approximation step introduced in Remark~\ref{remark:GenerlizationAlgoNonGaus} is accounted for by an extra application of the triangle inequality, which introduces an additional error term \(\wasserstein_2(p_{\vw_k}, \signature_{\sC_{\vw_k}}\#p_{\vw_k})\) in Eqn~\eqref{eq:hat{wasserstein}_k,update}. While the conditional expectations in Proposition~\ref{prop:OTSignature} may no longer admit closed-form expressions in this case, they can always be conservatively estimated using approximate numerical integration.
\end{remark}}

\changed{
\begin{remark}
    In Theorem~\ref{thm:WasserNetwork}, instead of bounding the error from the signature approximation of $\compress_{\changed{M_k}}(q_{nn(\vx),k})$ in Eqn~\eqref{eq:hat{wasserstein}_k,update}, propagated through the activation function $\act$, using the global Lipschitz constant $\lipschitz_\act$ of $\act$, one can alternatively use a bound based on the local Lipschitz constant \(\lipschitz_{\act\mid\sR_i}\) w.r.t. region $\sR_i$, as reported in Table~\ref{tab:wasser4layers}. Specifically, one can replace $\lipschitz_\act\hat\wasserstein_{2,\bm\signature_{\bm\sC_k}}\left(\compress_{\changed{M_k}}(q_{nn(\vx),k})\right)$ with 
    \[
        \left(\sum_{i=1}^N\lipschitz_{\act\mid\sR_i}^2\evpi{i}\expect_{\vz\sim \compress_{\changed{M_k}}(q_{nn(\vx),k})}[\|\vz-\vc_i\|^\changed{\rho}\mid \vz \in \sR_i]\right)^{1/2}
    \]
    in Eqn~\eqref{eq:hat{wasserstein}_k,update}. 
    While generally tighter, this bound requires explicit computation of the local Lipschitz constant and the expectation w.r.t each region $\sR_i$, which limits the efficiency of the grid structure \(\sC_i\) used in Corollary~\ref{corol:Was4SignMultGaus}. 
\end{remark}}

\begin{table}[H]
    \centering
    \begin{tabular}{ll}\toprule
        Operation & Wasserstein Bounds \\ \midrule
        \changed{Affine} &
        \changed{$\wasserstein_2^2\left(\expect_{\Tilde\vz\sim p}\left[L^{\vw}(\Tilde\vz)\#p_{\vw_k}\right],\expect_{\vz\sim q}\left[L^{\vw}(\vz)\#p_{\vw}\right]\right)\leq \expect_{\vw\sim p_{\vw}}\left[\lipschitz_{L^{\vw}}^2\right]\wasserstein_2^2(p,q)$}\\
        \midrule
        Activation&$\wasserstein_2(\act\#p,\act\#q)\leq \lipschitz_{\act}\wasserstein_2(p,q)$ 
        \\ 
        &
        $\wasserstein^2_2(\act\#p,\act\#\signature_\sC\#p)\leq 
        \changed{\sum_{i=1}^N \lipschitz_{\act\mid\sR_i}^2\evpi{i}\expect_{\vz\sim p}[\|\vz-\vc_i\|^2\mid \vz \in \sR_i]}$ \\ 
        \midrule
        \makecell[l]{Signature/\\ Compression} & 
        $\wasserstein_2(p,h\#q)\leq \wasserstein_2(p,q) + \wasserstein_2(q,h\# q)$, where $h\in\{\signature_\sC, \compress_M\}$ 
        \\  \bottomrule
    \end{tabular}
    \caption{
    Summary of the bounds on the $2$-Wasserstein distance between two distributions $p,q\in\probMeas_2(\eucl^{n})$, obtained by pushing forward $p$ \changed{and $q$} through common SNN operation types. 
    Here, \changed{$\lipschitz_{L^\vw}$ denotes the global Lipschitz constant of a Lipschitz function $L^\vw$, parametrized by \(\vw\in\realNum^d\) and distributed according to \(p_\vw\in\probMeas(\realNum^d)\)}, and $\lipschitz_\sigma$ and $\lipschitz_{\sigma\mid\sR_i}$ denote, respectively, 
    the global Lipschitz constant and the Lipschitz constant over region $\sR_i$ of \changed{a Lipschitz function $\sigma$}.
    % \tablefootnote{For ReLU, $\lipschitz_{\act\mid\sR_i}=1$ for regions $\sR_i$ overlapping with the positive half-space, and $0$ else. For tanh, $\lipschitz_{\act\mid\sR_i}\in(0,1]$.} 
    Proofs of these results are provided in Appendix~\ref{append:wasser4SNNs}.
    }
    \label{tab:wasser4layers}
\end{table}

\subsection{Approximation for Sets of Input Points}\label{sec:extendToInputSets}
We now consider the case where $\sX=\{\vx_1,...,\vx_D\}$, that is, a finite set of input points, and extend our approach to  a Gaussian mixture approximation of $p_{nn(\sX)}$. 
We first note that $p_{nn(\sX)}$ can be equivalently represented by extending Eqn\ \eqref{eq:p_nn(x),k} as follows, where in the equation below $\vectorize(\sX)=(\vx_1^T,\hdots,\vx_D^T)^T$ is the vectorization of $\sX$,
\changed{
\begin{subequations}\label{eq:p_nn(sX),k}
\begin{align}
    &p_{nn(\sX),1}=\expect_{\vz_0\sim\delta_{\vectorize(\sX)}}\left[\bm{L}^{\vw_0}(\vz_0)\#p_{\vw_0}\right], \\ 
    &p_{nn(\sX),k+1}=\expect_{\vz_{k}\sim p_{nn(\sX),k}}\left[\bm{L}^{\vw_k}(\act(\vz_{k}))\#p_{\vw_k}\right], \qquad k \in \{1,\hdots,K\},   \\
    & p_{nn(\sX)} = p_{nn(\sX),K+1},
\end{align}
\end{subequations}}
where $\bm{L}^{\vw_k}: \eucl^{D \cdot n_k}\rightarrow \eucl^{D\cdot n_k}$ is the stacking of $D$ times $L^k$, i.e., 
\[
    \bm{L}^{\vw_k}\left((\vz_1^T,\hdots,\vz_D^T)^T)= (L^{\vw_k}(\vz_1)^T,\allowbreak \hdots, L^{\vw_k}(\vz_D)^T\right)^T.
\]
That is, $p_{nn(\sX)}$ is computed by stacking $D$ times the neural network for the inputs in $\sX$.
Following the same steps as in the previous section, $p_{nn(\sX)}$ can then be approximated by the Gaussian Mixture distribution $q_{nn(\sX)}$, for $k\in\{1,\hdots,K\}$ defined as
\changed{
\begin{subequations}\label{eq:q_nn(vec(X)),k}
\begin{align}
    &q_{nn(\sX),1} = \expect_{\state_0\sim \delta_{\vectorize(\sX)}}\left[\bm{L}^{\vw_0}(\state_0)\#p_{\vw_0}\right],  & \quad \text{(Initialization and Affine operation)} \label{eq:q_nn(vec(X)),k,init} \\ 
    &d_{nn(\sX),k} = \bm\signature_{\bm\sC_k}\#\compress_{\changed{M}}(q_{nn(\sX),k}),  & \quad \text{(Compression and Signature operations)} \label{eq:q_nn(vec(X)),k,compr+sign}  \\
    &q_{nn(\sX),k+1}= \expect_{\state_k\sim d_{nn(\sX),k}}\left[\bm{L}^{\vw_k}(\act(\state_k))\#p_{\vw_k}\right], &  \text{(Activation and Affine operations)} \label{eq:q_nn(vec(X)),k,act+lin} \\
    &q_{nn(\sX)} = q_{nn(\sX),K+1}, &\quad \text{(Output)} \label{eq:q_nn(vec(X)),k,out} 
\end{align}
\end{subequations}}
where \changed{$\bm\sC_k=\{\sC_{k,i}\subset\eucl^{D\cdot n_k}\}_{i=1}^{M_k}$ are the sets of signature locations} at layer $k$, and $\changed{M_k}\in\natNum$ is the compression size. Theorem~\ref{thm:WasserNetwork} implies the following \changed{corollary}, which bounds the approximation error of $q_{nn(\sX)}$.
\newpage
\begin{corollary}[of Theorem~\ref{thm:WasserNetwork}]\label{corol:WasserNetworkSet}
    Let $p_{nn(\sX)}$ be the joint distribution of a SNN with $K$ hidden layers over a finite set of input points $\sX=\{\vx_i\}_{i=1}^D\subset\eucl^{n_0}$ and 
    $q_{nn(\sX)}$ be a Gaussian mixture distribution built according to Eqn~\eqref{eq:q_nn(vec(X)),k}. Iteratively define $\hat{\wasserstein}_2$ as
    \changed{
    \begin{subequations}\label{eq:hat{wasserstein}_k_set_of_points}
        \begin{align}
            &\hat{\wasserstein}_{2,1}=0, \\
            &\hat{\wasserstein}_{2,k+1} = 
            \lipschitz_k\left[\hat{\wasserstein}_{2,k} + \mw_2\left(q_{nn(\sX),k}, \compress_{\changed{M_k}}(q_{nn(\sX),k})\right) \right. \nonumber \\
            &\hspace{5cm}\left. +\hat\wasserstein_{2,\bm\signature_{\bm\sC_k}}\left(\compress_{\changed{M_k}}(q_{nn(\sX),k})\right) \right], \qquad \forall k\in\{1,\hdots,K\}, \\
            &\hat{\wasserstein}_{2}=\hat{\wasserstein}_{2,K+1},
        \end{align}
    \end{subequations}
    where $\hat\wasserstein_{2,\bm\signature_{\bm\sC}}$ is defined in Eqn~\eqref{eq:Was4SignMixMultGaus}, $\mw_2$ as defined in Eqn~\eqref{eq:definitionMW2}, and $\lipschitz_k$ is as defined Eqn~\eqref{eq:lipschitz_k}.}
    Then, it holds that 
    \begin{equation*}\label{eq:WasserNetwork}
        \wasserstein_2(p_{nn(\sX)}, q_{nn(\sX)}) \leq \hat{\wasserstein}_2(p_{nn(\sX)}, q_{nn(\sX)}) \coloneqq \hat{\wasserstein}_{2}.
    \end{equation*}
\end{corollary}
Note that the effect of multiple input points on the error introduced by \changed{$\hat\wasserstein_{2,\bm\signature_{\bm\sC_k}}\left(\compress_{\changed{M_k}}(q_{nn(\sX),k})\right)$ and $\mw_2\left(q_{nn(\sX),k}, \compress_{\changed{M_k}}(q_{nn(\sX),k})\right)$} depends on the specific choice of points in $\sX$. For points in $\sX$ where $q_{nn(\sX),k}$ is highly correlated, the error will be similar to that of a single-input case. However, if there is little correlation, the error bound will increase linearly with the number of points.

\subsection{Convergence Analysis}\label{sec:ConvergenceAnalysis}
In Theorem~\ref{thm:WasserNetwork} and \changed{Corollary}~\ref{corol:WasserNetworkSet}, we derived error bounds on the distance between $q_{nn(\sX)}$ and $p_{nn(\sX)}$, \changed{which depend on the signature locations and on the size of the compressed mixtures.
The following theorem shows that by appropriately selecting these quantities, the error bound can be made arbitrarily small.}
\changed{
\begin{theorem}\label{theorem:ConvergenceWasserBound}
    Let $p_{nn(\sX)}$ be the output distribution of a SNN with $K$ hidden layers for a finite set of inputs $\sX\subset\eucl^{n_0}$. 
    For any $\epsilon>0$, for layer $k\in\{1,\hdots,K\}$, choose a compression size $M_k\in\natNum$ such that
    \begin{equation}\label{eq:condition_M_k}
        \epsilon_k\coloneqq\frac{\epsilon}{K \prod_{i=k}^{K}\lipschitz_i}  - \mw_2\left(q_{nn(\sX),k}, \compress_{M_k}(q_{nn(\sX),k})\right) > 0,
    \end{equation}
    where $\lipschitz_i$ is defined in Eqn~\eqref{eq:lipschitz_k}. Choose per-component, per-dimension signature grid sizes $N_{k,i,j}\in\natNum$ satisfying
    \begin{equation}\label{eq:condition_N_k_i_j}
        \sum_{i=1}^{M_k}\elem{\tilde{\pi_k}}{i}\sum_{j=1}^{n_k}\elem{\vlambda_{k,i}}{j}\left(\frac{4\log N_{k,i,j}}{N_{k,i,j}^2} + r(N_{k,i,j})\right) \leq \epsilon_k^2,
    \end{equation}
    where $\tilde\pi_k$ are the mixture weights, $\vlambda_{k,i}$ is the vector of eigenvalues of the covariance matrix of the $i$-th component of $q_{nn(\sX),k}$, and $r(\cdot)$ is given in Eqn~\eqref{eq:uniformGridRemainder}.
    Construct for each component $i$ of $q_{nn(\sX),k}$ a set of signature locations $\sC_{k,i}\subset\eucl^{n_k}$ as in Eqn~\eqref{eq:AxisAlignedGridAssumption}, define
    $
        \bm\sC_k = \{\sC_{k,i}\}_{i=1}^{M_k}, 
    $
    and let $q_{nn(\sX),k}$ be as in Eqn~\eqref{eq:q_nn(vec(X)),k}. Then, for $\hat{\wasserstein}_2$ defined in Eqn~\eqref{eq:hat{wasserstein}_k_set_of_points}, it holds that
    \begin{equation}
        \wasserstein_2\left(p_{nn(\sX)},q_{nn(\sX),K+1}\right)\leq\hat\wasserstein_2\leq\epsilon.
    \end{equation}
\end{theorem}
To prove Theorem~\ref{theorem:ConvergenceWasserBound}, we use that Corollary~\ref{corol:WasserNetworkSet} implies the layer-wise decomposition
\begin{equation*}%\label{eq:hat_w2_decomp}
    \hat\wasserstein_2 = \sum_{k=1}^K \Bigg(\prod_{i=k}^K \lipschitz_i\Bigg) \Big[\mw_2\big(q_{nn(\sX),k},\compress_{M_k}(q_{nn(\sX),k})\big)+\hat\wasserstein_{2,\bm\signature_{\bm\sC_k}}\big(\compress_{M_k}(q_{nn(\sX),k})\big)\Big].
\end{equation*}
That is, $\hat\wasserstein_2$ equals the sum of compression and signature errors at layer $k$, scaled by the product of Lipschitz constants of the remaining layers $\prod_{i=k}^K \lipschitz_i$. If at each layer $k$ the compression size $M_k$ and grid sizes $N_{k,i,j}$ are such that
\[
\mw_2\big(q_{nn(\sX),k},\compress_{M_k}(q_{nn(\sX),k})\big)+\hat\wasserstein_{2,\bm\signature_{\bm\sC_k}}\big(\compress_{M_k}(q_{nn(\sX),k})\big)\le \frac{\epsilon}{K\prod_{i=k}^K \lipschitz_i},
\]
it follows that $\hat\wasserstein_2\le \epsilon$.

For any choice of $M_k$ such that $\epsilon_k>0$ (Eqn~\ref{eq:condition_M_k}), Corollary~\ref{corol:WasConvergence4SignMixtureGaus} allows us to select grid sizes $N_{k,i,j}$ so that $\hat\wasserstein_{2,\bm\signature_{\bm\sC_k}}\big(\compress_{M_k}(q_{nn(\sX),k})\big)\leq \epsilon_k$, thereby satisfying the per-layer requirement and yielding $\hat\wasserstein_2\le \epsilon$. 
Applying no compression (i.e., setting $M_k$ equal to the original number of mixture components in $q_{nn(\sX),k}$) ensures $\epsilon_k>0$ but can cause exponential growth in the total number of signature locations across depth, since signatures are taken component-wise. In practice, we therefore compress only redundant or low-weight components, i.e., retaining distinct modes, which preserves most of the signature budget while substantially reducing the number of grids (see Section~\ref{sec:expSNN2MGP}).

\begin{remark}\label{remark:GridSizeVsConvergenceRate}
    The number of signature locations required at layer $k$ in Theorem~\ref{theorem:ConvergenceWasserBound} is essentially inversely proportional to the layer accuracy parameter $\epsilon_k$, and hence to the target accuracy $\epsilon$ (up to the distribution of the budget across layers).\footnote{At layer $k$, the signature contribution behaves like $O\left(\tfrac{\log N_{k,i,j}}{N_{k,i,j}^2}\right)$ for large $N_{k,i,j}$. Solving $\log N_{k,i,j}/N_{k,i,j}^2 \approx \epsilon_k^2$ gives $N_{k,i,j}=O\big(\epsilon_k^{-1}\sqrt{\log(1/\epsilon_k)}\big)$: a $1/\epsilon_k$ rate with only a mild $\sqrt{\log}$ factor.} 
    In the worst-case scenario where each of the $n_k$ dimensions of a component requires the same grid size, the total number of signature locations for that component scales as 
    $$\prod_{j=1}^{n_k} N_{k,i,j}=O\big(\epsilon_k^{-n_k}(\log(1/\epsilon_k))^{n_k/2}\big).$$ 
    In practice, this curse of dimensionality is strongly mitigated because Condition~\eqref{eq:condition_N_k_i_j} weights each dimension by its eigenvalue. Empirically, most variance concentrates in a few neurons per layer, so the effective dimensionality is lower and many low-variance dimensions require only a single point ($N_{k,i,j}=1$). Additionally, non-uniform and adaptive (e.g., variance-proportional) grids further reduce required sizes while preserving $\wasserstein_2$ accuracy (see Remark~\ref{remark:optimalSignature}). In Section~\ref{sec:expSNN2MGP}, we empirically analyze how this nominal exponential dependence is mitigated in practice. 
\end{remark}
}

\changedNEW{
\begin{example}
    Consider the two-hidden layer ($K=2$) SNN in Figure~\ref{fig:method}, where the Lipschitz constants (as defined in Eqn~\ref{eq:lipschitz_k}) of the last two layers are $\lipschitz_1=0.5$ and $\lipschitz_2=0.41$.
    Using Theorem~\ref{theorem:ConvergenceWasserBound}, we derive the compression and grid sizes required to ensure that the approximation $q_{nn(\vx)}$ defined in Eqn~\eqref{eq:q_nn(x),k} is within $0.2$ in 2-Wasserstein distance ($\epsilon = 0.2$) of the true SNN output distribution $p_{nn(\vx)}$ at the input point considered in the figure.
    
    The distribution after the first linear layer, $q_{nn(\vx),1}$, is a univariate Gaussian with variance $\lambda_1 = 0.8$ (step a in the figure), so no compression is required and we set $M_1 = 1$.  
    The maximum allowable signature approximation error at this layer is
    $$
        \epsilon_1=\tfrac{\epsilon}{K\prod_{i=1}^2\lipschitz_i} = \tfrac{0.2}{2\cdot0.5\cdot0.41}=0.49.
    $$
    Constructing the signature locations $\sC_1$ uniformly according to Eqn~\eqref{eq:uniform1DGrid} yields
    $$
        \wasserstein(\signature_{\sC}\#q_{nn(\vx),1},q_{nn(\vx),1})\leq 0.49,
    $$
    that is, Eqn~\eqref{eq:condition_N_k_i_j} holds for $N_{1}=5$ and $\epsilon_1=0.49$.
    Using instead the optimal grid computed via Algorithm~\ref{al:signaturesOfGaussianMixutres}, we obtain an error of $0.23$ with the same number of locations.
    Using these optimal locations, we propagate the signature of $q_{nn(\vx),1}$ through the activation and the second linear layer and obtain $q_{nn(\vx),2}$ as in Eqn~\eqref{eq:q_nn(x),k:actiLinOperation} (steps b-d).  
    Next, we compress the resulting 5-component mixture $q_{nn(\vx),2}$ to a mixture with $M_2 = 3$ components, which yields a compression error
    $$
        \mw_2\left(q_{nn(\vx),2}, \compress_{M_2}(q_{nn(\vx),2})\right)=0.02.
    $$
    This leaves a remaining signature approximation budget of 
    $$
        \epsilon_2 = \tfrac{0.2}{2\cdot0.41}-0.02=0.22.
    $$
    The compressed mixture $\compress_{M_2}(q_{nn(\vx),2})$ has three components with weights $\vpi = (0.2, 0.6, 0.2)$ and variances $(0.5, 0.2, 0.5)$. Eqn~\eqref{eq:condition_N_k_i_j} is therefore satisfied with grid sizes $N_{2,1}=3$, $N_{2,2}=4$, and $N_{2,3}=3$.  
    Taking the signature of $\compress_{M_2}(q_{nn(\vx),2})$ and passing it through the activation and final linear layer, we obtain an approximation $q_{nn(\vx)}$ with $M = 10$ components achieving $\epsilon = 0.2$.  
    Using Algorithm~\ref{al:signaturesOfGaussianMixutres} again to construct the signature locations, we instead guarantee $\epsilon_2\leq 0.22$ with only $M = 5$ and obtain a final approximation of $5$ components (steps f-h), achieving $\epsilon=0.12$.
\end{example}
}

\begin{algorithm}[htbp]
\DontPrintSemicolon
\SetKwInOut{Input}{input}\SetKwInOut{Output}{output}
\Input{$p_{nn(\sX)}$}
\Output{$q_{nn(\sX)}$ and $\hat{\wasserstein}_{2}\left(p_{nn(\sX)}, q_{nn(\sX)}\right)$}
\Begin{
    \nl
    Initialize $q_{nn(\sX),1}$ as in Eqn \eqref{eq:q_nn(vec(X)),k,init} and $\hat{\wasserstein}_{2,1}$ as in Eqn \eqref{eq:hat{wasserstein}_k,init}\;
    \changed{
    \For{$k\in \{1,\hdots,K\}$}{
        \nl \changed{Construct $\compress_{M_k}(q_{nn(\sX),k})$ using compression Algorithm \ref{al:compressGMMs}}\;
        \nl \changed{Compute $\mw_2\left(q_{nn(\sX),k}, \compress_{M_k}(q_{nn(\sX),k})\right)$ as in Eqn~\eqref{eq:definitionMW2}} \;
        \nl Construct signature locations $\bm\sC_k$ and the signature $d_{nn(\sX),k}$ of $\compress_{M_k}(q_{nn(\sX),k})$, and compute the signature error $\hat{\wasserstein}_{2,\signature_{\bm\sC_k}}(\compress_{M_k}(q_{nn(\sX),k}))$ all via Algorithm~\ref{al:signaturesOfGaussianMixutres}\;
        \nl
        Update $q_{nn(\sX),k+1}$ as in Eqn~\eqref{eq:q_nn(vec(X)),k,act+lin} \; 
        \nl
        \changed{Compute \(\expect_{\vw_k\sim p_{\vw_k}}[\lipschitz_{L^{\vw_k}}^2]^{\frac{1}{2}}\) as described in Appendix~\ref{append:expected_lipscthiz}}\;
        \nl
        Update $\hat{\wasserstein}_{2,k+1}$ as in Eqn~\eqref{eq:hat{wasserstein}_k,update}
    }
    \nl
    Return $q_{nn(\sX)}=q_{nn(\sX),K+1}$ and $\hat{\wasserstein}_{2}\left(p_{nn(\sX)}, q_{nn(\sX)}\right)=\hat{\wasserstein}_{2,K+1}$
    }
    }
\caption{Gaussian Mixture Approximation of a SNN\label{al:SNN2GMM}}
\end{algorithm}

\section{Algorithmic Framework}\label{sec:AlgorithmicSection}
In this section, the theoretical methodology to solve Problem \ref{Prob:mainProb} developed in the previous sections is translated into an algorithmic framework. First, following Eqn~\eqref{eq:q_nn(vec(X)),k}, we present an algorithm to construct the Gaussian mixture approximation of a SNN with error bounds in the 2-Wasserstein distance, including details on the compression step. 
Then, we rely on the fact that the error bound resulting from our approach is piecewise differentiable with respect to the parameters of the SNN to address the complementary problem highlighted in Remark \ref{remark:InverseProb}. This problem involves deriving an algorithm to optimize the parameters of a SNN such that the SNN approximates a 
 given GMM.

\begin{algorithm}[ht]
\DontPrintSemicolon
\SetKwInOut{Input}{input}\SetKwInOut{Output}{output}
\Input{$\changed{q=}\sum_{i=1}^{M}\elem{\Tilde{\vpi}}{i}\Ndist(\mNdist_{i},\vNdist_{i})$}
\Output{$d=\sum_{j=1}^N\evpi{j}\delta_{\vc_j}$ and \changed{$\hat{\wasserstein}_{2,\bm\signature_{\bm\sC_k}}(q)$}}
\Begin{
    Construct the signature of the components of the mixture using the (fixed and optimal) signature of $\Ndist(0,1)$ with locations $\sC_{1d}$: \;
    \For{$i\in \{1,\hdots, M\}$}{
        \nl Compute eigenvalues vector $\vlambda_i$ and eigenvector matrix $\mV_i$ of $\vNdist_i$ \;
        \nl Find the optimal grid sizes $\{N_1^*, \hdots,  N_n^*\}$ for $\vlambda_i$ according to Eqn~\eqref{eq:OptimalGridUnivariateGauss} \;
        \nl Construct grid of points $\sC^* = \sC^*(N_1) \times\hdots\times\sC^*(N_n) $ , with $\sC^*(N_i)$ as in Eqn~\eqref{eq:OptimalGridSizes} \;
        \nl Transform the grid to the original space: 
        $\sC_i = \postImage(\sC^*,\mS) + \mNdist_i$ where $\mS = \mV_i\diag(\vlambda_i)^{\frac{1}{2}}$\; 
        \changed{
        \nl Set $\hat\wasserstein_{2,\signature_{\sC_i}}^2 = \sum_{l=1}^n\elem{\vlambda_i}{l}\sum_{j=1}^{N_l}\elem{\vpi_l}{j}\big[\elem{\vnu_j}{l}+(\elem{\vmu_j}{l}-\elem{\Tilde\vc_j}{l})^2\big]$, with $\sC^*(N_l)=\{\elem{\vc_j}{l}\}_{j=1}^{N_l}$, and $\elem{\vpi_l}{j}$, $\elem{\vmu_j}{l}$, and $\elem{\vnu_j}{l}$ as in Eqn~\eqref{eq:Was4SignStdGausTerms}\;
        }
        \changed{
        \nl Define $d_i = \sum_{j=1}^{|\sC_i|}\elem{\vpi}{j} \vc_{j} $ for
        $\elem{\vpi}{j}=\prod_{l=1}^n\elem{\vpi_l}{j}$\;
        }
    }
    \nl \changed{Return} $d=\sum_{i=1}^M\elem{\Tilde{\vpi}}{i}d_i$ and \changed{$\hat\wasserstein_{2,\bm\signature_{\bm\sC_k}}^2(q)=\sum_{i=1}^{M} \elem{\Tilde{\vpi}}{i} 
    \hat\wasserstein_{2,\signature_{\sC_i}}^2$} \;
}
\caption{Signatures of Gaussian Mixtures}\label{al:signaturesOfGaussianMixutres}
\end{algorithm}
\begin{algorithm}[htbp]
\DontPrintSemicolon
\SetKwInOut{Input}{input}\SetKwInOut{Output}{output}
\Input{$\Tilde{q}=\sum_{i=1}^{\Tilde{M}}\elem{\Tilde{\vpi}}{i}\Ndist(\Tilde{\mNdist}_{i},\Tilde{\vNdist}_{i})$}
\Output{$q=\sum_{j=1}^{M}\elem{\vpi}{j}\Ndist(\mNdist_{j},\vNdist_{j})$}
\Begin{
    \nl Collect the components of $\Tilde{q}$ in $M$ clusters by applying $M$-means clustering to $\{\Tilde{\mNdist}_1,\hdots,\Tilde{\mNdist}_{\Tilde{M}}\}$ using Lloyd's Algorithm\;
    \For{$j\in\{1,\hdots,M\}$}{
        Compress all components in the $j$-th cluster with indices in $\sI_j$ 
        to component $\Ndist(\mNdist_{j},\vNdist_{j})$ with weight $\evpi{j}$
        by taking the weighted average of the means and covariances:\;
        \nl $\evpi{j} = \sum_{i\in\sI_{j}}\elem{\Tilde{\vpi}}{i}$;  $\mNdist_{j} = \frac{1} {\evpi{j}}\sum_{i\in\sI_{j}}\elem{\Tilde{\vpi}}{i} \Tilde{\mNdist}_i$;  $\vNdist_{j} = \frac{1}{\evpi{j}}\sum_{i\in\sI_{j}}\elem{\Tilde{\vpi}}{i}\left[\Tilde{\vNdist}_i + (\Tilde{\mNdist}_i - \mNdist_j)(\Tilde{\mNdist}_i - \mNdist_j)^T\right]$ \;
    }
    \nl \changed{Return $q=\sum_{j=1}^{M}\elem{\vpi}{j}\Ndist(\mNdist_{j},\vNdist_{j})$}\;
}
\caption{Compress Gaussian Mixture
}\label{al:compressGMMs}
\end{algorithm}

\subsection{Construction of the GMM Approximation of a SNN}\label{sec:algo:gmmApproximateSNN}
The procedure to build a Gaussian mixture approximation $q_{nn(\sX)}$ of $p_{nn(\sX)}$, i.e., the output distribution of a SNN at a set of input points $\sX$ as in Eqn~\eqref{eq:q_nn(vec(X)),k}, is summarized in Algorithm~\ref{al:SNN2GMM}.
The algorithm consists of a forward pass over the layers of the SNN. First, the Gaussian output distribution after the first \changed{affine} layer is constructed (line~1). Then, for each layer $k>0$, $q_{nn(\sX),k}$ is compressed to a mixture of size $M$ (line~2); the signature of $\compress_{\changed{M_k}}(q_{nn(\sX)})$ is computed (line 4); and the signature is passed through the activation and \changed{affine} layer to obtain $q_{nn(\sX),k+1}$ (line~5). Parallel to this, error bounds on the 2-Wasserstein distance are computed for the compression operation 
(line~3), \changed{signature operation} (line~4), and \changed{affine} operation (line~6), which are then composed into a bound on $\wasserstein_2\left(p_{nn(\sX),k+1},q_{nn(\sX),k+1}\right)$ that is propagated to the next layer (\changed{line~7}). 

\paragraph{Signature Operation}
The signature operation on $q_{nn(\sX),k}$ in line~4 of Algorithm~\ref{al:SNN2GMM} is performed according to the steps in~Algorithm \ref{al:signaturesOfGaussianMixutres}. That is, the signature of $q_{nn(\sX),k}$ is taken as the mixture of the signatures of the components of $q_{nn(\sX),k}$ (lines~\changed{1-6}). 
For each Gaussian component of the mixture, to ensure an analytically exact solution for the Wasserstein distance resulting from the signature, we position its signature locations on a grid that aligns with the covariance matrix's basis vectors, i.e., the locations satisfy the constraint as in Eqn~\eqref{eq:AxisAlignedGridAssumption} of Corollary~\ref{corol:Was4SignMultGaus}. 
In particular, the grid of signatures is constructed so that the Wasserstein distance from the signature operation is minimized. According to the following proposition, the problem of finding this optimal grid can be solved in the transformed space, in which the Gaussian is independent over its dimensions. 
\begin{proposition}\label{prop:OptimalGrid}
    For $\Ndist(\mNdist, \vNdist)\in\probMeas(\eucl^n)$, let matrix $\mT \in\eucl^{n\times n}$ be such that $\mT=\diag(\vlambda)^{-\frac{1}{2}}\mV^T$ where $\diag(\vlambda)=\mV^T\vNdist\mV$ is a diagonal matrix whose entries are the eigenvalues of $\vNdist$, and $\mV$ is the corresponding orthogonal (eigenvector) matrix. Then,
    \begin{align*}
        \argmin{\substack{\{\sC^1,\hdots,\sC^n\}\subset\eucl, \\ \sum_{i=1}^n|\sC^\changed{i}|=N}}
        \wasserstein_2\left(\signature_{\postImage(\sC^1\times\hdots\times\sC^n, \mT^{-1})+\{\mNdist\}}\#\Ndist(\mNdist,\vNdist),\Ndist(\mNdist,\vNdist)\right) 
        = \left\{\sC^*(N_1^*),\hdots,\sC^*(N_n^*)\right\},
    \end{align*}
    where
    \begin{align}
        \sC^*(N)&=\argmin{\sC\in\eucl,|\sC|=N}\wasserstein_2\left(\signature_{\sC}\#\Ndist(0,1),\Ndist(0,1)\right), \label{eq:OptimalGridUnivariateGauss} \\
        \{N_1^*,\hdots,N_n^*\} &= \argmin{\substack{\{N_1,\hdots,N_n\}\in\natNum^n, \\ \prod_{j=1}^n, N_j=N}}\sum_{j=1}^n\evlambda{j}\wasserstein_2\left(\signature_{\sC^*(N_i)}\#\Ndist(0,1),\Ndist(0,1)\right). \label{eq:OptimalGridSizes}
    \end{align}
\end{proposition}
Following Proposition~\ref{prop:OptimalGrid}, the optimal grid of signature locations of any Gaussian is defined by the optimal signature locations for a univariate Gaussian (Eqn~\ref{eq:OptimalGridUnivariateGauss}) and the optimal size of each dimension of the grid (Eqn~\ref{eq:OptimalGridSizes}).
While the optimization problem in Eqn~\eqref{eq:OptimalGridUnivariateGauss} is non-convex, it can be solved up to any precision using the fixed-point approach proposed in \cite{kieffer1982exponential}. In particular, we solve the problem once for a range of grid sizes $N$ and store the optimal locations in a lookup table that is called at runtime. To address the optimization problem in Eqn~\eqref{eq:OptimalGridSizes}, we use that $\wasserstein_2\left(\signature_{\sC^*(N)}\#\Ndist(0,1),\Ndist(0,1)\right)$, with optimal locations $\sC^*(N)$ as in Eqn~\eqref{eq:OptimalGridUnivariateGauss}, is strictly decreasing for increasing $N$ so that even for large $N$, the number of feasible non-dominated candidates $\{N_1,\hdots,N_n\}$ is small. 

\changed{
\begin{remark}\label{remark:GridVsSigmaPointSignatures}
    Algorithm~\ref{al:signaturesOfGaussianMixutres} constructs a set of grid-constrained signatures that satisfy Eqn~\eqref{eq:AxisAlignedGridAssumption}, thereby enabling formal error evaluation via Corollary~\ref{corol:Was4SignMultGaus}.     
    In settings where formal error bounds are not required, more flexible choices for the signature locations can alternatively be employed, provided they satisfy the conditions in Remark~\ref{remark:tractable_signature} that ensure tractability of the signature operations.
    One such alternative signature placement strategy is presented in Appendix~\ref{append:altApproxScheme}.
    In Section~\ref{sec:expSNN2MGP}, we empirically compare the strategy in Algorithm~\ref{al:signaturesOfGaussianMixutres} with the strategy in Appendix ~\ref{append:altApproxScheme} by estimating the Wasserstein distance using Monte Carlo methods.
\end{remark}
}

\paragraph{Compression Operation}
Our approach to compress a GMM of size $N$ into a GMM of size $M<N$, i.e., operation $\compress_{\changed{M_k}}$ in Eqn~\eqref{eq:q_nn(x),k:Compress&Signature}, is described in Algorithm \ref{al:compressGMMs} and is based on the M-means clustering using Lloyd's algorithm \citep{lloyd1982least} on the means of the mixture's components (line 1). That is, each cluster is substituted by a Gaussian distribution with mean and covariance equal to those of the mixture in the cluster (line 2).
\begin{remark}
    For the compression of discrete distributions, such as in the case of Bayesian inference via dropout, we present a computationally more efficient alternative procedure to Algorithm \ref{al:compressGMMs} in Section \ref{append:compression4dropout} of the Appendix. Note that in the dropout case, the support of the multivariate Bernoulli 
    distribution representing the dropout masks at each layer's input grows exponentially with the layer width, making compression of paramount importance in practice.
\end{remark}

% \newpage
\subsection{Construction of a SNN Approximation of a GMM}\label{sec:differentiability}
The algorithmic framework presented in the previous subsection finds the Gaussian Mixture distribution that best approximates the output distribution of a SNN and returns bounds on their distance. 
In this subsection, we show how our results can be employed to solve the complementary problem of finding the \changed{parameter} distribution of a SNN to best match a given GMM. 

Let us consider a finite set of input points $\sX\subset\eucl^{n_0}$ and a SNN and GMM, whose distributions at $\sX$ are respectively $p_{nn(\mathcal{X})}$ and $q_{gmm(\sX)}$. Furthermore, in this subsection, we assume that the \changed{parameter} distribution $p_\vw$ of the SNN is parametrized by a set of parameters $\psi$.  
For instance, in the case where $p_\vw$ is a multivariate Gaussian distribution, then $\psi$ parametrizes its mean and covariance. 
Note that, consequently, also $p_{nn(\sX)}$ depends on $\psi$ through Eqn~\eqref{eq:p_nn(sX),k}.  
To make the dependence explicit, in what follows, we will use a superscript to emphasize the dependency on $\psi$.
Then, our goal is to find $\psi^*=\mathrm{argmin}_\psi \wasserstein_2(p_{nn(\mathcal{X})}^{\psi},q_{gmm(\mathcal{X})})$.
To do so, we rely on Corollary \ref{corol:Wasser4SNNtoGMM} below, which uses the triangle inequality to extend \changed{Corollary~\ref{corol:WasserNetworkSet}}.

\begin{corollary}[of \changed{Corollary~\ref{corol:WasserNetworkSet}}]\label{corol:Wasser4SNNtoGMM}
    Let $\sX\subset\eucl^{n_0}$ be a finite set of input points. Then, for a SNN and a GMM, whose distributions at $\sX$ are respectively $p_{nn(\mathcal{X})}^{\psi}$ and $q_{gmm(\sX)},$ it holds that 
    \begin{equation*}%\label{eq:wasser4SNNtoGMM}
        \wasserstein_2\big(p_{nn(\sX)}^{\psi}, q_{gmm(\sX)}\big) \leq \hat{\wasserstein}_2\big(p_{nn(\sX)}^{\psi}, q_{nn(\sX)}^{\psi}\big) + \mw_2\big(q_{nn(\sX)}^{\psi}, q_{gmm(\sX)}\big),
    \end{equation*}
    where the GMM $q_{nn(\sX)}^{\psi}$ and the bound $\hat{\wasserstein}_2$ are obtained via Algorithm \ref{al:SNN2GMM}, and $\mw_2$ is as defined in Definition \ref{def:mw2}, with $q_{nn(\sX)}^{\psi}$ also depending on $\psi$ through Eqn~\eqref{eq:q_nn(vec(X)),k}.
\end{corollary}
Using, Corollary \ref{corol:Wasser4SNNtoGMM}, we can approximate $\psi^*$ as follows
\begin{equation}\label{eq:objective}
        \psi^* \approx \argmin{\psi} \big\{\beta \hat{\wasserstein}_2\big(p^\psi_{nn(\sX)}, q^\psi_{nn(\sX)}\big) + (1-\beta)\mw_2\big(q^\psi_{nn(\sX)}, q_{gmm(\sX)}\big)\big\},
\end{equation}
where $\beta\in[0,1]$ allows us to trade off between the gap between the bound $\hat{\wasserstein}_2\big(p^\psi_{nn(\sX)}, q^\psi_{nn(\sX)}\big)$ and $\wasserstein_2\big(p^\psi_{nn(\sX)}, q^\psi_{nn(\sX)}\big)$ and the gap between the bound $\mw_2\big(q^\psi_{nn(\sX)}, q_{gmm(\sX)}\big)$ and $\wasserstein_2\big(q^\psi_{nn(\sX)}, q_{gmm(\sX)}\big)$.
For instance, in the case where $q^\psi_{nn(\sX)}$ and $q_{gmm(\sX)}$ are Gaussian distributions, $\mw_2\big(q^\psi_{nn(\sX)}, q_{gmm(\sX)}\big)$ equals $\wasserstein_2\big(q^\psi_{nn(\sX)}, q_{gmm(\sX)}\big)$, leading us to choose $\beta<\frac{1}{2}$.
As discussed next, under the assumption that the mean and covariance of $p_\vw$ are differentiable with respect to $\psi$, the objective in Eqn~\eqref{eq:objective} is piecewise differentiable with respect to $\psi$. Hence, the optimization problem can be approximately solved using gradient-based optimization techniques, such as Adam \citep{kingma2014adam}.

Let us first consider the term $\hat{\wasserstein}_2$ in the objective, which is iteratively defined in Eqn~\eqref{eq:hat{wasserstein}_k} over the layers of the network using the quantities \changed{$\expect_{\vw_k\sim p_{\vw_k}}[\lipschitz_{L^{\vw_k}}^2]^{1/2}$, $\hat\wasserstein_{2,\bm\signature_{\bm\sC_k}}\left(\compress_{\changed{M_k}}(q_{nn(\sX),k})\right)$, and
$\mw_2\left(q_{nn(\vx),k}, \compress_{\changed{M_k}}(q_{nn(\vx),k})\right)$} for each layer $k$. 
In Algorithm \ref{al:SNN2GMM}, we compute these quantities so that the gradients w.r.t. the mean and covariance of $p_\vw$ are (approximately) analytically tractable:
\begin{enumerate}
    \item \changed{$\expect_{\vw_k\sim p_{\vw_k}}[\lipschitz_{L^{\vw_k}}^2]^{1/2}$ is computed as described in Appendix~\ref{append:expected_lipscthiz} (line 6 of Algorithm \ref{al:SNN2GMM}) as a function of} the sum of the variances of the $k$-th layer's \changed{parameters} and the spectral norm of the mean of the $k$-th layer's weight matrix. Hence, the gradient with respect to the variance is well defined, and the gradient with respect to the mean exists if the largest singular value is unique; otherwise, we can use a sub-gradient derived from the singular value decomposition (SVD) of the mean matrix \citep{rockafellar2015convex}.
    
    \item \changed{$\hat\wasserstein_{2,\bm\signature_{\bm\sC_k}}\left(\compress_{\changed{M_k}}(q_{nn(\sX),k})\right)$}
    is obtained according to Algorithm~\ref{al:signaturesOfGaussianMixutres} (line 4 of Algorithm \ref{al:SNN2GMM}) based on the eigendecomposition of the covariance matrix for each component of the GMM \changed{$\compress_{M_k}(q_{nn(\sX),k})$} (line 1 of Algorithm~\ref{al:signaturesOfGaussianMixutres}). 
    Differentiability of \changed{$\hat\wasserstein_{2,\bm\signature_{\bm\sC_k}}\left(\compress_{\changed{M_k}}(q_{nn(\sX),k})\right)$} with respect to the mean and covariance of $p_\vw$ follows from the piecewise differentiability of the eigenvalue decomposition of a (semi-)positive definite matrix \citep{magnus2019matrix},\footnote{Note that there are discontinuous changes in the eigenvalues for small perturbations if the covariance matrix is (close to) degenerate, which occurs in our cases if the behavior of the SNN is strongly correlated for points in $\sX$. To prevent this, we use a stable derivation implementation for degenerate covariance matrices as provided by \cite{kasim2020derivatives} to compute the eigenvalue decomposition.} provided that the covariances of \changed{$\compress_{M_k}(q_{nn(\sX),k})$} are differentiable with respect to the mean and covariance of $p_\vw$.
    To establish the latter, note that the compression operation $\compress_{\changed{M_k}}$, as performed according to Algorithm \ref{al:compressGMMs}, 
    given the clustering of the components of the mixture $q_{nn(\sX),k}$, reduces to computing the weighted sum of the means and covariances of each cluster.
    Hence, the covariances of \changed{$\compress_{M_k}(q_{nn(\sX),k})$} are piecewise differentiable with respect to the means and covariances of $q_{nn(\sX),k}$. 
    Here, $q_{nn(\sX),k}$ is an \changed{affine} combination of $p_{\vw_{k-1}}$ and \changed{$\sX$} if $k=1$, and the support of the signature approximation of \changed{$\compress_{M_k}(q_{nn(\sX),k-1})$} otherwise (see \changed{Example~\ref{example:closedFormMixtureFeedforward}} for the closed-form expression of $q_{nn(\sX),k}$ in the case $\sX=\{\vx\}$). 
    As such, from the chain rule, it follows that for all $k$ the means and covariances of \changed{$\compress_{M_k}(q_{nn(\sX),k})$}, and consequently, \changed{$\hat\wasserstein_{2,\bm\signature_{\bm\sC_k}}\left(\compress_{\changed{M_k}}(q_{nn(\sX),k})\right)$}, are piecewise differentiable with respect to the means and covariances of $p_{\vw_{k-1}},\hdots, p_{\vw_0}$. 
    
    \item \changed{$\mw_2\left(q_{nn(\vx),k}, \compress_{\changed{M_k}}(q_{nn(\vx),k})\right)$} is taken as the $\mw_2$ distance
    between $q_{nn(\sX),k}$ and \changed{$\compress_{\changed{M_k}}(q_{nn(\sX),k})$} (line 3 in Algorithm \ref{al:SNN2GMM}). Given the solution of the discrete optimal transport problem in Eqn~\eqref{eq:definitionMW2}, the $\mw_2$ distance between two Gaussian mixtures reduces to the sum of 2-Wasserstein distances between Gaussian components, which have closed-form expressions \citep{givens1984class} that are differentiable with respect to the means and covariances of the components. 
    In point 2, we showed that the means and covariances of $q_{nn(\sX),k}$ and \changed{$\compress_{\changed{M_k}}(q_{nn(\sX),k})$} are piecewise differentiable w.r.t the mean and covariance of $p_\vw$; therefore, \changed{$\mw_2\left(q_{nn(\vx),k}, \compress_{\changed{M_k}}(q_{nn(\vx),k})\right)$} is piecewise differentiable with respect to $p_\vw$.
\end{enumerate}
From the piecewise differentiability of \changed{all three quantities} with respect to the mean and covariance of $p_\vw$ for all layers $k$, it naturally follows that $\hat{\wasserstein}_2\big(p^\psi_{nn(\sX)}, q^\psi_{nn(\sX)}\big)$ is piecewise differentiable with respect to the mean and covariance of $p_\vw$ using simple chain rules. 
For the second term of the objective, $\mw_2\big(q^\psi_{nn(\sX)}, q_{gp(\sX)}\big)$, we can apply the same reasoning from point 3 to conclude that it is piecewise differentiable with respect to the means and covariances of the components of the GMM $q^\psi_{nn(\sX)}$, which, according to point 2, are piecewise differentiable with respect to the mean and covariance of $p_\vw$.
Therefore, we can conclude that the objective in Eqn~\eqref{eq:objective} is piecewise differentiable with respect to the mean and covariances of the \changed{parameter} distribution of the SNN.
Note that, in practice, the gradients can be computed using automatic differentiation, as demonstrated in Subsection \ref{sec:expPriorTuning} where we encode informative priors for SNNs via Gaussian processes by solving the optimization problem in Eqn~\eqref{eq:objective}.

\begin{remark}
    While straightforward automatic differentiation shows to handle the discontinuity in the gradients well in practice, as analyzed in Subsection \ref{sec:expPriorTuning}, more advanced techniques for non-smooth optimization can be employed to solve Eqn~\eqref{eq:objective} \citep{makela2002survey,burke2020gradient}.  
\end{remark}

% \newpage
\section{Experimental Results}\label{sec:experiments}
In this section we experimentally evaluate the performance of our framework in solving Problem~\ref{Prob:mainProb} and then demonstrate its practical usefulness in two applications: uncertainty quantification and prior tuning for SNNs. We will start with Section~\ref{sec:expSNN2MGP}, where we consider various trained SNNs and analyze the precision of the GMM approximation obtained with our framework.  
Section~\ref{sec:expAnalysis} focuses on analyzing the uncertainty of the predictive distribution of SNNs trained on classification data sets utilizing our GMM approximations. Finally, in Section~\ref{sec:expPriorTuning}, we show how our method can be applied to encode functional information in the priors of SNNs.\footnote{Our code is available at \url{https://github.com/sjladams/experiments-snns-as-mgps}.} 

\paragraph{Data Sets}
We study SNNs trained on various regression and classification data sets. For regression tasks, we consider the NoisySines data set, which contains samples from a 1D mixture of sines with additive noise as illustrated in Figure \ref{fig:NoisySine}, and the Kin8nm, Energy, Boston Housing, and Concrete UCI data sets, which are commonly used as regression tasks to benchmark SNNs \citep{hernandez2015probabilistic}.  
For classification tasks, we consider the  MNIST, Fashion-MNIST, and CIFAR-10 data sets. 

\paragraph{Networks \& Training}
We consider networks composed of fully-connected and convolutional layers, among which the VGG style architecture  \citep{simonyan2014very}.
We denote by $[n_l\times n_n]$ an architecture with $n_l$ fully-connected layers, each with $n_n$ neurons. 
The composition of different layer types is denoted using the summation sign. 
For each SNN, we report the inference techniques used to train its stochastic layers in brackets, while deterministic layers are noted without brackets. For instance, VGG-2 + [1x128] (VI) represents the VGG-2 network with deterministic \changed{parameters}, stacked with a fully-connected stochastic network with one hidden layer of 128 neurons learned using Variational Inference (VI). For regression tasks, the networks use tanh activation functions, whereas ReLU activations are employed for classification tasks. 
For VI inference, we use VOGN \citep{khan2018fast}, and for deterministic and dropout networks, we use Adam. 

\paragraph{Metrics for Error Analysis}
To quantify the precision of our approach for the different settings, we report the Wasserstein distance between the approximate distribution $q_{nn(\sX)}$ and the true SNN distribution $p_{nn(\sX)}$ relative to the 2-nd moment of the approximate distribution. That is, we report
\begin{equation}\label{eq:relative2wasserstein}
    \overline{W}_2\left(p_{nn(\sX)},q_{nn(\sX)}\right)\coloneqq \frac{\wasserstein_2\left(p_{nn(\sX)},q_{nn(\sX)}\right)}{\expect_{\vz\sim q_{nn(\sX)}}[\|\vz\|^2]^{1/2}},  
\end{equation}
which we refer to as the relative 2-Wasserstein distance. We use $\overline{W}_2$ because, as guaranteed by Lemma~\ref{lemma:1WasserBounds12Moment} \changed{in the appendix}, it provides an upper bound on the relative difference between the second moment of $p_{nn(\sX)}$ and $q_{nn(\sX)}$, i.e., 
\begin{equation*}
    \frac{\big|\expect_{\vz\sim p_{nn(\sX)}}[\|\vz\|^2]^{1/2}-\expect_{\Tilde{\vz}\sim q_{nn(\sX)}}[\|\Tilde{\vz}\|^2]^{1/2}\big|}{\expect_{\Tilde{\vz}\sim q_{nn(\sX)}}[\|\Tilde{\vz}\|^2]^{1/2}} \leq \overline{W}_2\left(p_{nn(\sX)},q_{nn(\sX)}\right).
\end{equation*}
Note that in our framework, $q_{nn(\sX)}$ is a GMM, hence $\expect_{\vz\sim q_{nn(\sX)}}[\|\vz\|^2]$ can be computed analytically.\footnote{
We have that $\expect_{\vz\sim q_{nn(\sX)}}[\|\vz\|^2] = \|\mNdist\|^2 + \trace(\Sigma)$, where $\mNdist$ and $\vNdist$ are the mean vector and covariance matrix of $q_{nn(\sX)}$, respectively. Closed-forms for $\mNdist$ and $\vNdist$ exist for $q_{nn(\sX)}\in\GMM_{<\infty}(\eucl^{n})$.
}
In our experiments, we report both a formal bound of $\overline{W}_2\left(p_{nn(\sX)},q_{nn(\sX)}\right)$ obtained using Theorem~\ref{thm:WasserNetwork} and an empirical approximation.
To compute the empirical approximation of $\overline{W}_2\left(p_{nn(\sX)},q_{nn(\sX)}\right)$, we use $1e3$ Monte Carlo (MC) samples from $p_{nn(\sX)}$ and $q_{nn(\sX)}$ and solve the resulting discrete optimal transport problem to approximate $\wasserstein_2\left(p_{nn(\vx)},q_{nn(\vx)}\right)$.

\begin{table}[htbp]
  \centering
  {\small
    \begin{tabular}{ll|rr|r}
        & & \multicolumn{2}{c}{Grid} & \multicolumn{1}{|l}{\changed{Cross}} \\ \midrule
          Data Set & Network Architecture   & \multicolumn{1}{l}{Emp.} & \multicolumn{1}{l}{Formal} & \multicolumn{1}{|l}{\changed{Emp.}}\\
    \midrule
    \multicolumn{1}{l}{NoisySines} & [1x64] (VI) & 0.00109 & 0.33105 & \changed{0.00028} \\
          & [1x128] (VI) & 0.00130 & 0.13112 & \changed{0.00041} \\
          & [2x64] (VI) & 0.00355 & 1.63519 & \changed{0.00305} \\
          & [2x128] (VI) & 0.01135 & 2.48488 & \changed{0.00956} \\
    \multicolumn{1}{l}{Kin8nm} & [1x128] (VI) & 0.00012 & 0.04243 & \changed{0.00012}\\
          & [2x128] (VI) & 0.00018 & 0.46303 & \changed{0.00005} \\ \midrule
    \multicolumn{1}{l}{MNIST} & [1x128] (VI) & 0.00058 & 0.05263 &\changed{0.00037} \\
          & [2x128] (VI) & 0.00127 & 0.44873 & \changed{0.00124}\\
    \multicolumn{1}{l}{Fashion MNIST} & [1x128] (VI) & 0.00017 & 0.06984 &\changed{0.00013} \\
          & [2x128] (VI) & 0.02787 & 1.45466 &\changed{0.02228} \\
    \multicolumn{1}{l}{CIFAR-10} & VGG-2 + [1x128] (VI) & 0.00617 & 0.57670 & \changed{0.00363} \\
          & VGG-2 + [2x128] (VI) & 0.04901 & 2.35686 & \changed{0.02890} \\
    \midrule
    \multicolumn{1}{l}{MNIST} & [1x64] (Drop) & 0.27501 & 0.94371 & \changed{0.18400} \\
          & [2x64] (Drop) & 0.41874 & 1.37594 & \changed{0.08603} \\
      \multicolumn{1}{l}{CIFAR-10} & VGG-2 (Drop) + [1x32] (VI) & 0.27492 & 179.9631 &\changed{0.18996} \\ \bottomrule
    \end{tabular}%
    }
    \caption{The empirical estimates and formal bounds on the relative 2-Wasserstein distance between various SNNs and their Gaussian Mixture approximates obtained via Algorithm~\ref{al:SNN2GMM} using signature approximations of size $10$ and a compression size (\changed{$M_k$}) of $5$ \changed{for all layers}.
    \changed{The Grid column reports the error for approximates constructed using Algorithm~\ref{al:signaturesOfGaussianMixutres}. The Cross column reports the empirical errors for approximates constructed using the alternative cross-shaped signatures via Algorithm~\ref{al:signaturesOfGaussianMixutresCross}, for which formal bounds are intractable for this method.}
    For the dropout networks, instead of Algorithm~\ref{al:compressGMMs}, we use the compression procedure as described in Appendix~\ref{append:compression4dropout}. 
    The reported values are the average over $100$ randomly selected points from the test sets.} 
    \label{tab:Emp_vs_Form}
\end{table} 

\begin{table}[htbp]
  \centering
    {\small
    \begin{tabular}{ll|rrr}
    \multicolumn{2}{c|}{Compression Size:} & 1     & 3     & \multicolumn{1}{l}{No Compression} \\
    \midrule
    NoisySines & [2x64] (VI) & 1.97279 & 1.63519 & 1.08763 \\
          & [2x128] (VI) & 2.48109 & 2.48488 & 2.48398 \\
    Kin8nm & [2x128] (VI) & 0.48740 & 0.46303 & 0.45387 \\
    MNIST & [2x128] (VI) & 0.45036 & 0.44873 & 0.44873 \\
    Fashion MNIST & [2x128] (VI) & 1.45489 & 1.45466 & 1.45466 \\
    CIFAR-10 & VGG-2 + [2x128] (VI) & 2.36616 & 2.35686 & 2.35686 \\ \bottomrule
    \multicolumn{2}{c|}{Compression Size:} & 10    & 100   & 10000 \\
    \midrule
    MNIST & [1x64] (Drop) & 0.94371 & 0.73682 & 0.3571 \\
          & [2x64] (Drop) & 1.37594 & 1.09053 & 0.89742 \\
    CIFAR-10 & VGG-2 (Drop) + [1x32] (VI) & 179.9631 & 180.0963 & 145.2984 \\ \bottomrule
    \end{tabular}%
    }
    \caption{The formal bounds on the relative 2-Wasserstein distance between various SNNs and their Gaussian Mixture approximations obtained via Algorithm \ref{al:SNN2GMM} for a signature size of $100$ and different compression sizes (\changed{$M_k$}). The values are the average of 100 random test points.
    }
    \label{tab:CompressionSize_vs_W2Bound}
\end{table}

\subsection{Gaussian Mixture Approximations of SNNs}\label{sec:expSNN2MGP}
In this subsection, we experimentally evaluate the effectiveness of Algorithm~\ref{al:SNN2GMM} in constructing a GMM that is $\epsilon$-close to the output distribution of a given SNN.
We first demonstrate that, perhaps surprisingly,  even a small number of discretization points for the signature operation and a large reduction in the compression operation (i.e., \changed{$M_k$} relatively small), often suffice for Algorithm~\ref{al:SNN2GMM} to generate GMMs that accurately approximate SNNs both empirically and formally. Then, we provide empirical evidence that by increasing the signature and compression sizes, the approximation error of the GMM can be made arbitrarily small and the formal error upper bound resulting from Theorem~\ref{thm:WasserNetwork} converges to $0$.

\paragraph{Baseline Performance}
We start our analysis with Table~\ref{tab:Emp_vs_Form}, where we assess the performance of Algorithm~\ref{al:SNN2GMM} in generating GMM approximations of SNN trained using both VI and dropout.
The results show that for the SNNs trained with VI, even with only 10 signature points \changed{per-component}, Algorithm~\ref{al:SNN2GMM} is able to generate GMMs, whose empirical distance from the true BNN is always smaller than $10^{-2}$. 
This can be partly explained as for VI generally only a few neurons are active \citep{louizos2017bayesian,frankle2018lottery}, so it is necessary to focus mostly on these on the signature approximation to obtain accurate GMM approximations at the first $K$-th layers \changed{(See Remark~\ref{remark:GridSizeVsConvergenceRate})}. 
Moreover, the propagation of the GMM approximation through the last \changed{fully-connected} layer is performed in closed form (i.e., by Eqn~\ref{eq:q_nn(x),k:actiLinOperation} for $\sX=\{\vx\}$) and this mitigates the influence of the approximation error of the previous layers on the final approximation error. 
In contrast, for the dropout networks, the empirical approximation error is, on average, one order of magnitude larger. This can be explained because the input-output distribution of dropout networks is inherently highly multimodal. Thus, as we will show in the next paragraph, to obtain tighter GMM approximations, it is necessary 
to use $M>5$ (the value used in Table~\ref{tab:Emp_vs_Form}) and allow for a larger size in the GMM approximation.

Note that the formal error bounds tend to become more conservative with increasing network depth. This is related to the transformations in the (stochastic) \changed{fully-connected} layers, for which the (expected) norm of the weight matrices is used to bound its impact on the Wasserstein guarantees (see the \changed{affine} operator bound in Table~\ref{tab:wasser4layers}). It is also interesting to note how the error bounds tend to be more accurate, and consequently closer to the empirical estimates, for the dropout networks. This is because for VI it is necessary to bound the expected matrix norm of each \changed{affine} operation, which has to be upper bounded \changed{(see Appendix~\ref{append:expected_lipscthiz})}, thereby introducing conservatism. Fortunately, as we will illustrate in the following paragraph and as we have mathematically proven in the previous sections, in all cases the error bounds can be made arbitrarily small 
by increasing signature size and $M$. 

\changed{Lastly, in settings where formal error bounds are not required, alternative signature-placement methods such as Algorithm~ \ref{al:signaturesOfGaussianMixutresCross} in the appendix, which places the signature locations in a cross along the principal axes of the components' covariance matrix, provide a practical alternative. 
In the cases considered, this cross-shaped placement consistently attains higher empirical accuracy with the same number of signature points and lower runtime, so fewer points often suffice, however, it does not yield formal guarantees. The compute cost of Algorithm \ref{al:SNN2GMM} is modest, with processing a single test point typically taking on the order of tenths of a second.}\footnote{\changed{All experiments were run on an Intel Core i7-10610U (1.80–2.30 GHz) with 16 GB RAM.}}

\begin{figure}[htbp]
    \centering
    % \begin{subfigure}[b]{\textwidth}
        \includegraphics[width=0.95\textwidth]{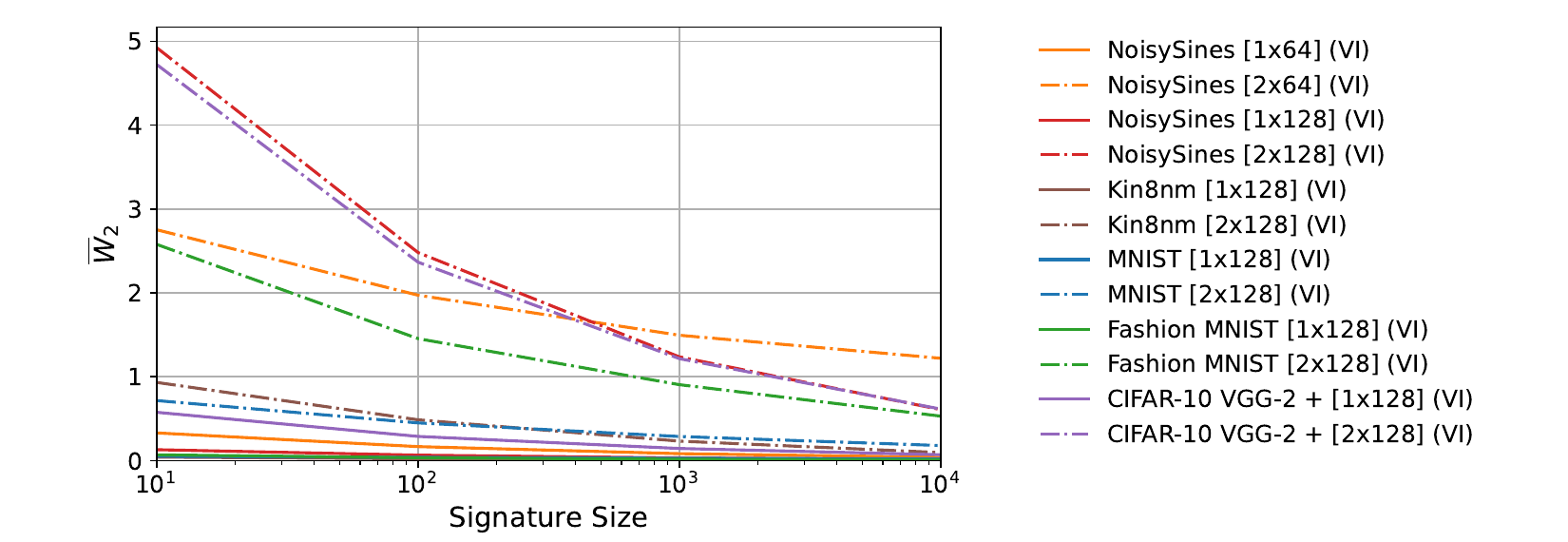}
    % \end{subfigure}
    \caption{Formal bounds on the relative 2-Wasserstein distance between various SNNs and their Gaussian Mixture approximations obtained via Algorithm \ref{al:SNN2GMM}, for a compression size ($M_k$) of $3$ and different signature sizes. The lines show the average of 100 random test points.}
    \label{tab:SignatureSize_vs_W2Bound}
\end{figure}

\paragraph{Uniform Convergence}
We continue our analysis with Figure \ref{tab:SignatureSize_vs_W2Bound}, where we conduct an extensive analysis of how the formal error bound from Theorem\ref{thm:WasserNetwork} changes with increasing signature size for various SNN architectures trained on both regression and classification data sets. 
The plot confirms that, with increasing signature sizes, the approximation error, measured as $\overline{W}_2$, tends to decrease uniformly. 
In particular, we observe an inverse‑linear convergence rate. \changed{This behaviour matches our theoretical result in Theorem~\ref{theorem:ConvergenceWasserBound}. That theorem implies that a per‑dimension grid size of order $N=O(\epsilon^{-1}\sqrt{\log(1/\epsilon)})$ is sufficient to reach a target precision $\epsilon$.}\footnote{In Appendix \ref{append:ExperimentsSigns4GMMs}, we experimentally investigate in more detail the uniform convergence of the 2‑Wasserstein distance resulting from the signature operation for Gaussian mixtures of different sizes.}

In Table \ref{tab:CompressionSize_vs_W2Bound}, we investigate the other approximation step: the compression operation, which reduces the components of the approximation mixtures (step e in Figure \ref{fig:method}). 
The results show that while the networks trained solely with VI can be well approximated with single Gaussian distributions, the precision of the approximation for the networks employing dropout strongly improves with increasing compression size. This confirms that the distribution of these networks is highly multi-modal and underscores the importance of allowing for multi-modal approximations, such as Gaussian mixtures.

\begin{figure}[p]
    \centering
    \begin{minipage}{\textwidth}
        \begin{minipage}{0.48\textwidth}
            \centering
            \begin{subfigure}[b]{0.6\textwidth}
                \includegraphics[width=\textwidth]{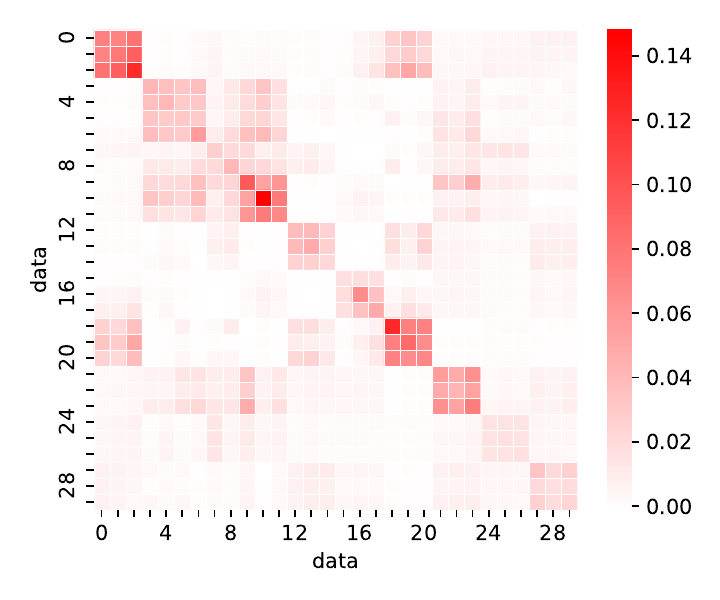}    
                \caption*{(a) Kernel SNN}
            \end{subfigure}
            \hfill
            \begin{subfigure}[b]{0.38\textwidth}
                \includegraphics[width=1.05\textwidth]{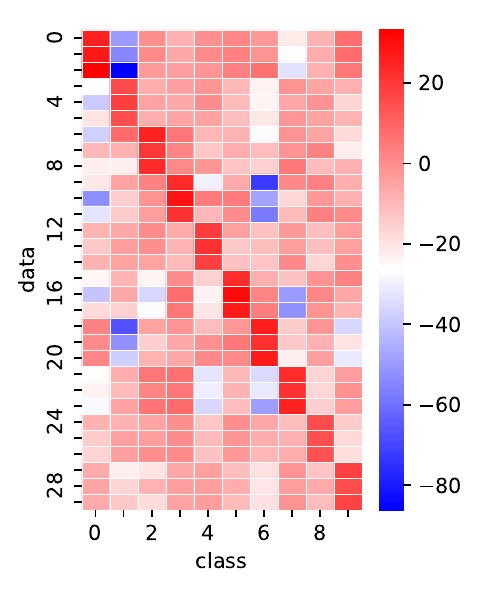}    
                \caption*{(b) Mean SNN}
            \end{subfigure}
        \end{minipage}
        \hfill
        \begin{minipage}{0.48\textwidth}
            \begin{subfigure}[b]{0.6\textwidth}
                \includegraphics[width=\textwidth]{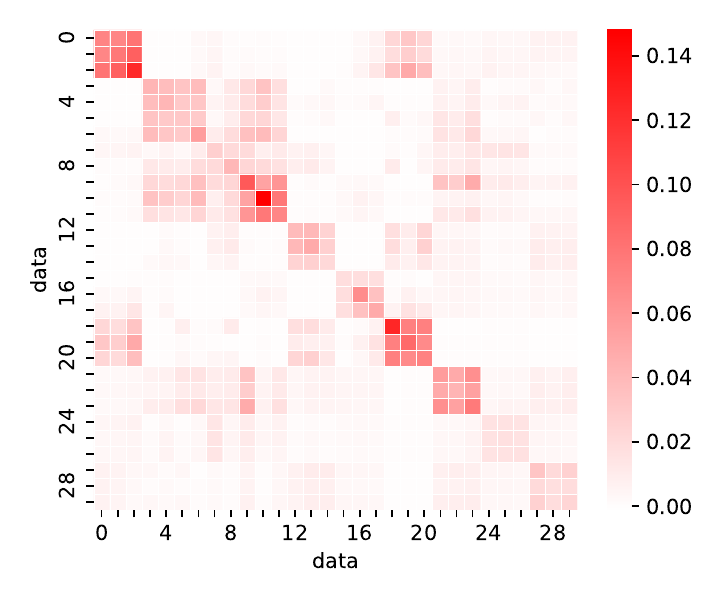}    
                \caption*{(c) Kernel GMM}
            \end{subfigure}
            \hfill
            \begin{subfigure}[b]{0.38\textwidth}
                \includegraphics[width=1.05\textwidth]{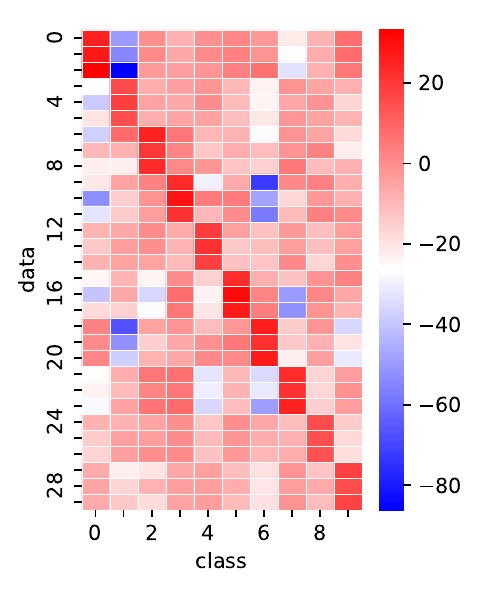}    
                \caption*{(d) Mean GMM}
            \end{subfigure}
        \end{minipage}
        \par\bigskip{\centering 
        (I) [2x128] (VI) trained on MNIST
        \par}
    \end{minipage}
    \vfill
    \begin{minipage}{\textwidth}
    \centering
    \begin{minipage}{0.48\textwidth}
        \centering
        \begin{subfigure}[b]{0.6\textwidth}
            \includegraphics[width=\textwidth]{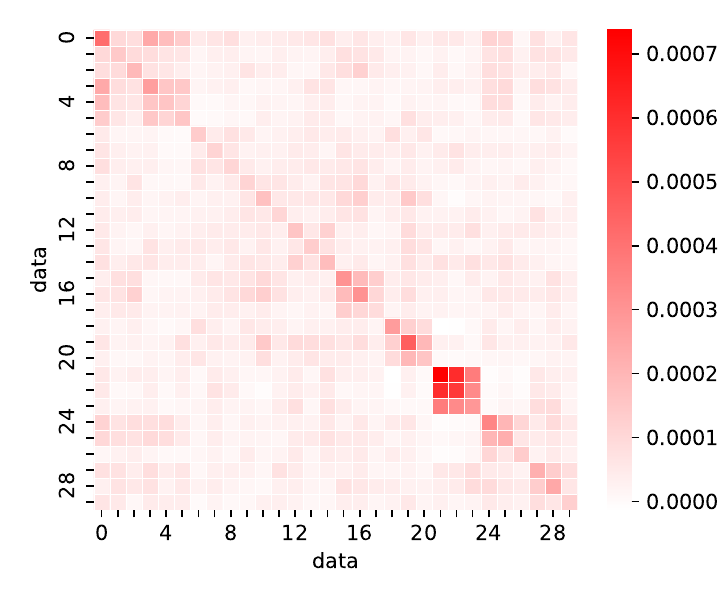}    
            \caption*{(a) Kernel SNN}
        \end{subfigure}
        \hfill
        \begin{subfigure}[b]{0.38\textwidth}
            \includegraphics[width=1.05\textwidth]{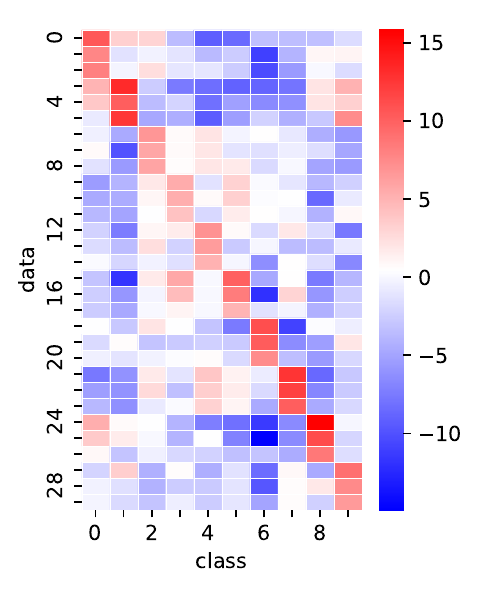}    
            \caption*{(b) Mean SNN}
        \end{subfigure}
    \end{minipage}
    \hfill
    \begin{minipage}{0.48\textwidth}
        \begin{subfigure}[b]{0.6\textwidth}
            \includegraphics[width=\textwidth]{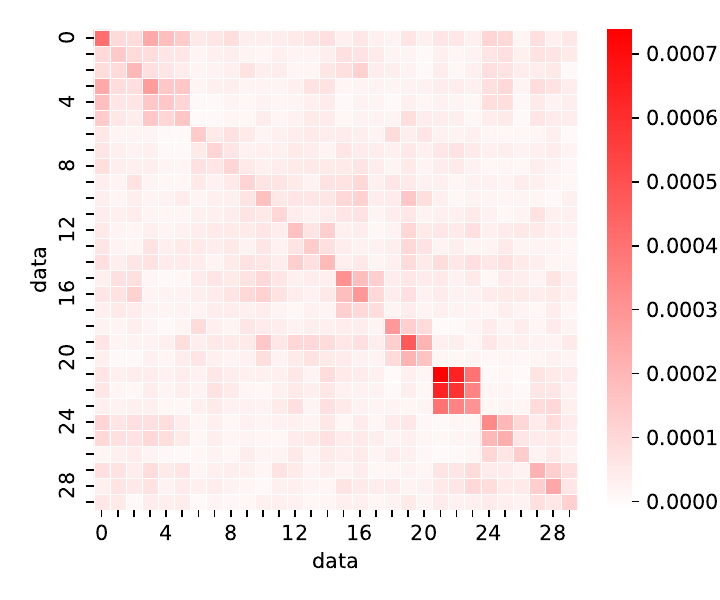}    
            \caption*{(c) Kernel GMM}
        \end{subfigure}
        \hfill
        \begin{subfigure}[b]{0.38\textwidth}
            \includegraphics[width=1.05\textwidth]{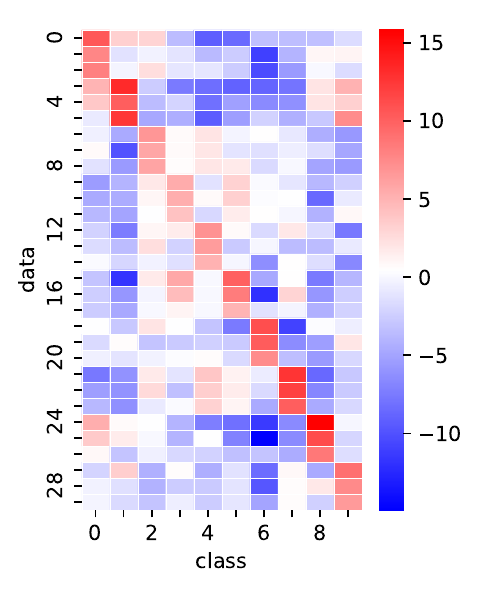}    
            \caption*{(d) Mean GMM}
        \end{subfigure}
    \end{minipage}
    \par\bigskip{\centering 
    (II) VGG-2 + [1x128] (VI) trained on CIFAR-10
    \par}
    \end{minipage}
    \vfill
    \begin{minipage}{\textwidth}
    \centering
    \begin{minipage}{0.48\textwidth}
        \centering
        \begin{subfigure}[b]{0.6\textwidth}
            \includegraphics[width=\textwidth]{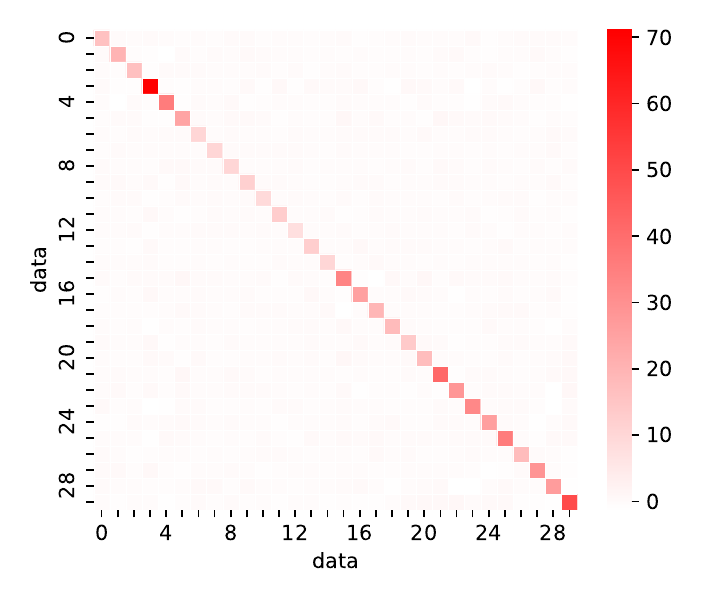}    
            \caption*{(a) Kernel SNN}
        \end{subfigure}
        \hfill
        \begin{subfigure}[b]{0.38\textwidth}
            \includegraphics[width=1.05\textwidth]{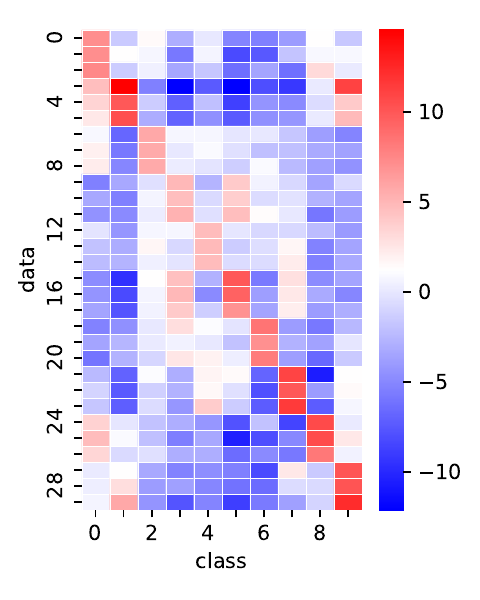}    
            \caption*{(b) Mean SNN}
        \end{subfigure}
    \end{minipage}
    \hfill
    \begin{minipage}{0.48\textwidth}
        \begin{subfigure}[b]{0.6\textwidth}
            \includegraphics[width=\textwidth]{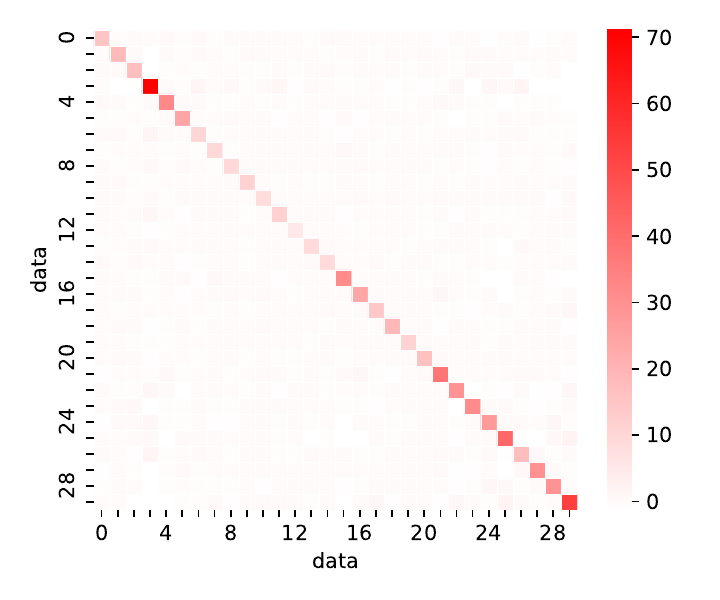}                \caption*{(c) Kernel GMM}
        \end{subfigure}
        \hfill
        \begin{subfigure}[b]{0.38\textwidth}
            \includegraphics[width=1.05\textwidth]{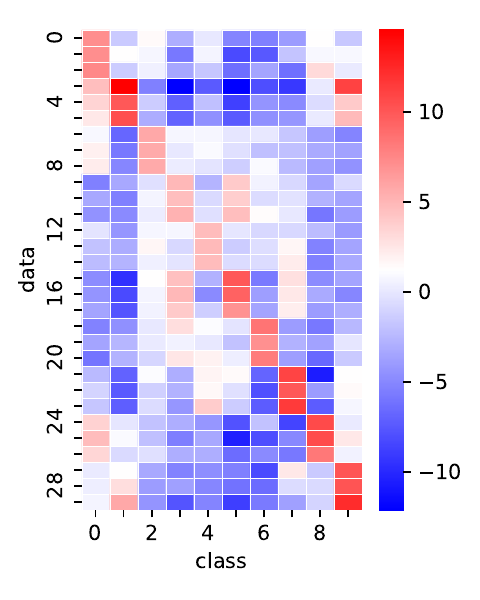}    
            \caption*{(d) Mean GMM}
        \end{subfigure}
    \end{minipage}
    \par\bigskip{\centering 
    (III) VGG-2 (Drop) + [1x128] (Drop) trained on CIFAR-10
    \par}
    \end{minipage}
    \caption{
   Kernels and means for various SNNs obtained via Monte Carlo sampling (left two columns) and of our GMM approximation (right two columns) for $30$ input points on MNIST in (I) and CIFAR-10 in (II) and (III). 
    Instead of the full kernel matrices, which have dimensionality $300\times 300$, we store the trace of the sub-blocks in the output dimension, and show the $30\times 30$ matrix that captures the covariance between the input points. 
    The rows and columns are grouped according to the classes, where the colored regions on the y-axis in the mean plots mark the classes. 
    (I) shows the approximate kernels and the predictive means for a fully-connected SNN trained via VI, which gives 98\% test accuracy. 
    We see in the kernel that examples with the same class labels are correlated. In (II), the VGG network with dropout is combined with a linear stochastic layer on the more complex CIFAR-10 data set, achieving 65\% accuracy. (III) shows VGG stacked with a linear layer, to which dropout is applied, achieving 64\% accuracy on CIFAR-10.
    }
    \label{fig:SNN2GMM_class}
\end{figure}

\subsection{Posterior Performance Analysis of SNNs}\label{sec:expAnalysis}
The GMM approximations obtained by Algorithm \ref{al:SNN2GMM} naturally allow us to visualize the mean and covariance matrix (i.e.,~kernel) of the predictive distribution at a set of points. The resulting kernel offers the ability to reason about our confidence on the predictions and enhance our understanding of SNN performances \citep{khan2019approximate}. The state-of-the-art for estimating such a kernel is to rely on Monte Carlo sampling \citep{van2020uncertainty}, which, however, is generally time-consuming due to the need to consider a large amount of samples for obtaining accurate estimates.

In Figure \ref{fig:SNN2GMM_class}, we compare the mean and covariance of various SNNs (obtained via Monte Carlo sampling using $1e4$ samples) and the relative GMM approximation on a set of $30$ randomly collected input points for various neural network architectures on the MNIST and CIFAR-10 data sets. We observe that the GMMs match the MC approximation of the true mean and kernel. 
However, unlike the GMM approximations, the MC approximations lack formal guarantees of correctness, and additionally, computing the MC approximations is generally two orders of magnitudes slower than computing the GMM approximations. 
By analyzing Figure \ref{fig:SNN2GMM_class}, it is possible to observe that since the architectures allow for the training of highly accurate networks, each row in the posterior mean reflects that the classes are correctly classified. For the networks trained using VI, the kernel matrix clearly shows the correlations trained by the SNN. For the dropout network, the GMM visualizes the limitations of dropout learning in capturing correlation among data points. This finding aligns with the result in \cite{gal2016dropout}, which shows that Bayesian inference via dropout is equivalent to VI inference with an isotropic Gaussian over the \changed{parameters} per layer.

\begin{figure}[htbp]
    \centering
    \begin{subfigure}[b]{0.48\textwidth}
        \includegraphics[width=\textwidth]{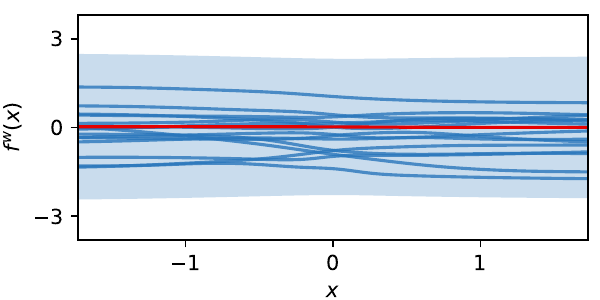}
        \caption{Uninformative SNN Prior}
        \label{fig:PriorTuning_SNNUnInf}
    \end{subfigure}
    \begin{subfigure}[b]{0.48\textwidth}
        \includegraphics[width=\textwidth]{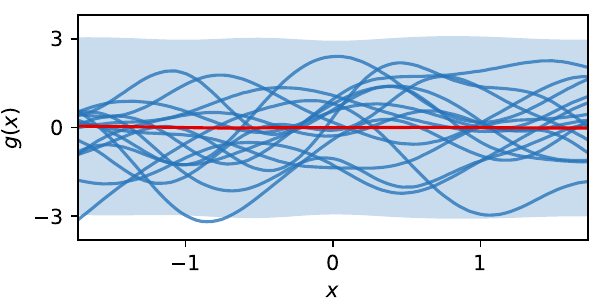}
        \caption{GP Prior}
        \label{fig:PriorTuning_GMM}
    \end{subfigure}
    \begin{subfigure}[b]{0.48\textwidth}
        \includegraphics[width=\textwidth]{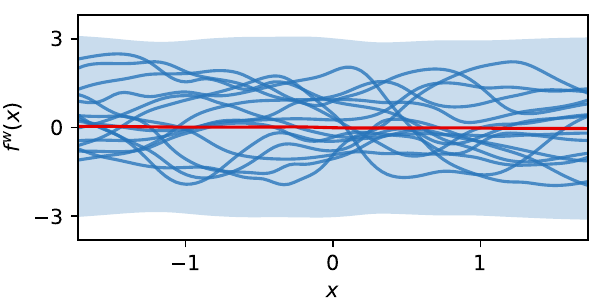}
        \caption{GP Induced SNN Prior (Ours)}
        \label{fig:PriorTuning_SNN}
    \end{subfigure}
    \begin{subfigure}[b]{0.48\textwidth}
        \includegraphics[width=\textwidth]{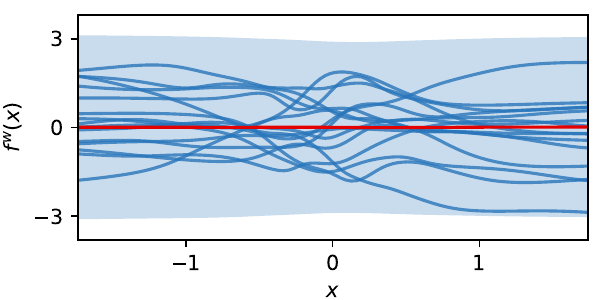}
        \caption{GP Induced SNN Prior \citep{tran2022all}}
        \label{fig:PriorTuning_SNN_bench}
    \end{subfigure}
    \caption{Visualization of the prior distribution of (a) a SNN with an isotropic Gaussian distribution over its \changed{parameters}; (b) a zero mean GP with the RBF kernel; (c-d) SNNs with a Gaussian distribution over its weight that is optimized to mimic the GP in (b) via, respectively, our framework and the approach proposed in \cite{tran2022all}. The SNN has 2 hidden layers, 128 neurons per layer, and the tanh activation function, while the GP has zero mean and a RBF kernel with length-scale $0.5$ and signal variance $1$.
    Red lines, blue shading, and blue lines respectively indicate the mean, $99.7\%$ credible intervals, and a subset of the MC samples from the output distributions.}
    \label{fig:exp1dPriorTuning}
\end{figure}

\subsection{Functional Priors for Neural Networks via GMM Approximation}\label{sec:expPriorTuning}
In this section, to highlight another application of our framework, we show how our results can be used to encode functional priors in SNNs. We start from Figure \ref{fig:PriorTuning_SNNUnInf}, where we plot functions sampled from a two hidden layer SNN with an isotropic Gaussian distribution over its \changed{parameters}, which is generally the default choice for priors in Bayesian learning \citep{fortuin2022priors}. As it is possible to observe in the figure, the sampled functions tend to form straight horizontal lines; this is a well-known pathology of neural networks, which becomes more and more exacerbated with deep architectures \citep{duvenaud2014avoiding},
Instead, one may want to consider a prior model that reflects certain characteristics that we expect from the underlying true function, e.g., smoothness or oscillatory behavior. Such characteristics are commonly encoded via a Gaussian process with a given mean and kernel \citep{tran2022all}, as shown for instance in Figure \ref{fig:PriorTuning_GMM}, where we plot samples from a GP with zero mean and the RBF kernel \citep{mackay1996bayesian}.
We now demonstrate how our framework can be used to encode an informative functional prior, represented as a GP, into a SNN. In particular, in what follows, we assume centered Mean-Field Gaussian priors on the \changed{parameters}, i.e., $p_\vw=\Ndist(\vzero,\diag(\psi))$, where  $\psi$ is a vector of possibly different parameters representing the variance of each weight.
Our objective is to optimize $\psi$ to minimize the 2-Wasserstein distance between the distributions of the GP and the SNN at a finite set of input points, which we do by following the approach described in Section \ref{sec:differentiability}. For the low-dimensional 1D regression problem, we use uniformly sampled input points. For the high-dimensional UCI problems, we use the training sets augmented with noisy samples. We solve the optimization problem as described in Section \ref{sec:differentiability}, setting $\beta$ to $0.01$, using mini-batch gradient descent with randomly sampled batches.

\begin{table}[htbp]
  \centering
    {\small
    \begin{tabular}{ll|rrr}
    \makecell[l]{Length-Scale \\ RBF Kernel} & \makecell[l]{Network \\ Architecture} & \makecell[c]{Uninformative \\ Prior} & \makecell[c]{GP Induced Prior \\ \citep{tran2022all}} & \makecell[c]{GP Induced Prior \\ (Ours)} \\
    \midrule
    1     & [2x64] & 0.97  &  0.30     & 0.25 \\
          & [2x128] & 0.46  & 0.30  & 0.25 \\ \midrule
    0.75  & [2x64] & 1.09  & 0.39  & 0.32 \\
          & [2x128] & 0.56  &  0.39     & 0.31 \\ \midrule
    0.5   & [2x64] & 1.23  & 0.53  & 0.47 \\
          & [2x128] & 0.69  &   0.53    & 0.43 \\ \bottomrule
    \end{tabular}
    }
    \caption{The empirical estimates of the relative 2-Wasserstein distance between the prior distributions of zero-mean GPs with the RBF kernel and SNNs with uninformative or GP-induced priors over the \changed{parameters}, at a subset of $20$ points from the test set. The GP-induced prior are either obtained via Algorithm \ref{al:SNN2GMM} with signatures of size $10$ and compression size (\changed{$M_k$}) of $1$ \changed{for all layers} as described in Section \ref{sec:differentiability}, or via the method in \cite{tran2022all}.}
  \label{tab:prior1d_EmpRelW2}
\end{table}

\paragraph{1D Regression Benchmark}
We start our analysis with the 1D example in Figure \ref{fig:exp1dPriorTuning}, where we compare our framework with the state-of-the-art approach for prior tuning by \cite{tran2022all}. Figure \ref{fig:PriorTuning_SNN} demonstrates the ability of our approach to encode the oscillating behavior of the GP in the prior distribution of the SNN. In contrast, the SNN prior obtained using the method in \cite{tran2022all} in Figure \ref{fig:PriorTuning_SNN_bench} is visually less accurate in capturing the correlation imposed by the RBF kernel of the GP. 
The performance difference can be explained by the fact that \cite{tran2022all} optimize the 1-Wasserstein distance between the GP and SNN, which relates to the closeness in the first moment (the mean), as they rely on its dual formulation. In contrast, we optimize for the 2-Wasserstein distance, which relates to the closeness in the second moment, including both mean and correlations (see Lemma \ref{lemma:1WasserBounds12Moment}). 
In Table \ref{tab:prior1d_EmpRelW2}, we investigate the quantitative difference in relative 2-Wasserstein distance between the prior distributions of the GP and SNN in Figure \ref{fig:exp1dPriorTuning}, among other network architectures and GP settings. 
The results show that our method consistently outperforms \cite{tran2022all} on all considered settings.

\paragraph{UCI Regression Benchmark}\label{sec:expPriorTuningUCI}
We continue our analysis on several UCI regression data sets to investigate whether SNNs with an informative prior lead to improved posterior performance compared to uninformative priors.
We evaluate the impact of an informative prior on the predictive performance of SNNs using the root-mean-square error (RMSE) and negative log-likelihood (NLL) metrics. The RMSE solely measures the accuracy of the mean predictions, whereas the NLL evaluates whether the predicted distribution matches the actual distribution of the outcomes, taking into account the uncertainty of the predictions. Lower RMSE and NLL values indicate better performance. 
Figure \ref{fig:uci} shows that for all data sets, the GP outperforms the SNN with an uninformative prior both in terms of mean and uncertainty accuracy. 
Encoding the GP prior in the SNN greatly improves its predictive performance, bringing it closer to that of the GP, with the remaining gap explained by the VI approximation error of the posterior distribution.

\begin{figure}[htbp]
    \centering
    \includegraphics[width=0.85\textwidth]{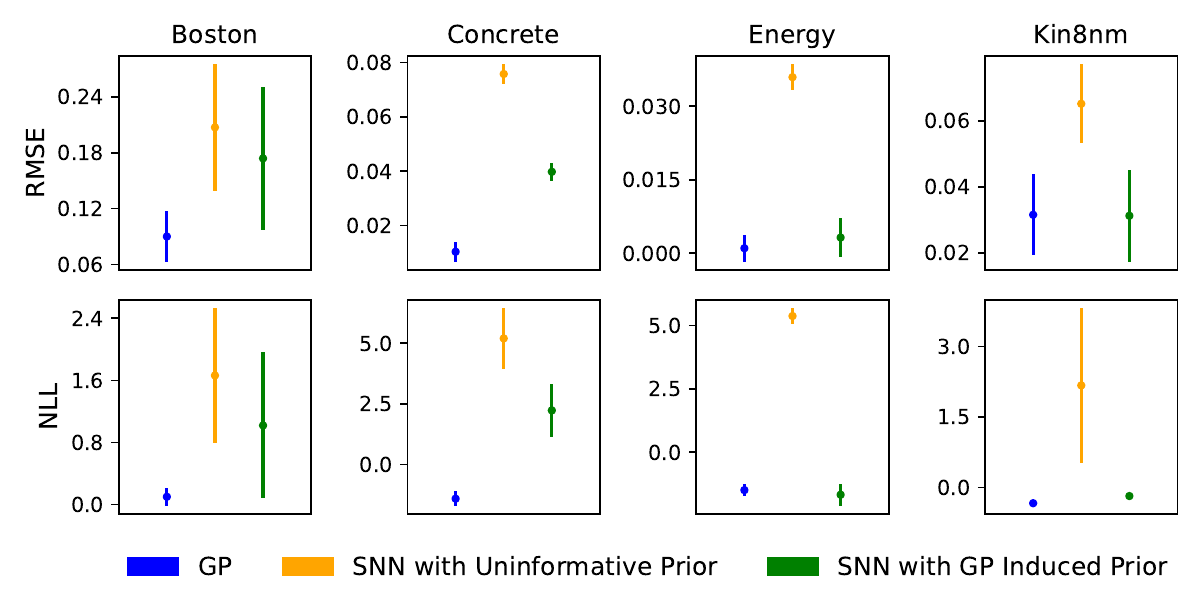}
    \caption{Report of the predictive performance of a GP with zero-mean and RBF kernel, and a $[2\times 128]$ SNN with tanh activation function on several UCI regression data sets in terms of the root-mean-squared error (RMSE) and negative log-likelihood (NLL). 
    The dots and bars represent the means and standard deviations over the 10 random test splits of the data sets, respectively.
    The GP results are shown in blue, the SNN results with an uninformative isotropic Gaussian prior are in yellow, and the results for the SNN with the GP-induced informative priors obtained via our approach are in green.
    }
    \label{fig:uci}
\end{figure}

\section{Conclusion}
We introduced a novel algorithmic framework to approximate the input-output distribution of a SNN with a GMM, providing bounds on the approximation error in terms of the Wasserstein distance. 
Our framework relies on techniques from quantization theory and optimal transport, and is based on a novel approximation scheme for GMMs using signature approximations. 
We performed a detailed theoretical analysis of the error introduced by our GMM approximations and complemented the theoretical analysis with extensive empirical validation, showing the efficacy of our methods in both regression and classification tasks, and for various applications, including prior selection and uncertainty quantification. 
\changed{An interesting future direction is to employ the framework in safety-critical control, planning, or decision-making applications, where our method critically enables principled uncertainty quantification.}
We should stress that in this paper we focused on finding GMM approximations of neural networks at a finite set of input points, thus also accounting for the correlations among these points. While reasoning for a finite set of input points is standard in stochastic approximations \citep{yaida2020non,basteri2022quantitative}
and has many applications, 
for some other applications, such as adversarial robustness of Bayesian neural networks \citep{wicker2020probabilistic}, one may require approximations that are valid for compact sets of input points. We believe this is an interesting future research question.

\acks{This work was supported in part by the NSF grant 2039062.
}

% \newpage
\appendix

\section{Properties of the Wasserstein Distance}\label{append:wasserGeneralProp}
\changed{The results in this work rely on several less commonly reported properties of the Wasserstein distance, which are recalled below.}
\changed{
\begin{lemma}[Closeness in \(\wasserstein_\rho\) Implies Closeness in \(\rho\)-Moments]\label{lemma:1WasserBounds12Moment}
    For \(p, q \in \probMeas_\rho(\eucl^n)\) and \(\rho \in \natNum\), it holds that
    \begin{equation*}%\label{eq:wasserBound2ndmoment}
        \left| \expect_{\vx \sim p}[\|\vx\|^\rho]^{1/\rho} - \expect_{\vz \sim q}[\|\vz\|^\rho]^{1/\rho} \right|
            \leq \wasserstein_\rho(p, q).
    \end{equation*}
\end{lemma}
\begin{proof} 
     The result follows directly from Proposition~7.29 in \cite{villani2009optimal}, applied to the $1$-Lipschitz function \(\phi(\vx) = \|\vx\|\).
\end{proof}
}

\changed{
\begin{lemma}[Translation and Rotation Invariance of \(\wasserstein_\rho\)]\label{lemma:wass_invariance_orthogonal_trans}
    Let $\tau:\realNum^n\to\realNum^n$ be an isometry of \(\realNum^n\), i.e., $\tau(\vx)=\mR\vx+\vb$ for some orthogonal matrix $\mR\in\realNum^{n\times n}$ and vector $\vb\in\realNum^n$. Then, for \(\rho\in\natNum\) and $p,q\in\probMeas(\realNum^n)$, we have
    \begin{equation*}
        \wasserstein_\rho(\tau\#p,\tau\#q) = \wasserstein_\rho(p,q).
    \end{equation*}
\end{lemma}
\begin{proof}
    Since \(\tau\) preserves the Euclidean norm, \(\|\tau(\vx) - \tau(\vz)\| = \|\vx - \vz\|\), any coupling \(\gamma\in\Gamma(p,q)\) is mapped to a coupling \((\tau \times \tau)\#\gamma \in \Gamma(\tau\#p,\tau\#q)\) with the same transport cost. The optimal value is therefore unchanged.
\end{proof}
}

\begin{lemma}[\(\wasserstein_2\) for Product Distributions]\label{lem:wasser2ViaMarginals}
    Let \(p,q \in \probMeas(\eucl^n) \) be joint distributions that are independent across dimensions. 
    Then,
    \begin{equation*}
        \wasserstein_2^2(p, q) = \sum_{l=1}^n \wasserstein_2^2\left(\elem{p}{l}, \elem{q}{l}\right),
    \end{equation*}
    where \( \elem{p}{l} \) and \( \elem{q}{l} \) denote the marginals of \( p \) and \( q \) on the \( l \)-th dimension. 
\end{lemma}
\begin{proof}
    \changed{
    See \cite[Section~2, fourth bullet point]{panaretos2019statistical}.
    }
\end{proof}

\begin{lemma}[\(\wasserstein_2\) Distance for Mixture Distributions]\label{lem:wasser4Mixtures}
    Let \(p,q\in\probMeas(\realNum^n)\) be mixture distributions, i.e., \(p = \sum_{i=1}^M\elem{\vpi_p}{i} p_i\) and \(q = \sum_{j=1}^N\elem{\vpi_q}{j}q_j\), where \(\vpi_p\in\Pi_M\), \(\vpi_q\in\Pi_N\), and \( p_i, q_j \in \probMeas(\realNum^n)\). Then,
    \begin{equation*}
        \wasserstein_2^2(p, q) \leq \inf_{\bar{\vpi}\in\Gamma(\vpi_p, \vpi_q)} \sum_{i=1}^M \sum_{j=1}^N \elem{\bar{\vpi}}{i,j}  \wasserstein_2^2(p_i, q_j),
    \end{equation*}
    where \(\Gamma(\vpi_p, \vpi_q)\subset\Pi_{M\times N}\) denotes the set of couplings between \(\vpi_p\) and \(\vpi_q\).
\end{lemma}
\begin{proof}
    Consider any coupling $\bar\vpi \in \Gamma(\vpi_p,\vpi_q)$, and for each pair $(i,j)$ let $\gamma_{i,j}$ be an optimal
    coupling between $p_i$ and $q_j$. Then the mixture
    \(
    \gamma = \sum_{i,j} \bar\pi_{i,j} \gamma_{i,j},
    \)
    is a valid coupling between $p$ and $q$ with transport cost
    \(\sum_{i,j} \bar\pi_{i,j} W_2^2(p_i,q_j)\). Taking the infimum over all \( \bar{\vpi} \in \Gamma(\vpi_p, \vpi_q) \) completes the proof.
\end{proof}

\changed{
\begin{lemma}[$\wasserstein_\rho$ Distance under Lipschitz Parametrized Stochastic Operations]\label{lemma:Wasserstein4StochLayer}
    Let $\rho\in\natNum$, $L^{\vw}:\eucl^n\rightarrow\eucl^m$ be a $\lipschitz_{L^\vw}$-Lipschitz continuous function parametrized by $\vw\in\realNum^d$, and let $\vw$ be distributed according to $p_\vw\in\probMeas_\rho(\realNum^d)$. Then, for any $p, q\in \probMeas_\rho(\eucl^n)$, it holds that 
    \begin{equation*}%\label{eq:WasserLinearLayerBound}
        \wasserstein_\rho^\rho\left(\expect_{\vx\sim p}\left[L^\vw(\vx)\#p_\vw\right],\expect_{\vz\sim q}\left[L^\vw(\vz)\#p_\vw\right]\right)
        \leq\expect_{\vw\sim p_\vw}\left[\lipschitz_{L^{\vw}}^\rho\right]\wasserstein_\rho^\rho(p,q). 
    \end{equation*}
\end{lemma}
\begin{proof}
    Since \(\|L^\vw(\vx)-L^\vw(\vz)\|\leq\lipschitz_{L^\vw}\|\vx-\vz\|\), it follows that
    \begin{align*}
        \wasserstein_\rho^\rho\left(\expect_{\vx\sim p}[L^\vw(\vx)\#p_\vw],\expect_{\vz\sim q}[L^\vw(\vz)\#p_\vw]\right)
        &= \inf_{\gamma\in\Gamma(p_\vw, p, q)}\expect_{(\vw,\vx,\vz)\sim\gamma}\left[\|L^\vw(\vx)-L^\vw(\vz)\|^\rho\right]\\
        &\leq \inf_{\gamma\in\Gamma(p_\vw, p, q)}\expect_{(\vw,\vx,\vz)\sim\gamma}\left[\lipschitz_{L^\vw}\|\vx-\vz\|^\rho\right].
    \end{align*}
    Using that $p_\vw \times \Gamma(p,q) \subset \Gamma(p_\vw,p,q)$, we obtain
    \begin{align*}
        \inf_{\gamma\in\Gamma(p_\vw, p, q)}\expect_{(\vw,\vx,\vz)\sim\gamma}\left[\lipschitz_{L^{\vw}}^\rho \|\vx-\vz\|^\rho\right]
        &\leq \expect_{\vw\sim p_\vw}\left[\lipschitz_{L^\vw}^\rho\right]\inf_{\gamma\in\Gamma(p, q)}\expect_{(\vx,\vz)\sim\gamma}\left[\|\vx-\vz\|^\rho\right] \\
        &= \expect_{\vw\sim p_\vw}\left[\lipschitz_{L^\vw}^\rho\right] \wasserstein_\rho^\rho(p, q),
    \end{align*}
    which concludes the proof.
\end{proof}
}

\section{Proofs}
We present the proofs for all results discussed in the main text of this work ordered per section.

\subsection{Proofs of the Results in Section \ref{sec:signatures}}\label{append:wasser4Signatures}
\subsubsection*{Proof of Proposition \ref{prop:Was4SignStdGaus}: $\wasserstein_2$-Error of the Signature of a Standard Gaussian}\label{prop;ExpecteSquaredEuclNorm}
    From Proposition~\ref{prop:OTSignature}, we have that
    \begin{align*}
        \wasserstein_2^2(\signature_{\sC}\#\Ndist(0,1),\Ndist(0,1)) &= 
        \sum_{i=1}^{N}\evpi{i}\expect_{z\sim \Ndist(0,1)}\left[(\vz-c_i)^2\mid z\in[l_i, u_i]\right].
    \end{align*}
    where 
    $$
        \evpi{i}=\prob_{z\sim\Ndist(0,1)}[z\in[l_i,u_i]] =\Phi(u_i)-\Phi(l_i).
    $$ 
    Using the closed-form expression for the mean $\mu_i\expect_{z\sim \Ndist(0,1)}[z\mid z\in[l_i, u_i]]$ and variance $\nu_i=\expect_{z\sim \Ndist(0,1)}\left[z^2\mid z\in[l_i, u_i]\right] - \mu_i^2$ of a standard univariate Gaussian random variable truncated on $[l_i, u_i]$ as given by Eqn~\eqref{eq:Was4SignStdGausTerms}, respectively, and derived in Section~3 of \cite{manjunath2021moments}, we can write the conditional expectation term as:
    \begin{align}
      &\expect_{z\sim \Ndist(0,1)}\left[(\vz-c_i)^2\mid z\in[l_i, u_i]\right] \nonumber \\
      &\qquad= \expect_{z\sim \Ndist(0,1)}\left[\vz^2\mid z\in[l_i, u_i]\right] - 2c_i\expect_{z\sim \Ndist(0,1)}\left[\vz\mid z\in[l_i, u_i]\right] + c_i^2 \nonumber \\
      &\qquad= \nu_i + \mu_i^2 - 2c_i \mu_i + c_i^2 \nonumber \\
      &\qquad= \nu_i + (\mu_i - c_i)^2. \label{eq:2ndMomentTruncGauss}
    \end{align}
    Substituting this expression into the original sum concludes the proof. 
\qed

\subsubsection*{Proof of Corollary~\ref{corol:Was4SignMultGaus}: \(\wasserstein_2\)-Error of the Signature of a Multivariate Gaussian}
    \changed{
    Let $\tau:\realNum^n\to\realNum^n$ denote the transformation defined by $\tau(\vx)=\mV^T(\vx-\vm)$, where \(\mV\) and \(\vm\) are, respectively, the (orthogonal) eigenvector matrix and mean of the Gaussian distribution. By Lemma~\ref{lemma:wass_invariance_orthogonal_trans}, the $2$-Wasserstein is invariant under translations and orthogonal transformations, and thus
    \begin{align*}
        \wasserstein^2_2(\signature_\sC\#\Ndist(\mNdist,\vNdist),\Ndist(\mNdist,\vNdist))=
        \wasserstein^2_2(\tau\#\signature_\sC\#\Ndist(\mNdist,\vNdist),\tau\#\Ndist(\mNdist,\vNdist)).
    \end{align*}
    Since, \(\Sigma=\mV\diag(\vlambda)\mV^T\), it holds that
    \[
        \tau\#\Ndist(\mNdist,\vNdist) = \Ndist(\vzero,\diag(\vlambda)) = \diag(\vlambda)^{1/2}\#\Ndist(\vzero,\mI),
    \]
    where, with a slight abuse of notation, we write \(\diag(\vlambda)^{1/2}\#p\) to denote the pushforward under the linear map \(\vx\mapsto\diag(\vlambda)^{1/2}\vx\).
    Next, consider the pushforward of the signature:
    \begin{align*}
        \tau\#\signature_\sC\#\Ndist(\mNdist,\vNdist)&=\signature_{\postImage(\sC-\{\vm\}, \mV^T)}\#\tau\#\Ndist(\mNdist,\vNdist) \\
        &=\signature_{\postImage(\sC-\{\vm\}, \mV^T)}\#\diag(\lambda)^{1/2}\#\Ndist(\vzero,\mI) \\
        % &=\signature_{\postImage(\sC-\{\vm\}, \diag(\lambda)^{1/2}\mT)}\#\diag(\lambda)^{1/2}\#\Ndist(\vzero,\mI) \\
        &=\diag(\lambda)^{1/2}\#\signature_{\postImage(\sC-\{\vm\}, \mT)}\#\Ndist(\vzero,\mI) \\
        &=\diag(\lambda)^{1/2}\#\signature_{\sC^1\times \sC^2\times\hdots\times\sC^n}\#\Ndist(\vzero,\mI),
    \end{align*}
    where the final equality holds due to the axis-aligned grid assumption in \eqref{eq:AxisAlignedGridAssumption}. 
    Since both terms are product distributions, we can apply Lemma~\ref{lem:wasser2ViaMarginals}, yielding
    \begin{align*}
        \wasserstein^2_2(\tau\#\signature_\sC\#\Ndist(\mNdist,\vNdist),\tau\#\Ndist(\mNdist,\vNdist)) &= \sum_{l=1}^n \wasserstein_2^2\left(\sqrt{\evlambda{l}}\#\Ndist(0,1),\sqrt{\evlambda{l}}\#\signature_{\sC^l}\#\Ndist(0,1)\right)\\
        &= \sum_{l=1}^n \wasserstein_2^2\evlambda{l}\left(\Ndist(0,1),\signature_{\sC^l}\#\Ndist(0,1)\right), 
    \end{align*}
    where the final equality follows from the homogeneity of the $2$-Wasserstein distance. This completes the proof.
    }
\qed

\changed{
\subsubsection*{Proof of Proposition~\ref{prop:Was4SignStdGausUniformGrid}: }
    If $N=1$, Eqn~\eqref{eq:uniform1DGrid} gives $\sC=\{0\}$ and Proposition~\ref{prop:OTSignature} yields
    $$
        \wasserstein_2^2(\signature_{\{0\}}\#\Ndist(0,1),\Ndist(0,1)) = \int z^2 \, \phi(dz)=1 = \frac{4\log 1}{1^2} + r(1) = r(1),
    $$
    since $\log 1=0$ and $r(1)=1$ by Eqn~\eqref{eq:uniformGridRemainder}, and where $\phi$ is the pdf of a standard (univariate) Gaussian distribution. Hence the bound holds for $N=1$. For the remainder of the proof assume $N\geq2$. Let $\eta=\frac{2\sqrt{\log N}}{N}$ and $\sC=\{c_i\}_{i=1}^N$ with $c_i= \eta(2i - N -1)$. The induced partition (Eqn~\ref{eq:VoronoiPartitionOfSignature}) is
    $$
        \sR_1=\left[-\infty, c_1 + \eta \right], \quad \sR_i=\left[c_i-\eta, c_i+\eta\right], \quad \forall i\in\{2,\hdots,N-1\}, \quad \sR_N=\left[c_N - \eta,\infty\right].
    $$
    By Proposition~\ref{prop:OTSignature}, we have
    \begin{align*}
        &\wasserstein_2^2(\signature_\sC\#\Ndist(0,1),\Ndist(0,1)) \\
        &\qquad=\sum_{i=2}^{N-1} \int_{c_i-\eta}^{c_i+\eta}(z-c_i)^2\phi(dz) + \int_{-\infty}^{c_1+\eta}(z-c_1)^2\phi(dz) + \int_{c_N-\eta}^{\infty}(z-c_N)^2\phi(dz)\\
        &\qquad=\sum_{i=1}^{N} \int_{c_i-\eta}^{c_i+\eta}(z-c_i)^2\phi(dz) + \int_{-\infty}^{-\eta N}(z-c_1)^2\phi(dz) + \int_{\eta N}^{\infty}(z-c_N)^2\phi(dz),
    \end{align*}
    where in the last step we split the unbounded regions and used $c_1=-\eta N+\eta$, $c_N=\eta N-\eta$.

    Bounded integrals: since $|z-c_i|\leq\eta$ for all $z\in\sR_i$,
    \begin{align*}
        \sum_{i=1}^{N} \int_{c_i-\eta}^{c_i+\eta}(z-c_i)^2\phi(dz)
        &\leq \eta^2 \sum_{i=1}^{N}[\Phi(c_i+\eta) - \Phi(c_i-\eta)]\\
        &=\eta^2\left[\Phi\left(\eta N\right) - \Phi\left(-\eta N\right)\right]
        =\eta^2\left[2\Phi\left(\eta N\right) - 1\right],
    \end{align*}
    where the last equalities follow from symmetry and telescoping of the CDF terms. 
    
    Unbounded integrals: by symmetry ($c_1=-c_N$),
    \begin{align*}
        &\int_{-\infty}^{-\eta N}(z-c_1)^2\phi(dz) + \int_{\eta N}^{\infty}(z-c_N)^2\phi(dz)
        =2\int_{\eta N}^{\infty}(z-c_N)^2\phi(dz).
    \end{align*}
    For $z\geq \eta N$, $z-c_N=z - \eta N + \eta \leq z$, hence
    \begin{align*}
        \int_{\eta N}^{\infty}(z-c_N)^2\Ndist(dz\mid 0,1)
        \leq  \int_{\eta N}^{\infty}z^2\Ndist(dz\mid 0,1)
        = [1-\Phi(\eta N)] + \eta N \phi(\eta N),
    \end{align*}
    using the truncated second-moment identity for standard Gaussians (Eqn~\ref{eq:2ndMomentTruncGauss}).

    Combining and denoting $L=\eta N=2\sqrt{\log N}$,
    \begin{align*}
        \wasserstein_2^2(\signature_\sC\#\Ndist(0,1),\Ndist(0,1))
        &\leq \eta^2\left[2\Phi\left(L\right) - 1\right] + 2[1-\Phi(L)] + 2L \phi(L)\\
        &= -2\eta^2\left[1 - \Phi\left(L\right)\right] + \eta^2 + 2[1-\Phi(L)] + 2L \phi(L)\\
        &=2(1- \eta^2)\left(1 - \Phi\left(L\right)\right) + \eta^2 + 2L\phi\left(L\right) \\
        &=2\left(1- \frac{L^2}{N^2}\right)\left(1 - \Phi\left(L\right)\right) + \frac{L^2}{N^2} + 2L\phi\left(L\right)\tag{$*$}
    \end{align*}
    By Mills' ratio $1-\Phi(L) \leq \frac{\phi(L)}{L}$ and $0\leq \frac{L^2}{N^2} \leq 1$, we bound the first term in $(*)$ by $2\frac{\phi(L)}{L}$, yielding
    \begin{align*}
        (*) &\leq 2\frac{\phi(L)}{L} + \frac{L^2}{N^2} + 2L\phi(L) = 2\phi(L)\left(L + \frac{1}{L}\right) + \frac{L^2}{N^2}. 
    \end{align*}
    With $L=2\sqrt{\log N}$ we have $\phi(L)=\frac{1}{\sqrt{2\pi}}e^{-L^2/2}=\frac{1}{\sqrt{2\pi}}e^{-2\log N}=\frac{1}{N^2\sqrt{2\pi}}$, so
    \[
        2\phi(L)\left(L+\tfrac{1}{L}\right)= \frac{4}{N^{2}\sqrt{2\pi}}\left(\sqrt{\log N}+\frac{1}{4\sqrt{\log N}}\right).
    \]
    Therefore
    \[
        \wasserstein_2^2(\signature_\sC\#\Ndist(0,1),\Ndist(0,1)) \leq \frac{4\log N}{N^2} + \frac{4}{N^{2}\sqrt{2\pi}}\left(\sqrt{\log N}+\frac{1}{4\sqrt{\log N}}\right) = \frac{4\log N}{N^2} + r(N).
    \]
    This completes the proof.
\qed
}

\changed{
\subsubsection*{Proof of Corollary~\ref{corol:WasConvergence4SignMixtureGaus}}
    From Eqn~\eqref{eq:Was4SignMixMultGaus}, the $2$-Wasserstein error of the component-wise signature of a Gaussian mixture is bounded by the mixture-weighted sum of the signature errors of its components:
    \[
        \wasserstein_2^2\big(p,\bm\signature_{\{\sC_i\}_{i=1}^M}\#p\big) \leq \sum_{i=1}^M \tilde{\pi}_i\, \wasserstein_2^2\big(\Ndist(m_i,\Sigma_i), \signature_{\sC_i}\#\Ndist(m_i,\Sigma_i)\big).
    \]
    For each component $i$, diagonalize $\Sigma_i=\mV_i\diag(\vlambda_i)\mV_i^T$ and apply Corollary~\ref{corol:Was4SignMultGaus} together with the axis-aligned grid assumption \eqref{eq:AxisAlignedGridAssumption} to obtain
    \[
        \wasserstein_2^2\big(\Ndist(m_i,\Sigma_i), \signature_{\sC_i}\#\Ndist(m_i,\Sigma_i)\big)=\sum_{j=1}^n \lambda_{i,j}\, \wasserstein_2^2\big(\Ndist(0,1), \signature_{\sC_i^j}\#\Ndist(0,1)\big).
    \]
    Proposition~\ref{prop:Was4SignStdGausUniformGrid} applied to each one-dimensional uniform grid of points $\sC_i^j$ (of size $N_{i,j}$) yields
    \[
        \wasserstein_2^2\big(\Ndist(0,1), \signature_{\sC_i^j}\#\Ndist(0,1)\big) \leq \frac{4\log N_{i,j}}{N_{i,j}^2} + r(N_{i,j}),
    \]
    with $r(\cdot)$ given in Eqn~\eqref{eq:uniformGridRemainder}. Combining the three displays finishes the proof:
    \[
        \wasserstein_2^2\big(p,\bm\signature_{\{\sC_i\}_{i=1}^M}\#p\big) \leq \sum_{i=1}^M \tilde{\pi}_i \sum_{j=1}^n \lambda_{i,j}\left(\frac{4\log N_{i,j}}{N_{i,j}^2}+ r(N_{i,j})\right).\qed
    \]
}

\subsection{Proofs of the Results in Section \ref{sec:approxPredPostByGMM}}\label{append:wasser4SNNs}
\changed{
\subsubsection*{Proofs of the \(\wasserstein_2\) Distance Bounds in Table~\ref{tab:wasser4layers}}
The bound for the signature and compression operations directly follows from the triangle inequality. 
The bound for the affine transformation is a direct application of Lemma~\ref{lemma:Wasserstein4StochLayer}, as is the bound on \(\wasserstein(\act\#p,\act\#q)\).

To derive the bound on \(\wasserstein^2_2(\act\#p,\act\#\signature_\sC\#p)\), we invoke Proposition~\ref{prop:OTSignature}, which yields
\begin{align*}
    \wasserstein_\rho^\rho(\act\#p, \act\#\signature_\sC\#p) = \sum_{i=1}^N \evpi{i}\expect_{\vz\sim p}[\|\act(\vz)-\act(\vc_i)\|^\rho\mid \vz \in \sR_i].
\end{align*}
Since the activation function $\act$ is $\lipschitz_{\act\mid\sR_i}$-Lipschitz on region $\sR_i$, we have $\|\act(\vz) - \act(\vc_i) \| \leq \lipschitz_{\act\mid\sR_i}\|\vz-\vc_i\|$ for all  $\vz,\vc_i\in\sR_i$. Substituting this bound gives
\begin{align*}
    \sum_{i=1}^N \evpi{i}\expect_{\vz\sim p}[\|\act(\vz)-\act(\vc_i)\|^\rho\mid \vz \in \sR_i] &\leq \sum_{i=1}^N \lipschitz_{\act\mid\sR_i}^\rho\evpi{i}\expect_{\vz\sim p}[\|\vz-\vc_i\|^\rho\mid \vz \in \sR_i].
\end{align*} 
}

\changed{
\subsubsection*{Proof of Theorem~\ref{thm:WasserNetwork}: $\wasserstein_2$-Error of the GMM Approximation of a SNN at a Single Input Point}
    We prove Theorem~\ref{thm:WasserNetwork} via induction and generalize the result to the $\rho$-Wasserstein distance with $\rho\in\{1,2\}$. To do so, we first derive the base case of the induction.
    For $k\in\{1,\hdots,K\}$,  we have that:
    \begin{align*}
        &\wasserstein_\rho\left(p_{nn(\vx),k+1},q_{nn(\vx),k+1}\right) \\ 
        &\qquad\textrm{(By Eqns \ref{eq:p_nn(x),k} \& \ref{eq:q_nn(x),k})}\\
        &=\wasserstein_\rho\left(\expect_{\state_k\sim p_{nn(\vx),k}}\left[L^{\vw_k}(\act(\state_k))\#p_{\vw_k}\right],\expect_{\Tilde{\state}_k\sim \compress_{\changed{M_k}}(q_{nn(\vx),k})}\left[L^{\vw_k}(\act(\signature_{\sC_k}(\Tilde{\state}_k)))\#p_{\vw_k}\right]\right) \\
        &\qquad\textrm{(By the bound for the affine operation in Table~\ref{tab:wasser4layers})} \\
        &\leq\expect_{\vw_k\sim p_{\vw_k}}\left[\lipschitz_{L^{\vw_k}}^\rho\right]^{1/\rho}\wasserstein_\rho\left(\act\#p_{nn(\vx),k},\act\#\signature_{\sC_k}\#\compress_{\changed{M_k}}(q_{nn(\vx),k})\right) \\
        &\qquad \textrm{(By the triangle inequality)} \\
        &\leq\expect_{\vw_k\sim p_{\vw_k}}\left[\lipschitz_{L^{\vw_k}}^\rho\right]^{1/\rho}\left[\wasserstein_\rho\left(\act\#p_{nn(\vx),k},\act\#q_{nn(\vx),k}\right) + \wasserstein_\rho\left(\act\#q_{nn(\vx),k},\act\#\compress_{\changed{M_k}}(q_{nn(\vx),k})\right) \right.
        \\ &\qquad
        \left.+\wasserstein_\rho\left(\act\#\compress_{\changed{M_k}}(q_{nn(\vx),k}),\act\#\signature_{\sC_k}\#\compress_{\changed{M_k}}(q_{nn(\vx),k})\right)\right] \\
        &\qquad\textrm{(By Lemma~\ref{lemma:Wasserstein4StochLayer}, i.e., applying the bound for the activation operation in Table~\ref{tab:wasser4layers})}\\
        &\leq\expect_{\vw_k\sim p_{\vw_k}}\left[\lipschitz_{L^{\vw_k}}^\rho\right]^{1/\rho}\lipschitz_\act\left[\wasserstein_\rho\left(p_{nn(\vx),k},q_{nn(\vx),k}\right) + \wasserstein_\rho\left(q_{nn(\vx),k},\compress_{\changed{M_k}}(q_{nn(\vx),k})\right)\right.
        \\ &\qquad
        \left.+\wasserstein_\rho\left(\compress_{\changed{M_k}}(q_{nn(\vx),k}),\signature_{\sC_k}\#\compress_{\changed{M_k}}(q_{nn(\vx),k})\right)\right] \\
        &\qquad \textrm{(By the definition of the $\mw_2$-distance, and because $\wasserstein_1\leq\wasserstein_2$)} \\
        &\leq\expect_{\vw_k\sim p_{\vw_k}}\left[\lipschitz_{L^{\vw_k}}^\rho\right]^{1/\rho}\lipschitz_\act\left[\wasserstein_\rho\left(p_{nn(\vx),k},q_{nn(\vx),k}\right) + \mw_2\left(q_{nn(\vx),k},\compress_{\changed{M_k}}(q_{nn(\vx),k})\right)\right.
        \\ &\qquad
        \left.+\wasserstein_\rho\left(\compress_{\changed{M_k}}(q_{nn(\vx),k}),\signature_{\sC_k}\#\compress_{\changed{M_k}}(q_{nn(\vx),k})\right)\right] \\
        &\qquad \textrm{(By Eqn~\ref{eq:Was4SignMixMultGaus} and because $\wasserstein_1\leq\wasserstein_2$)} \\
        &\leq\expect_{\vw_k\sim p_{\vw_k}}\left[\lipschitz_{L^{\vw_k}}^\rho\right]^{1/\rho}\lipschitz_\act\left[\wasserstein_\rho\left(p_{nn(\vx),k},q_{nn(\vx),k}\right) + \mw_2\left(q_{nn(\vx),k},\compress_{\changed{M_k}}(q_{nn(\vx),k})\right)\right.
        \\ &\qquad
        \left.+\hat\wasserstein_{2,\bm\signature_{\bm\sC_k}}\left(\compress_{M_k}(q_{nn(\vx),k})\right)\right]
    \end{align*}
    Hence, if \(\wasserstein_\rho\left(p_{nn(\vx),k},q_{nn(\vx),k}\right)\leq \hat\wasserstein_{\rho,k}\), then
    \begin{align*}
        \wasserstein_\rho\left(p_{nn(\vx),k+1},q_{nn(\vx),k+1}\right) 
        &\leq\expect_{\vw_k\sim p_{\vw_k}}\left[\lipschitz_{L^{\vw_k}}^\rho\right]^{1/\rho}\lipschitz_\act\left[\hat{\wasserstein}_{\rho,k} + \mw_2\left(q_{nn(\vx),k},\compress_{\changed{M_k}}(q_{nn(\vx),k})\right)\right.
        \\ &\qquad
        \left.+\hat\wasserstein_{2,\bm\signature_{\bm\sC_k}}\left(\compress_{M_k}(q_{nn(\vx),k})\right)\right] = \hat{\wasserstein}_{\rho,k+1}
    \end{align*}
    The assumption is naturally satisfied for $k=1$, such that by induction, it holds that
    $$
        \wasserstein_\rho\left(p_{nn(\vx)},q_{nn(\vx)}\right)\leq\Hat{\wasserstein}_{\rho,K+1}.
    $$
\qed
}

\subsubsection*{Proof of Corollary~\ref{corol:WasserNetworkSet}: $\wasserstein_2$-Error of the GMM Approximation of a SNN at a Set of Input Points}
\changed{
    The corollary follows directly from Theorem~\ref{thm:WasserNetwork} by replacing $L^{\vw_k}$ with the stacked affine operation 
    $$
        \bm{L}^{\vw_k}\left((\vz_1^T,\hdots,\vz_D^T)^T\right) = \left(L^{\vw_k}(\vz_1)^T,\hdots, L^{\vw_k}(\vz_D)^T\right)^T,
    $$
    and noting that the Lipschitz constant of $\bm{L}^{\vw_k}$ can be taken equal to the Lipschitz constant of $L^{\vw_k}$.
    To see the latter, note that for any $\bar\vz = \left(\vz_1^T, \ldots, \vz_D^T\right)^T$ and $\bar\vz' = \left(\vz_1'^T, \ldots, \vz_D'^T\right)^T$, we have:
    \begin{align*}
        \|\bm{L}^{\vw_k}(\bar\vz) - \bm{L}^{\vw_k}(\bar\vz')\|^2 
        = \sum_{i=1}^D \|L^{\vw_k}(\vz_i) - L_{\vw_k}(\vz_i')\|^2 
        \leq \sum_{i=1}^D \lipschitz_{L_{\vw_k}}^2 \|\vz_i - \vz_i'\|^2 
        = \lipschitz_{L_{\vw_k}}^2 \|\bar\vz - \bar\vz'\|^2.
    \end{align*}
    Therefore,
    $
        \expect_{\vw_k\sim p_{\vw_k}}\left[ \lipschitz_{\bm{L}^{\vw_k}}^2\right]^{1/2} \leq \expect_{\vw_k\sim p_{\vw_k}}\left[ \lipschitz_{L^{\vw_k}}^2\right]^{1/2}.
    $
    and the same bound from Theorem~\ref{thm:WasserNetwork} applies to the set-valued input case.
\qed
}

\changed{
\subsubsection*{Proof of Theorem~\ref{theorem:ConvergenceWasserBound}: Convergence of the \(\wasserstein_\rho\)-Error of the GMM Approximation of a SNN}
    We prove the theorem for the $\rho$-Wasserstein with $\rho \in 1,2$. 
    Let us denote 
    $$
        \hat\wasserstein_{2,\compress_{M_k}}=\mw_2\left(q_{nn(\sX),k}, \compress_{\changed{M_k}}(q_{nn(\sX),k})\right), \quad \text{and} \quad
        \hat\wasserstein_{2,\bm\signature_{\bm\sC_k}}=\hat\wasserstein_{2,\bm\signature_{\bm\sC_k}}\left(\compress_{\changed{M_k}}(q_{nn(\sX),k})\right).
    $$ 
    From Corollary~\ref{corol:WasserNetworkSet}: 
    \begin{align*}
        &\wasserstein_\rho(p_{nn(\sX)}, q_{nn(\sX)}) \\
        &\qquad\leq \lipschitz_K\left[\hat\wasserstein_{\rho,K} +  \hat\wasserstein_{2,\compress_{M_K}}+\hat\wasserstein_{2,\bm\signature_{\bm\sC_K}} \right]\\
        &\qquad=\lipschitz_{K-1}\lipschitz_K\left[\wasserstein_{\rho,K-1} + \hat\wasserstein_{2,\compress_{M_{K-1}}}+\hat\wasserstein_{2,\bm\signature_{\bm_{K-1}}}\right] + \lipschitz_K\left[ \wasserstein_{2,\compress_{M_K}}+\wasserstein_{2,\bm\signature_{\bm\sC_K}} \right]\\
        &\qquad\vdots\\
        &\qquad=\left(\prod_{i=1}^K\lipschitz_{i}\right)\left[\hat\wasserstein_{\rho,1}+\hat\wasserstein_{2,\compress_{M_1}}+\hat\wasserstein_{2,\bm\signature_{\bm\sC_1}}\right] + \sum_{k=2}^{K}\left(\prod_{i=k}^{K}\lipschitz_{i}\right)\left[ \hat\wasserstein_{2,\compress_{M_k}}+\hat\wasserstein_{2,\bm\signature_{\bm\sC_k}} \right]\\
        &\qquad=\sum_{k=1}^{K}\left(\prod_{i=k}^{K}\lipschitz_{i}\right) \lipschitz_k\left[ \hat\wasserstein_{2,\compress_{M_k}}+\hat\wasserstein_{2,\bm\signature_{\sC_k}}\right]
    \end{align*}
    According to Corollary~\ref{corol:WasConvergence4SignMixtureGaus}, the construction of the signature sets \(\sC_k\) in the theorem ensures that
    \[
        \hat\wasserstein_{2,\bm\signature_{\bm\sC_k}} \leq \epsilon_k = \frac{\epsilon}{K\prod_{i=k}^{K}\lipschitz_{i}} - \hat\wasserstein_{2,\compress_{M_k}}.
    \]
    Substituting this into the total bound, we get:
    \begin{align*}
        \sum_{k=1}^{K}\left(\prod_{i=k}^{K}\lipschitz_{i}\right)\left[ \hat\wasserstein_{2,\compress_{M_k}}+\hat\wasserstein_{2,\bm\signature_{\sC_k}}\right]
        &\leq \sum_{k=1}^{K}\left(\prod_{i=k}^{K}\lipschitz_{i}\right)\left[ \hat\wasserstein_{2,\compress_{M_k}}+\epsilon_k\right] \\
        &= 
        \sum_{k=1}^{K}\left(\prod_{i=k}^{K}\lipschitz_{i}\right)\left[
        \frac{\epsilon}{K\prod_{i=k}^{K}\lipschitz_{i}}
        \right] =\sum_{k=1}^{K}
        \frac{\epsilon}{K} = \epsilon.
    \end{align*}
    This completes the proof.
\qed
}

\subsection{Proofs of the Results in Section \ref{sec:AlgorithmicSection}}
\subsubsection*{Proof of Corollary \ref{corol:Wasser4SNNtoGMM} on the $\wasserstein_2$  Distance between any GMM and a SNN}
    Since $\wasserstein_2$ satisfies the triangle inequality, we have:
    $$
        \wasserstein_2\left(p_{nn(\sX)}, q_{gmm(\sX)}\right) \leq \wasserstein_2\left(p_{nn(\sX)}, q_{nn(\sX)}\right) + \wasserstein_2\left(q_{nn(\sX)},q_{gmm(\sX)}\right),
    $$
    Given that $\mw_2$ is an upper bound for the $2$-Wasserstein distance, we have: 
    $$
        \wasserstein_2\left(q_{nn(\sX)},q_{gmm(\sX)})\leq\mw_2(q_{nn(\sX)},q_{gmm(\sX)}\right).
    $$ 
    Combining these two inequalities, we obtain:
    $$
        \wasserstein_2\left(p_{nn(\sX)}, q_{gmm(\sX)}\right) \leq \wasserstein_2\left(p_{nn(\sX)}, q_{nn(\sX)}\right) + \mw_2\left(q_{nn(\sX)},q_{gmm(\sX)}\right),
    $$
    which completes the proof.
\qed

\subsubsection*{Proof of Proposition \ref{prop:OptimalGrid} on $\wasserstein_2$-Optimal Grid for Signatures for Gaussians}
    According to Corollary~\ref{corol:Was4SignMultGaus}:
    \begin{equation*}
         \wasserstein_2\left(\signature_{\postImage(\sC^1\times\hdots\times\sC^n, \mT^{-1})+\{\mNdist\}}\#\Ndist(\mNdist,\vNdist),\Ndist(\mNdist,\vNdist)\right)=\sum_{j=1}^n \evlambda{j}\wasserstein^2_2\left(\signature_{\sC^j}\#\Ndist(0,1),\Ndist(0,1)\right).
    \end{equation*} 
    This relation allows us to split the optimization problem of interest in $n$ smaller optimization problems as follows: 
    \begin{align*}
        &\argmin{\substack{\{\sC^1,\hdots,\sC^n\}\subset\eucl, \\ \sum_{i=1}^n|\sC^j|=N}}
        \wasserstein_2\left(\signature_{\postImage(\sC^1\times\hdots\times\sC^n, \mT^{-1})+\{\mNdist\}}\#\Ndist(\mNdist,\vNdist),\Ndist(\mNdist,\vNdist)\right)  \\
        &\qquad=\argmin{\substack{\{\sC^1,\hdots,\sC^n\}\subset\eucl, \\ \sum_{i=1}^n|\sC^j|=N}}
        \sum_{j=1}^n \evlambda{j}\wasserstein^2_2\left(\signature_{\sC^j}\#\Ndist(0,1),\Ndist(0,1)\right),\\
        &\qquad=\argmin{\substack{\{\sC^1,\hdots,\sC^n\}\subset\eucl, \\ \sC^j=N_j^*, \forall j\in\{1,\hdots,n\}}}
        \sum_{j=1}^n \evlambda{j}\wasserstein^2_2\left(\signature_{\sC^j}\#\Ndist(0,1),\Ndist(0,1)\right),\\
        &\qquad=\left\{\argmin{\sC^j\subset\eucl, \sC^j=N_j^*}
        \wasserstein^2_2\left(\signature_{\sC^j}\#\Ndist(0,1),\Ndist(0,1)\right)\right\}_{j=1}^n,
    \end{align*}
    where in step 2 we used the optimal grid-configuration $\{N^*_1,\hdots,N^*_n\}$ as defined in Eqn\eqref{eq:OptimalGridSizes}, and in step 3, we use the independence of objective terms across dimensions.
\qed

\section{Expected Lipschitz Constant for Stochastic Affine Operations}\label{append:expected_lipscthiz}
\changed{Here, we show how to compute the expected value of a Lipschitz constant \(\expect_{\vw_k\sim p_{\vw_k}}[\lipschitz_{L^{\vw_k}}^2]^{\frac{1}{2}}\) appearing in Theorem~\ref{thm:WasserNetwork}, in the case where $L^{\vw_k}$ represents either a fully-connected or convolutional layer, and where $\rho\in\{1,2\}$. 

\paragraph{Fully-connected layers.}
For a fully-connected layer of the form \(L^\vw(\vz) = \mW\vz + \vb\), the Lipschitz constant is \(\lipschitz_{L^{\vw}} = \|\mW\|\), where \(\|\cdot\|\) denotes the spectral norm. When the matrix \(\mW\) is distributed according to \(p_\mW \in \probMeas_\rho(\realNum^{m\times n})\), computing \(\expect[\|\mW\|^\rho]^{1/\rho}\) is generally intractable. Instead of using the spectral norm, we therefore consider the Frobenius norm, which yields a valid Lipschitz constant: while it is not the smallest such constant, it provides a tractable upper bound. Specifically, that
\begin{align*}
    \expect_{\mW\sim p_\mW}[\|\mW\|^\rho]^{1/\rho} \leq \expect_{\mW\sim p_\mW}\left[\|\mW\|_\mathcal{F}^\rho\right]^{1/\rho}.
\end{align*}
This bound can be tightened by centering the distribution of $\mW$ around its mean $\mM=\expect_{\mW\sim p_\mW}[\mW]$ and by applying Minkowski's inequality:
\begin{align*}
    \expect_{\mW\sim p_\mW}[\|\mW\|^\rho]^{1/\rho}
    % &= \expect_{\mW\sim p_\mW}[\|\mW - \mM + \mM\|^\rho]^{1/\rho} \\
    \leq \expect_{\mW\sim p_\mW}[\|\mW - \mM\|^\rho]^{1/\rho} + \|\mM\|
    \leq \expect_{\mW\sim p_\mW}\left[\|\mW - \mM\|_\mathcal{F}^\rho\right]^{1/\rho} + \|\mM\|.
\end{align*}
Here, we again use that the spectral norm is upper bounded by the Frobenius norm.

For \(\rho = 2\), the upper bound simplifies in terms of second moments of the entries of \(\mW\). For \(\rho = 1\), we additionally apply Jensen's inequality to ensure tractability:
\begin{align*}
    \expect_{\mW\sim p_\mW}\left[\|\mW - \mM\|_\mathcal{F}\right]
    \leq \expect_{\mW\sim p_\mW}\left[\|\mW - \mM\|_\mathcal{F}^2\right]^{1/2}.
\end{align*}

\paragraph{Convolutional layers.}
If \(L^{\vw_k}\) corresponds to a 2D convolutional operation with kernel weight matrix \(\mW\), then the Lipschitz constant is given by \(\lipschitz_{L^\vw} = \sqrt{m} \|\mW\|\), where \(m\) denotes the maximum number of times any input element contributes to the output (e.g., the number of spatial positions a kernel covers) \cite{gouk2021regularisation}. For example, with unit stride and padding of one, \(m\) equals the product of the kernel height and width. We apply the same bounding procedure as in the fully-connected setting, now including the multiplicative \(\sqrt{m}\) correction factor.
}

\changed{
\section{Alternative Signature Placement: Cross Scheme}\label{append:altApproxScheme}
As discussed in Remark~\ref{remark:GridVsSigmaPointSignatures}, we constrained signature locations to axis-aligned grids in order to obtain formal error bounds (Corollary~\ref{corol:Was4SignMultGaus}). When formal bounds are not required, we may instead place signature locations on a \emph{cross}-aligned with the principal axes of each Gaussian component. This ``cross scheme'' resembles classical sigma–point constructions \citep{julier2004unscented} in that points lie along eigenvector directions of the covariance; however, it fundamentally differs by (i) allowing an arbitrary number of points per axis, and (ii) assigning exact probabilities via the Voronoi partition w.r.t. to the points, rather than choosing weights to satisfy a finite set of moment constraints. 

\begin{figure}[htbp]
    \centering
    \includegraphics[width=0.7\textwidth]{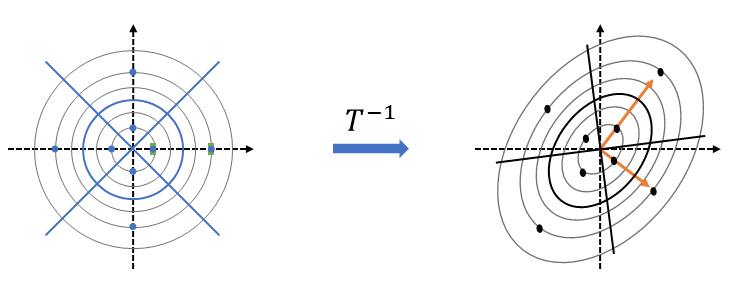}
    \caption{\changed{Construction of the cross-shaped signature of a 2D Gaussian (black dots) using the signature of a standard Gaussian (blue dots) as a reference. In the whitened eigenbasis space defined by $\mT$ (orange axes), we place symmetric points along each principal axis by mirroring the template locations (green markers) across all axes to form the cross. Mapping these points back via $\mT^{-1}$ yields the signature locations in the original space. The blue (left) and black (right) lines indicate the boundaries of the Voronoi regions induced by the signature locations in the transformed and original spaces, respectively. 
    }}
    \label{fig:signatureGaussianCross}
\end{figure}

The procedure to compute the component-wise signature of a Gaussian mixture (see Eqn~\ref{eq:Was4SignMixMultGaus}) is illustrated in Figure~\ref{fig:signatureGaussianCross} and summarized in Algorithm~\ref{al:signaturesOfGaussianMixutresCross}. First (line~2) we obtain, via Algorithm~\ref{al:CrossSignatureRadii}, the radial template (axis radii) in the transformed space (green stripes in Figure~\ref{fig:signatureGaussianCross}).
Note that we set the number of radial levels $L = \max\{1, \lfloor N_{\text{cross}}/(2n) \rfloor\}$ (line~1), allocating $2n$ points per level (two per axis) while ensuring at least one level, where $\lfloor\cdot\rfloor$ denotes the floor function. Next we compute the probabilities of the Voronoi regions induced by the eventual signature points by: (i) evaluating cumulative probabilities of the $n$-dimensional standard Gaussian ball at the radial Voronoi boundaries using the $\chi^2_n$ CDF (line~4); and (ii) differencing these to obtain shell masses $\Delta_\ell$ (line~5). The weight of each point at radial level $\ell$ is then $w_\ell = \Delta_\ell/(2n)$ (line~6). For each mixture component (lines~7-9), at every radial level $\ell$ and along every principal axis $j$, we place two symmetric points $\mNdist_i \pm r_\ell s_{i,j} \va_{i,j}$ with weight $w_\ell$ and map them back to the original space. The heuristic in Algorithm~\ref{al:CrossSignatureRadii} ensures that the resulting signature matches the second moment of the corresponding Gaussian component.

\begin{algorithm}[htbp]
\DontPrintSemicolon
\SetKwInOut{Input}{input}\SetKwInOut{Output}{output}
\changed{
\Input{$q = \sum_{i=1}^{M}\elem{\Tilde{\vpi}}{i}\, \Ndist(\mNdist_i, \vNdist_i)$, desired number of signature locations $N_{\text{cross}}$}
\Output{Signature $d = \sum_{j=1}^N \evpi{j} \delta_{\vc_j}$}
\Begin{
    \nl Set number of radial levels $L = \max\{1, \lfloor N_{\text{cross}}/(2n) \rfloor\}$\;
    \nl Obtain radii $\{r_{\ell}\}_{\ell=1}^L$ via Algorithm~\ref{al:CrossSignatureRadii} \;
    \nl Define radial Voronoi boundaries $u_{\ell} = \tfrac{1}{2}(r_{\ell} + r_{\ell+1})$ for $\ell=1,\dots,L-1$ \;
    \nl Compute cumulative $n$-D standard normal ball probabilities $B(u_{\ell}) = P(\|Z\|_2 \le u_{\ell}),\; Z\sim\Ndist(0,\mI_n)$ using $B(u_{\ell}) = \Phi_{\chi^2_n}(u_{\ell}^2)$ \;
    \nl Set shell masses $\Delta_1 = B(u_1)$, $\Delta_{\ell} = B(u_{\ell}) - B(u_{\ell-1})$ for $\ell=2,\dots,L-1$, and $\Delta_L = 1 - B(u_{L-1})$ \;
    \nl Compute per-level point weights $w_{\ell} = \Delta_{\ell}/(2n)$ \;
    \For{$i\in\{1,\dots,M\}$}{
        \nl Compute eigenvalues vector $\vlambda_i$ and eigenvector matrix $\mV_i$ of $\vNdist_i$ \;
        \nl Set unit axis directions $\va_{i,j} = \mV_i\ve_j$ and scales $s_{i,j}=\sqrt{\elem{\vlambda_i}{j}}$ for $j=1,\dots,n$ \;
            \nl Construct $d_i = \sum_{\ell=1}^L \sum_{j=1}^n w_{\ell}\Big(\delta_{\mNdist_i + r_{\ell} s_{i,j} \va_{i,j}} + \delta_{\mNdist_i - r_{\ell} s_{i,j} \va_{i,j}}\Big)$ \;
        }
        \nl Return $d = \sum_{i=1}^M \elem{\Tilde{\vpi}}{i} d_i$ \;
}}
\caption{Cross Signatures of Gaussian mixtures}\label{al:signaturesOfGaussianMixutresCross}
\end{algorithm}

\begin{algorithm}[htbp]
\DontPrintSemicolon
\SetKwInOut{Input}{input}\SetKwInOut{Output}{output}
\changed{
\Input{Number of radial levels $L$, dimension $n$}
\Output{Axis radii $\{r_\ell\}_{\ell=1}^L$}
\Begin{
        \nl Define probability edges $\{p_\ell\}_{\ell=0}^{L}$ with $p_\ell = 0.5 + \tfrac{\ell}{2L}$ (uniform partition of $[0.5,1]$) \;
        \nl Compute quantiles $q_{\ell} = \Phi^{-1}(p_{\ell})$ for $\ell=0,\dots,L$ \;
        \nl Compute shell probabilities $s_{\ell} = \Phi(q_{\ell})-\Phi(q_{\ell-1})$ for $\ell=1,\dots,L$ \;
        \nl Set radii $r_{\ell} = \big(\phi(q_{\ell-1})-\phi(q_{\ell})\big) / s_{\ell} \; \sqrt{n}$ for $\ell=1,\dots,L$ (positive since $\phi(q_{\ell-1})>\phi(q_{\ell})$) \;
        \nl Return $\{r_{\ell}\}_{\ell=1}^{L}$.
}}
\caption{Axis radii for the cross scheme}\label{al:CrossSignatureRadii}
\end{algorithm}
}

\section{Compression for Dropout-Induced Bernoulli Mixtures}\label{append:compression4dropout}
We present an alternative to Algorithm~\ref{al:compressGMMs} for performing compression in the case of mixtures of multivariate Bernoulli distributions. We begin by showing that such mixtures naturally arise in dropout networks and that compression is necessary to enable tractable analysis of their output distributions. We then introduce a structure-aware compression method that yields a closed-form upper bound on the resulting Wasserstein distance. 

Let $\phi_\vb:\eucl^n\rightarrow\eucl^n$ be the masking function defined by $\phi_\vb(\vz)=\vb\odot\vz$, where $\vb\in\{0,1\}^n$ is a binary vector and $\odot$ denotes the element-wise multiplication operation. 
In dropout, $\vb$ is sampled from a multivariate Bernoulli distribution $d_\vb$, given by
\begin{equation*}
    d_\vb=\sum_{\vb\in \{0,1\}^n} \theta^{\sum_{l=1}^n\evb{l}}(1-\theta)^{n-\sum_{l=1}^n\evb{l}}\delta_\vb,  
\end{equation*}
where $\theta\in[0,1]$ is the dropout rate. Note that $d_b$ is a discrete distribution with support \(\{0,1\}^n\). 
Given any input distribution $p_{in}$, the output distribution of a dropout layer is defined as
\begin{equation*}
    p_{out}=\expect_{\vb\sim d_\vb}[\phi_\vb\# p_{in}],
\end{equation*}
which is a mixture distribution of $2^n$ components. 
If the input distribution $p_{in}$ is a discrete distribution of size $N$, such as the distribution at an intermediate layer of a dropout network, or a signature approximation in a VI-trained network, then the output becomes a discrete distribution with support size $N\cdot 2^{n}$. Specifically, for $d_{in}=\sum_{i=1}^N\evpi{i}\delta_{\vc_i}$, we have
\begin{equation}\label{eq:outputDropout}
    d_{out}=\expect_{\vb\sim d_\vb}[\phi_\vb\# d_{in}]=\sum_{i=1}^N\evpi{i}\sum_{\vb\in \{0,1\}^n} 
    \theta^{\sum_{l=1}^n\evb{l}}(1-\theta)^{n-\sum_{l=1}^n\evb{l}}\delta_{\vb\odot \vc_i}.
\end{equation}
Analyzing or storing $d_{out}$ becomes quickly intractable for moderately large $n$ due to its exponential growth in support size. Therefore, we require an efficient compression procedure to approximate $p_{out}$ with a tractable distribution.

Given a discrete input distribution $d_{in}$, one could apply the approach from Algorithm~\ref{al:compressGMMs} to compress $d_{out}$ into a discrete distribution $\bar{d}_{out}\in\discMeas_M(\eucl^n)$, where $N\leq M<N\cdot 2^n$. This would involve clustering the support of $d_{out}$ into $M$ groups using M-means, applying moment matching within each cluster, and solving a discrete optimal transport problem to compute the 2-Wasserstein error. 
However, this approach is computationally demanding when the number of components is large. 

\changed{
Instead, we exploit the structure of the Bernoulli mixture to design a more efficient compression strategy that admits a closed-form upper bound on the 2-Wasserstein distance. Specifically, given $d_{in}$, we select the $\frac{M-N}{2^n-1}$ components $\vc_i$ with the largest norms and collect their indices into a set \(\sI\) using a simple sorting algorithm. We then define the compressed distribution as:
\begin{equation}\label{eq:CompressedOutputDropout}
    \bar{d}_{out} = \sum_{i\in\sI^c}\evpi{i}\delta_{\vc_i} + \sum_{i\in\sI}\evpi{i}\sum_{\vb\in \{0,1\}^n} 
    \theta^{\sum_{l=1}^n\evb{l}}(1-\theta)^{n-\sum_{l=1}^n\evb{l}}\delta_{\vb\odot \vc_i}
\end{equation}
where $\sI^c=\{1,\hdots,N\}\setminus \sI$. That is, we apply dropout only to the inputs in $\sI$, and leave the others unperturbed. The resulting distribution \(\bar{d}_{out}\) has support size
\[
    |\sI^c| + |\sI|\cdot 2^n  =\left(N - \frac{M-N}{2^n-1}\right) + \frac{M-N}{2^n-1} \cdot 2^n=M.
\]
The following proposition provides a closed-form upper bound on the 2-Wasserstein distance between $d_{out}$ and $\bar{d}_{out}$.
\begin{proposition}\label{prop:wasser4compressMultiBernoulli}
    Let $d_{out}\in\discMeas_{N\cdot 2^n}(\eucl^n)$ and $\bar{d}_{out}\in\discMeas_{M}(\eucl^n)$ be as defined in Eqn~\eqref{eq:outputDropout} and  Eqn~\eqref{eq:CompressedOutputDropout}, respectively. Then, 
    \begin{equation*}
        \wasserstein_2^2(d_{out}, \bar{d}_{out}) \leq \sum_{i\in\sI^c} (1-\theta)\|\vc_i\|^2
    \end{equation*}
\end{proposition}
\begin{proof}
    Note that we can write
    \[
        d_{out} = \sum_{i=1}^N \evpi{i}p_i,
    \]
    where $p_i=\sum_{\vb\in \{0,1\}^n} 
    \theta^{\sum_{l=1}^n\evb{l}}(1-\theta)^{n-\sum_{l=1}^n\evb{l}}\delta_{\vb\odot \vc_i}$, and
    \[
        \bar{d}_{out} = \sum_{i\in\sI}\evpi{i}p_i + \sum_{i\in\sI^c}\evpi{i}\delta_{\vc_i}.
    \]
    Hence, according to Lemma~\ref{lem:wasser4Mixtures},
    \begin{align*}
        \wasserstein_2^2(d_{out}, \bar{d}_{out}) = \sum_{i\in\sI}\evpi{i} \wasserstein_2^2\left(p_i, p_i\right) + \sum_{i\in\sI^c}\evpi{i} \wasserstein_2^2\left(p_i, \delta_{\vc_i}\right) = \sum_{i\in\sI^c}\evpi{i} \wasserstein_2^2\left(p_i, \delta_{\vc_i}\right)
    \end{align*}
    To compute \(\wasserstein_2^2\left(p_i, \delta_{\vc_i}\right)\), we observe that the optimal transport plan moves all probability mass from $p_i$ to $\vc_i$. Thus,
    \begin{align*}
        \wasserstein_2^2\left(p_i, \delta_{\vc_i}\right) &= \expect_{\vb\sim d_b}[\|\vb\odot\vc_i-\vc_i\|^2] =\expect_{\vb\sim d_b}\left[\sum_{j=1}^n\left(\evb{j}\elem{\vc_i}{j}-\elem{\vc_i}{j}\right)^2\right]=\sum_{j=1}^n\elem{\vc_i}{j}^2\expect_{\beta\sim p_\beta}\left[(\beta-1)^2\right],
    \end{align*}
    where \(p_\beta=\theta\delta_1 + (1-\theta)\delta_0\). Thus
    \[
        \expect_{\beta\sim p_\beta}\left[(\beta-1)^2\right] = 1 - \expect_{\beta\sim p_\beta}\left[\beta^2-2\beta\right]=1-\expect_{\beta\sim p_\beta}\left[\beta\right]=1-\theta.
    \]
    So, \(\wasserstein_2^2\left(p_i, \delta_{\vc_i}\right) = (1-\theta)\|\vc_i\|^2\), which completes the proof.
\end{proof}
}

\section{Results on Signatures of Gaussian Mixture Distributions}\label{append:ExperimentsSigns4GMMs}
Here, we experimentally evaluate the effectiveness of Algorithm~\ref{al:signaturesOfGaussianMixutres} in constructing signatures on Gaussian mixtures. Specifically, we examine the conservatism of the formal bound on the 2-Wasserstein distance resulting from the signature operation provided by Algorithm~\ref{al:signaturesOfGaussianMixutres}, and analyze the restrictiveness of having the signature locations of each component of the mixture on a grid as in Algorithm~\ref{al:signaturesOfGaussianMixutres}. 

In Figure~\ref{fig:wasserGMMsConservatismComponentWiseSignatures}, we analyze how the approximation error 
from the signature of a Gaussian mixture with $M$ components, obtained according to Algorithm~\ref{al:signaturesOfGaussianMixutres}, changes with increasing signature size. The plots show that as the signature size increases, the approximation error, measured as the 2-Wasserstein distance, decreases uniformly. 
Furthermore, in line with Corollary~\ref{corol:Was4SignMultGaus}, the formal upper bound on $\wasserstein_2$ is exact for Gaussians ($M=1$). 
For Gaussian mixtures ($M>1$), the conservatism introduced by the formal bounds grows approximately linearly with increasing $M$. 
This is because, in Algorithm~\ref{al:SNN2GMM}, we bound the Wasserstein distance resulting from the signature operation on the GMMs by the weighted sum of the 2-Wasserstein distance between each Gaussian component in the mixture and their signatures (see line~7 in Algorithm~\ref{al:signaturesOfGaussianMixutres} and Eqn~\ref{eq:Was4SignMixMultGaus}).

Recall from Subsection~\ref{sec:wasser4signatures} that to obtain a closed-form bound on the approximation error in Algorithm~\ref{al:signaturesOfGaussianMixutres}, we place the signature locations of each component of the mixture on a grid in the transformed space induced by the element's covariance matrix, as illustrated in Figure~\ref{fig:signatureGaussian}. Specifically, following Proposition~\ref{prop:OptimalGrid}, we choose the grids such that for a given signature size $N$, the 2-Wasserstein distance from the signature operation is minimized. In Figure~\ref{fig:EmpiricalConvergenceRate}, we estimate the optimality gap if we could instead place the signatures freely to minimize the 2-Wasserstein distance. Perhaps surprisingly, the plots show that placing the signatures freely only slightly reduces the 2-Wasserstein distance. 
Thus, although optimal signature locations are not grid-aligned (in line with literature, \citealp{graf2000foundations}), for the Gaussian mixtures with distinct modes considered here the optimally chosen grid closely approximates the unconstrained optimal locations, indicating that component-wise grids can be near-optimal for such mixtures. 

\begin{figure}[htbp]
    \centering
    \includegraphics[width=0.5\textwidth]{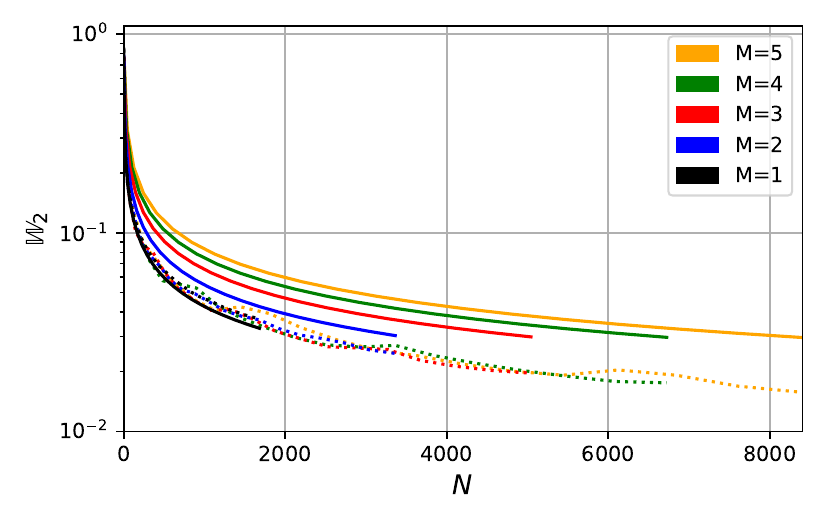}
    \caption{Empirical estimates (dashed lines) and formal bounds (solid lines) on $\wasserstein_2$ between a 2D Gaussian mixture distribution with $M$ components and the signature with $N$ grid-constrained locations obtained via Algorithm \ref{al:signaturesOfGaussianMixutres}. Formal bounds are computed with Algorithm~\ref{al:signaturesOfGaussianMixutres}. Empirical estimates are computed from $10^4$ samples from each distribution and computing the 2‑Wasserstein distance between the resulting empirical distributions.}    
    \label{fig:wasserGMMsConservatismComponentWiseSignatures}
\end{figure}

\bibliography{bibliography}

\end{document}